\documentclass{article}

\usepackage[margin=1in]{geometry}
\PassOptionsToPackage{numbers, compress}{natbib}
\usepackage[utf8]{inputenc} 
\usepackage[T1]{fontenc}    
\usepackage{hyperref}       
\usepackage{url}            
\usepackage{booktabs}       
\usepackage{amsfonts}       
\usepackage{nicefrac}       
\usepackage{microtype}      
\usepackage{subcaption}
\usepackage[font=small]{caption}
\usepackage{amsthm}
\usepackage[english]{babel}
\usepackage[T1]{fontenc}
\usepackage{amsmath}
\usepackage{amssymb}
\usepackage{amsfonts}
\usepackage{multirow}
\usepackage{mathtools}
\usepackage{breqn}
\usepackage{bbm}
\usepackage{dsfont}
\usepackage{graphicx}
\usepackage{natbib}
\usepackage{cases}
\usepackage[colorinlistoftodos]{todonotes}

\renewcommand{\hat}{\widehat}
\def\shownotes{1}  \ifnum\shownotes=1
\newtheorem{assumption}{Assumption}[section]
\newtheorem{proposition}{Proposition}[section]
\newtheorem{theorem}{Theorem}[section]
\newtheorem{lemma}{Lemma}[section]
\newtheorem{definition}{Definition}[section]

\newtheorem{example}{Example}[section]

\newtheorem{claim}{Claim}[section]
\newtheorem*{remark*}{Remark}

\newtheorem*{observation*}{Observation}

\numberwithin{equation}{section}

\newcommand{\E}{\mathbb{E}}
\DeclareMathOperator*{\Exp}{\mathbb{E}}

\newcommand{\R}{\mathbb{R}}

\newcommand{\cD}{\mathcal{D}}

\newcommand{\cN}{\mathcal{N}}

\newcommand{\Gnorm}[1]{{\left\vert\kern-0.25ex\left\vert\kern-0.25ex\left\vert #1 
		\right\vert\kern-0.25ex\right\vert\kern-0.25ex\right\vert}}
\newcommand{\gnorm}[1]{{\vert\kern-0.25ex\vert\kern-0.25ex\vert #1 
		\vert\kern-0.25ex\vert\kern-0.25ex\vert}}

\newcommand{\ent}{\ell_{ent}}
\newcommand{\lexp}{\ell_{exp}}
\newcommand{\expent}{g}
\newcommand{\density}{p}
\newcommand{\dratio}{s}
\newcommand{\sigopt}{\gamma^\star_{\sigma}}
\newcommand{\concave}{\nu}

\newcommand{\erf}{\textup{erf}}
\newcommand{\smooth}{\rho}

\newcommand{\constl}{0.25}

\newcommand{\ws}{{w^{\textup{s}}}}

\newcommand{\Dt}{\cD_{\textup{tg}}}
\newcommand{\Dtone}{\cD_{\textup{tg},1}}

\DeclareMathOperator{\sign}{\textup{sgn}}
\DeclareMathOperator{\erfc}{\textup{erfc}}

\title{Self-training Avoids Using Spurious Features Under Domain Shift}

\author{%
  Yining Chen\thanks{Equal Contribution}$^{*}$, Colin Wei$^{*}$, Ananya Kumar, Tengyu Ma \\
  Department of Computer Science\\
  Stanford University\\
  \texttt{\{cynnjjs, colinwei, ananya1, tengyuma\}@stanford.edu}
}

\begin{document}

\maketitle

\begin{abstract}
\noindent In unsupervised domain adaptation, existing theory focuses on situations where the source and target domains are close. In practice, conditional entropy minimization and pseudo-labeling work even when the domain shifts are much larger than those analyzed by existing theory. We identify and analyze one particular setting where the domain shift can be large, but these algorithms provably work: certain spurious features correlate with the label in the source domain but are independent of the label in the target. Our analysis considers linear classification where the spurious features are Gaussian and the non-spurious features are a mixture of log-concave distributions. For this setting, we prove that entropy minimization on unlabeled target data will avoid using the spurious feature if initialized with a decently accurate source classifier, even though the objective is non-convex and contains multiple bad local minima using the spurious features. We verify our theory for spurious domain shift tasks on semi-synthetic Celeb-A and MNIST datasets. Our results suggest that practitioners collect and self-train on large, diverse datasets to reduce biases in classifiers even if labeling is impractical.
\end{abstract}

\section{Introduction}

Reliable machine learning systems need to generalize to test distributions that are different from the training distribution. However, the test performance of machine learning models often significantly degrades as the test domain drifts away from the training domain. Various approaches have been proposed to adapt the models to new domains~\citep{tzeng2014domain, ganin2015domain, tzeng2017domain} but theoretical understanding of these algorithms is limited. Prior theoretical works focus on settings where the target domain is sufficiently close to the source domain \citep{shai2008unlabeled, shimodaira2000improving, huang2006correcting, sugiyama2007covariate}. To theoretically study realistic scenarios where domain shifts can be much larger, we need to leverage additional structure of the shifts.

Towards this goal, we propose to study a particular structured domain shift for which unsupervised domain adaptation is provably feasible: in the source domain, a subset of ``spurious'' features correlate with the label, whereas in the unlabeled target data, these features are independent of the label. In real-world training data, these spurious correlations can occur due to biased sampling or artifacts in crowd-sourcing~\citep{gururangan-etal-2018-annotation}. For example, we may have a labeled dataset for recidivism prediction where race correlates with recurrence of crime due to sample selection bias, but this correlation does not hold on the population. Models which learn spurious correlations can generalize poorly on population data which does not have these biases~\citep{mccoy-etal-2019-right}. In these settings, it could be impractical to acquire labels for an unbiased sub-sample of the population, but unlabeled data is often available.

We prove that in certain settings, perhaps surprisingly, self-training on \textit{unlabeled} target data can avoid using these spurious features. Our theoretical results apply to two closely-related popular algorithms: self-training~\citep{lee2013pseudo} and conditional entropy minimization~\citep{grandvalet05entropy}. In practice, self-training has achieved competitive or state-of-the-art results in unsupervised domain adaptation~\citep{long2013transfer, zou2019confidence, shu2018dirtt}, but there are few theoretical analyses of self-training when there is domain shift.

Our theoretical setting and analysis are consistent with recent large-scale empirical results by~\citet{xie2020selftraining}, which suggest that self-training on a more diverse unlabeled dataset can improve the robustness of a model, potentially by avoiding using spurious correlations. These results and our theory help emphasize the value of a large and diverse unlabeled dataset, even if labeled data is scarce.

Formally, we assume that each input consists of two subsets of features, denoted by $x_1$ and $x_2$. $x_1$ is the ``signal'' feature that determines the label $y$ in the target distribution. $x_2$ is the spurious feature that correlates with the label $y$ in the source domain, but $x_2$ is independent of $(x_1,y)$ in the target domain. For a first-cut result, we consider binary classification and linear models on the features $(x_1,x_2)$, where the spurious feature $x_2$ is a multivariate Gaussian and $x_1$ is a mixture of log-concave distributions. We aim to show that, initialized with some classifier trained on the source data, self-training on the unlabeled target will remove usage of the spurious feature $x_2$.

A challenge in the analysis is that self-training on an unlabeled loss can possibly harm, rather than help, target accuracy by amplifying the mistakes of source classifier (see Section~\ref{sec:failure_cases}). The classical idea of co-training~\citep{blum98cotraining} deals with this by assuming the mistakes of the classifier are independent of $x$, reducing the problem to learning from noisy labels. However, in our setting the source classifier makes biased mistakes which depend on $x$, and self-training potentially reinforces these biases if there are no additional assumptions. For example, we require initialization with a decently accurate source classifier, and we empirically verify the necessity of this assumption in Section~\ref{sec:experiments}.

Our main contribution (Theorem \ref{thm:general_main}) is to prove that self-training and conditional min entropy using finite unlabeled data converge to a solution that has 0 coefficients on the spurious feature $x_2$, assuming the following: 1. the signal $x_1$ is a mixture of well-separated log-concave distributions and 2. the initial source classifier is decently accurate on target data and avoids relying too heavily on the spurious feature. In a simpler setting where $x_1$ is a univariate Gaussian, we show that self-training using a decently accurate source classifier converges to the Bayes optimal solution (Theorem~\ref{thm:mixture}).

We run simulations on semi-synthetic colored MNIST~\citep{MNIST} and celebA~\citep{liu2015faceattributes} datasets to verify the insights from our theory and show that they apply to multi-layer neural networks and datasets where the spurious features are not necessarily a subset of the input coordinates (Section~\ref{sec:experiments}). Our code is available online at \url{https://github.com/cynnjjs/spurious_feature_NeuRIPS}.

\subsection{Related Work}

\textbf{Self-training} methods have achieved state-of-the-art results for semi-supervised learning~\citep{xie2020selftraining, sohn2020fixmatch, lee2013pseudo}, adversarial robustness~\citep{long2013transfer, zou2019confidence, shu2018dirtt}, and unsupervised domain adaptation~\citep{long2013transfer, zou2019confidence, shu2018dirtt}, but there is little understanding of when and why these methods work under domain shifts.
Two popular forms of self-training are pseudolabeling~\citep{lee2013pseudo} and conditional entropy minimization~\citep{grandvalet05entropy}, which have been observed to be closely connected~\citep{amini2003semisupervised,lee2013pseudo,shu2018dirtt,berthelot2020remixmatch}. We show that our analysis applies to both entropy minimization and a version of pseudo-labeling where we initialize the student model with the teacher model and re-label after each gradient step (Proposition~\ref{prop:pseudo=entropy}).

\citet{kumar2020gradual} examine self-training for domain adaptation, but assume that $P(X | Y)$ is an isotropic Gaussian, that entropy minimization converges to the nearest local minima, and infinite unlabeled data.
They use a symmetry argument that requires these assumptions.
In our setting, the signal $x_1$ can be a mixture of many log-concave, log-smooth distributions, and we show that self-training does in fact converge with only finite unlabeled data, even though the loss landscape is non-convex. These require new, more general, proof techniques.

\noindent{\textbf{Domain adaptation and semi-supervised learning theory}}:
Importance weighting~\citep{shimodaira2000improving, huang2006correcting, sugiyama2007covariate} is a popular way to deal with \emph{covariate shift} but these methods assume that $P(Y \mid X)$ is the same for the source and target, which may not hold when there are spurious correlations in the source but not target.
Additionally, sample complexity bounds for importance weighting depend on the expected density ratios between the source and target, which can often scale exponentially in the dimension. Our finite sample guarantees only depend on properties of the target distribution (assuming a decently accurate source classifier) and are agnostic to this density ratio. The theory of $H \Delta H$-divergence lower bounds target accuracy of a classifier in terms of source accuracy if some distance between the domains is small~\citep{ben2010theory}; ~\citet{zhang2019bridging} extend this distance measure to multiclass classification.
In contrast, we show self-training can improve accuracy under our structured domain shift, even when the shift is potentially large.
Other theoretical papers on semi-supervised learning focus on analyzing when unlabeled data can help, but do not analyze domain shift~\citep{rigollet2007generalization, singh2008unlabeled, shai2008unlabeled,balcan2010discriminative}.

Co-training~\citep{blum98cotraining} is an algorithm that can leverage unlabeled data when the input features can be split into $(x_1, x_2)$ that are
conditionally independent given the label. Co-training assumes
this grouping is known a-priori, and that either group can be used to predict the label accurately. In our setting, the spurious feature cannot be used to predict the label accurately in the target domain, and the algorithm does not have access to the grouping between spurious and signal features.

\noindent{\textbf{Spurious and non-robust features.}}
Many works seek to identify causal features invariant across domains~\cite{IRM, peters2015causal, heinze2017conditional}.
~\citet{heinze2017conditional} use distributionally robust optimization to reduce reliance on spurious features but they assume that the same object can be observed under multiple conditions, for example the same person in a variety of poses and outfits.
~\citet{wang2019learning} use the gray-level co-occurrence matrix to project out certain superficial, domain-specific, statistics but this is tailored to specific types of spurious correlations in image datasets.
~\citet{kim2019learning} propose a regularization method to remove spurious correlations, but they require domain experts to label the spurious features.
Concurrent work to ours~\cite{sagawa2020investigation} assumes spurious correlations are labeled and demonstrates that over-parameterization can cause overfitting to spurious correlations which are present for most, but not all, of the data.
They analyze supervised training of linear classifiers on Gaussian data without domain shift, whereas we analyze self-training when there is domain shift.
Spurious features are also related to adversarial examples, which can possibly be attributed to non-robust features that can predict the label but are brittle under domain shift \cite{ilyas2019adversarial}.

A number of papers theoretically analyze the connection between adversarial robustness and accuracy or generalization for linear classifiers in simple Gaussian settings~\citep{tsipras2018robustness,schmidt2018adversarially,uesato2019are}.~\citet{carmon2019unlabeled} show that self-training on unlabeled data can improve adversarially robust generalization for linear models in a Gaussian setting. Though these research questions are orthogonal to ours, one technical contribution of our work is that our analysis extends to more general distributions than Gaussians. 

\noindent{\textbf{Fairness.}}
Spurious correlations in datasets can lead to unfair predictions when protected attributes are involved. Our work shows that self-training can potentially employ unlabeled population samples to overcome bias in labeled data \cite{Gong12overcomingdataset, Tommasi2017}.

\section{Setup}

\noindent{\bf Model.} We consider a linear model $\hat{y} = w^\top x$ where $w = (w_1,w_2)$ and $x = (x_1, x_2)$ with $w_1, x_1\in \R^{d_1}$ and $w_2, x_2\in \R^{d_2}$. We assume that the spurious features $x_2$ have Gaussian distribution with covariance $\Sigma_2 \succ 0$, so the target data $(x,y)\sim \Dt$ is generated by
\begin{align}
y ~ \stackrel{\mathclap{\normalfont\mbox{unif}}}{\sim}~\{\pm 1\}, \textup{and } &x_1 \sim \Dtone (\cdot | y) \nonumber\\
&x_2  \sim \cN(\Vec{0}, \Sigma_2), \Sigma_2\in \R^{d_2\times d_2} \label{eqn:xtwoass}
\end{align}
for some distribution $\Dtone$ over $\R^{d_1}$. Note that $x_2$ is a spurious feature because it is independent of the label $y$. Our results and analysis also transfer to a ``scrambled setup'' \cite{IRM} where we observe $z = \mathcal{S}x$ for some rotation matrix $\mathcal{S} \in \R^{(d_1 + d_2) \times (d_1 + d_2)}$. This follows as a direct consequence of the rotational invariance of the algorithm~\eqref{eqn:alg} and our assumptions.

\noindent{\textbf{Min-entropy objective.}}
The min-entropy objective on a target unlabeled example is defined as
$ \ent(w^\top x)$ where $\ent(t) = H((1+\exp(-t))^{-1})$ and $H$ is the binary entropy function.
For mathematical convenience, we consider an approximation $\lexp(t) = \exp(-|t|)$, which is commonly used in the literature for studying the logistic loss~\cite{soudry2018the}. $\lexp$ approximates $\ent(t)$ up to a constant factor and exhibits the same tail behavior (Figure \ref{fig:ent}). We experimentally validate in Section \ref{sec:justify_approx} that training using $\exp(-|t|)$ achieves the same effect for the algorithms we analyze. The population unlabeled objective on the target distribution that we consider is 
\begin{align}
L(w) \triangleq \Exp_{x\sim \Dt} \lexp(w^\top x)
\end{align}
where $\Dt$ denotes the distribution in the target domain. We mainly focus on analyzing the population loss for simplicity, but in our main results (Theorems~\ref{thm:general_main} and~\ref{thm:mixture}) we also give finite-sample guarantees. We analyze the following equivalent algorithms for self-training.

\noindent{\bf Entropy minimization.} We initialize $w$ from a source classifier $\ws$ and run projected gradient descent on the entropy objective:\footnote{We project to the unit ball for simplicity, as the loss $\lexp$ is not scale-invariant.} 
\begin{align}
w^0  = \ws ~\textup{and} ~
w^{t+1}  = \frac{w^t - \eta \nabla L(w^t)}{\|w^t - \eta \nabla L(w^t)\|_2} \label{eqn:alg}
\end{align}

\noindent{\bf Pseudo-labeling.} We consider a variant of pseudo-labeling where we label the target data using the classifier from the previous iteration and run projected gradient descent on the supervised loss
\begin{align}
L_{pseudo}^{t+1}(w) \triangleq \Exp_{x\sim \Dt} \ell_{exp}(w^\top x, y^t)
\label{eq:pseudo_alg}
\end{align} where $y^t = \sign{({w^t}^\top x)}$ and $\ell_{exp}(t, y)=\exp{(-ty)}$.
The algorithm is the same as \ref{eqn:alg} with $L(w)$ replaced by $L_{pseudo}^{t+1}(w)$. Note that this is different from some versions of pseudo-labeling, which train for many rounds of gradient descent before re-labeling. We observe that the two algorithms above are equivalent because the iterates are the same (see Section~\ref{sec:entropy_min_pseudo} for the formal proof). 

\begin{proposition}
\label{prop:pseudo=entropy}
The pseudo-labeling algorithm above converges to the same solution as the entropy minimization algorithm in~\eqref{eqn:alg}.
\end{proposition}
\section{Overview of Main Results}

\label{sec:main_results}

We would like to show that entropy minimization~\eqref{eqn:alg} drives the spurious feature $w_2$ to 0. However, this is somewhat surprising and challenging to prove because nothing in the loss or algorithm explicitly enforces a decrease in $\|w_2\|_2$. Indeed, without additional assumptions on the target distribution $\Dt$ and the initial source classifier $\ws$, we show that entropy minimization can actually cause $\|w_2\|_2$ to increase because self-training can reinforce existing biases in the source classifier. 

Examples~\ref{ex:mu_is_zero} and~\ref{ex:small_loss} highlight cases where entropy minimization can fail, which motivates our assumptions of separation (Assumption~\ref{ass:seperation}) and that the spurious $x_1$ is a mixture of \emph{sliced log concave} distributions. Under these assumptions, our main Theorem~\ref{thm:general_main} shows that entropy minimization~\eqref{eqn:alg} initialized with a decently accurate source classifier drives the coefficient of the spurious feature, $w_2$, to 0. For a simpler Gaussian setting, Theorem~\ref{thm:mixture} shows that entropy minimization with a sufficiently accurate source classifier converges to the Bayes optimal classifier.

\subsection{Failure cases of self-training}
\label{sec:failure_cases}
We highlight cases where self-training increases reliance on the spurious features, justifying our assumptions in Section~\ref{sec:assumptions}. 

\begin{example} [No contribution from signal, i.e. $w_1^\top x_1 = 0$.] \label{ex:mu_is_zero}

See Figure~\ref{fig:hard_cases} (Left). For simplicity, suppose that $d_1 = d_2 = 1$, and suppose that $w_1 = 0$, so the signal feature is not used. In this case, increasing $|w_2|$ drives every prediction further from 0, decreasing the expected loss $L(w)$. Thus, in this example the min-entropy loss actually encourages the weight on the spurious feature, $|w_2|$, to increase. Note that this is not trivially true when $w_1$ is nonzero.

\end{example}

\begin{figure}
    \centering
    \includegraphics[width=0.45\textwidth]{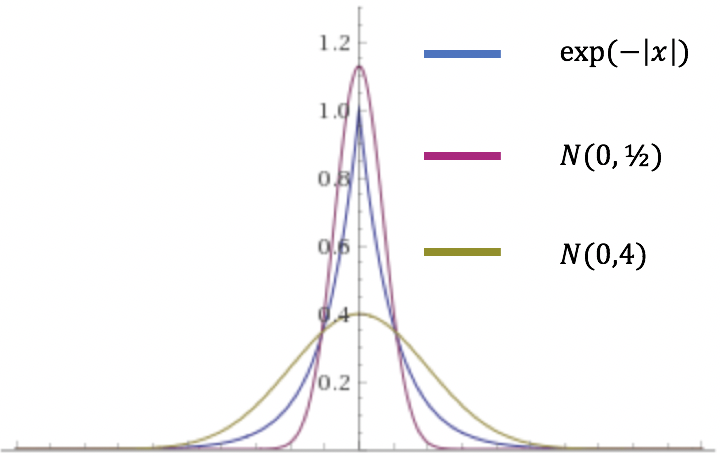}
    \includegraphics[width=0.45\textwidth]{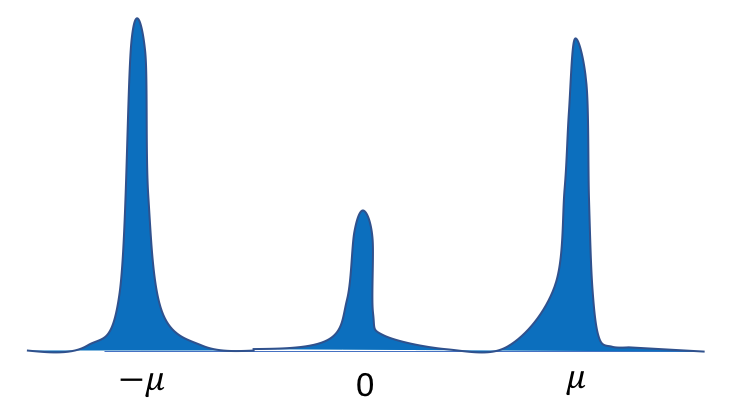} 
    \caption{Cases where entropy minimization fails to remove $w_2$. {\bf Left (Example~\ref{ex:mu_is_zero}):} When $w_1 = 0$, $w^\top x$ is distributed as a Gaussian. Increasing $w_2$, which increases this variance (e.g, going from the purple to green curve) decreases $L(w)$ by forcing every prediction further from 0. This means that entropy minimization causes reliance on the spurious feature to \textit{increase}. {\bf Right (Example~\ref{ex:small_loss}):} Distribution of $w_1^\top x_1$ in a hard case for general distributions. If there is a lot of mass of $w_1^\top x_1$ concentrated near the boundaries (i.e, $\pm \mu$ for some large $\mu \rightarrow \infty$) and a small amount of mass near 0, the loss could be small but the classifier will not want to shrink $\|w_2\|_2$.}
    \label{fig:hard_cases}
\end{figure}

In a realistic scenario, it's unlikely that $w_1^\top x_1 = 0$ for all examples because then the source accuracy on the target domain is very poor. So a priori, if we assume the source accuracy is decent (which implies $L(\ws)$ is small), we may avoid the pathological case above. However, this is not sufficient. 

\begin{example} [Initial $L(\ws)$ is small, but self-training still increases $\|w_2\|_2$.] \label{ex:small_loss}

See Figure~\ref{fig:hard_cases} (Right). Suppose that restricting to the signal feature, we have a mixture of perfectly and extremely confidently predicted examples, and a small amount of unconfident examples as in Example~\ref{ex:mu_is_zero}. The majority group of confident examples is already perfectly predicted with no incentive to remove $w_2$ (because the loss gradient is near 0), and the minority group encourages $\|w_2\|_2$ to increase as in Example~\ref{ex:mu_is_zero}, so the overall effect is for $\|w_2\|_2$ to increase though $L(w)$ is small.

\end{example}
For self-training to succeed, the correctly and confidently predicted examples must help remove the spurious features. As demonstrated above, this requires some continuum between confidently and unconfidently predicted examples. This motivates the log-concavity and smoothness assumptions, which guarantees that the sample distribution is not supported on too many extremely isolated clusters.

\subsection{Mixtures of log-concave and log-smooth distributions}
\label{sec:assumptions}

To avoid the failure cases above, we make realistic assumptions which are plausible in real-world data distributions. We start by defining a variant of log-concave and log-smooth distributions.

\begin{definition}[sliced log-concavity, log-smoothness]
	A distribution over $\R^d$ with density $p$ is $\alpha$-log-concave for $\alpha > 0$ if $ \nabla^2 \log p(t) \preceq -\alpha\cdot I_{d\times d}$, and is $\beta$-log-smooth if $\|\nabla^2 \log p(t)\|_{\textup{op}} \le \beta $. A distribution $p$ over $\R^d$ is sliced $\alpha$-log-concave or sliced $\beta$-log-smooth if for any unit vector $v$, the random variable $v^\top x$ with $x\sim p$ is $\alpha$-log-concave or $\beta$-log-smooth, respectively.
\end{definition}

A 1 dimensional density that is not Gaussian which satisfies these assumptions is $p(x) \propto \exp(-x^2 + \cos x)$. This density is 1-log concave and 3-log smooth.
Now we state our main assumption that $x_1$ consists of a mixture of sliced-log-concave and smooth distributions with sufficient separation.  
\begin{assumption}[Separation assumption on the data]~\label{ass:seperation}
	We assume that the distribution of $x_1$ in the target domain, denoted by $\Dtone$, is a mixture of $K$  sliced $\alpha$-log-concave and $\beta$-log-smooth distributions. (The reader can think of $\alpha$, $\beta$ and $K$ as absolute constants for simplicity.) Let $\tau_1,\dots, \tau_K$ denote the probability of each mixture and $\tau = \min_{i}\tau_K$. We assume that these mixtures are sufficiently separated in the sense that for scalar $\kappa$ (that depends on $\alpha$ and $\beta$), there exists $(w_1, 0) \in \R^{d_1 + d_2}$ such that  $L((w_1, 0)) \le \tau \kappa$.
\end{assumption}

We formally define $\kappa$ in Section~\ref{sec:general_case_overview}. When $\alpha$ and $\beta$ are of constant scale, $\kappa$ is also a constant. Assumption~\ref{ass:seperation} is always satisfied if the means of each mixture distribution in $\Dtone$ are sufficiently bounded from 0. We can see why Assumption~\ref{ass:seperation} is a separation condition by considering the case when there exists $(w_1, 0)$ with good classification accuracy on $x_1$. Obtaining good classification accuracy is only possible if the means of different classes are sufficiently far from 0 and also on opposite sides of 0, resulting in separation between the two classes. 

The sliced log-concavity ensures that each mixture component of $w_1^\top x_1$ is uni-modal, with upper bound $\alpha$ on its ``width''. Likewise, the sliced log-smoothness condition ensures that each component is not too narrow. These conditions rule out the hard distribution in Figure~\ref{fig:hard_cases} (right), as each of the three components change quite sharply, violating log-smoothness. 
Next, we assume the source classifier is decently accurate and has bounded usage of the spurious feature. 
\begin{assumption}[Source classifier is decently accurate, doesn't rely too much on spurious features.]\label{ass:boundinit}
We assume that the source classifier has a significant mass in the space of the signal $x_1$ in the sense that $\|\ws_1\|_2\ge 1/2$. We further assume that either $\Sigma_2$ is sufficiently small or $\ws_2$, the initialization in the spurious feature space, is sufficiently small, in the sense that 
	\begin{align}
	\sigma^2 = \ws^\top_2\Sigma_2 \ws_2\le c\cdot \min\{1, \alpha/\beta^2, \beta^{-1}, (\beta / |\log \beta|)^{-1}\}\nonumber
	\end{align}
	for some sufficiently small universal constant $c$ (e.g., $c=0.03$ can work.) Furthermore, we assume the source classifier $\ws$ has small entropy bound $L(\ws) \le \tau \kappa$, where $\kappa$ is the constant in Assumption~\ref{ass:seperation}.
\end{assumption}
The conditions on $\|\ws_1\|_2$ and $\sigma$ can be satisfied if $\ws_2$ is not too large. The following theorem shows that under our assumptions, entropy minimization succeeds in removing the spurious $w_2$.

\begin{theorem}[Main result]
	\label{thm:general_main}  
	In the setting above, suppose Assumptions~\ref{ass:seperation} and~\ref{ass:boundinit} hold and $L$ is smooth.\footnote{As there is a discontinuity in $\frac{d}{dt}\lexp(t)$ at $t = 0$, we need to assume smoothness. This regularity condition is easy to satisfy; for example, it holds if $\Dtone$ is a mixture of Gaussians.} If we run Algorithm~\eqref{eqn:alg} initialized with $\ws$ with sufficiently small step size $\eta$, after $O(\log \frac{1}{\epsilon})$ iterations, we will obtain $\hat{w}$ with very small usage of spurious features, i.e. $ \|\hat{w}_2\|_2 \le \epsilon $. The same conclusion holds with probability $1- \delta$ in the finite sample setting with $O(\frac{1}{\epsilon^4}\log \frac{1}{\delta})$ unlabeled samples from $\Dt$. 
\end{theorem}

Above, the notation $O(\cdot)$ hides dependencies on $\alpha, \beta, \Sigma_2$, and other distribution-dependent parameters. The interpretation of Theorem~\ref{thm:general_main} is that self-training can take a source classifier that has decent accuracy on the target and de-noise it completely, improving target accuracy by removing spurious extrapolations.
The main proof of Theorem~\ref{thm:general_main} is given in Section~\ref{sec:general:app}. We provide proof intuitions in Section~\ref{sec:proof_intuitions}. In Section~\ref{sec:general_proof}, we show we can ensure convergence to an approximate local minimum of the objective $\min_{\|w_1\|_2 \le 1} L((w_1,0))$ by adding Gaussian noise to the gradient updates.

\noindent{\bf Special case: mixtures of Gaussians.} 
We provide a slightly stronger analysis of Algorithm~\ref{eqn:alg} when the signal $x_1$ is a one-dimensional Gaussian mixture, i.e. we set $\Dtone(\cdot | y) = \cN(y \gamma, \sigma_1^2)$. Let $\tilde{\sigma}_{min}, \tilde{\sigma}_{max}$ denote the minimum and maximum eigenvalues of $\tilde{\Sigma}\triangleq \begin{pmatrix}
\sigma_1^2 & 0 \\
0 & \Sigma_2 
\end{pmatrix}$. 

We analyze a slightly more general variant of Algorithm~\ref{eqn:alg} which projects to the $R$-norm ball rather than unit ball. We now show that starting from an initial source classifier with sufficiently high accuracy on the target domain, self-training will avoid using the spurious feature and converge to the Bayes optimal classifier. Note that this is a stronger statement than Theorem~\ref{thm:general_main}, which does not bound the final target accuracy of the classifier. 

\begin{theorem}
	\label{thm:mixture} In the setting above, suppose we are given a classifier (trained on a source distribution) $\ws$ with  $\|\ws\| \le R$ and 0-1 error on the target domain at most $\rho = \frac{1}{2}\erfc{\left(\frac{r(R\tilde{\sigma}_{max})}{\sqrt{2}R\tilde{\sigma}_{min}}\right)}$.\footnote{$\erfc{(t)}=\frac{2}{\pi}\int_t^\infty{\exp{(-x^2)}dx}$.} ($r$ is a function as defined in Section~\ref{sec:gaussian_app}). Then Algorithm~\ref{eqn:alg} converges to $w^K$ satisfying 
\begin{align}
w^K_1   \ge \sqrt{R^2-\epsilon^2} \textup{ and } 
\|w^K_2\|_2  \le \epsilon \nonumber
\end{align}
within $K=O(\log\frac{1}{\epsilon})$ iterations. For the finite sample setting, the same conclusion holds with probability $1 - \delta$ using $O(\frac{1}{\epsilon^4}\log\frac{1}{\delta})$ samples. 
\end{theorem}

As above, $O(\cdot)$ hides dependencies in $R, \tilde{\sigma}_{min}, \tilde{\sigma}_{max}, \rho, d_2$. In particular, $w$ converges to $(R, 0)$, the classifier in $\{w : \|w\|_2 \le R\}$ with the best possible accuracy. The full proof is in Section~\ref{sec:gaussian_app}. 

\section{Overview of Analysis}
\label{sec:proof_intuitions}
We will summarize the key intuitions for proving Theorems~\ref{thm:general_main} and Theorem~\ref{thm:mixture}. The main ingredient is to show that the min entropy objective encourages a decrease in $\|w_2\|_2$, as stated below:

\begin{lemma}
	\label{lem:w2_dec-simplified}In the setting of Theorem~\ref{thm:mixture}, suppose that the classifier $w$ has at most $\rho$ error on the target. Then $\langle \nabla_{w_2}L(w), w_2\rangle \ge 0.$
	This same conclusion holds in the setting of Theorem~\ref{thm:general_main} for any $w$ satisfying the conditions in Assumption~\ref{ass:boundinit}.  
\end{lemma}

The consequence of Lemma~\ref{lem:w2_dec-simplified} is that one step of gradient descent on the loss function $L(w)$ shrinks the norm of $w_2$. This leads to the conclusion of Theorems~\ref{thm:general_main} and~\ref{thm:mixture}, modulo a few other nuances such as showing that the conditions of Lemma~\ref{lem:w2_dec-simplified} hold for all the iterates, which is done inductively. We also show that $\|w_1\|_2$ increases after one gradient step (Lemma~\ref{lem:w1_inc}), so the norm of $w_2$ still decreases after re-normalization. 
To prove Lemma~\ref{lem:w2_dec-simplified}, we first express the objective as follows:
\begin{align}
L(w) = \Exp_{x_1} \left[\Exp_{x_2} \lexp(w_1^\top x_1 + w_2^\top x_2)\right]
\end{align}
Note that $w_2^\top x_2$ has Gaussian distribution with mean zero and variance $\sigma^2 \triangleq w_2^\top \Sigma_2 w_2$. Let $g_\sigma(t) = \E_{z\sim \cN(0,\sigma^2)}[\ell(t+z)]$ denote the convolution of $\lexp$ with $\cN(0,\sigma^2)$. Then we can rewrite the loss as 
$
L(w) = \Exp_{w_1^\top x_1} \left[g_\sigma(w_1^\top x_1)\right]
$.

We now have $\nabla_{w_2}L(w) = \frac{\partial L(w)}{\partial \sigma} \cdot \frac{\partial \sigma}{\partial w_2} =\frac{\partial L(w)}{\partial \sigma}  \cdot 2\Sigma_2 w_2$, which implies $\langle \nabla_{w_2}L(w), w_2\rangle = 2w_2^\top \Sigma_2 w_2 \frac{\partial}{\partial \sigma} L(w)$. As $\Sigma_2 \succ 0$, proving Lemma~\ref{lem:w2_dec-simplified} is equivalent to proving $\frac{\partial}{\partial \sigma} L(w) \ge 0$. Letting $\mu \triangleq w_1^\top x_1$, we have $\frac{\partial}{\partial \sigma} L(w) = \E_{\mu}[q_{\sigma}(\mu)]$ where $q_{\sigma}(\mu) = \frac{\partial}{\partial \sigma} g(\mu)$. We now investigate when $q_{\sigma}(\mu) \ge 0$. As discussed in Example~\ref{ex:mu_is_zero} (and visualized in Figure~\ref{fig:q_sigma_mu}), $q_{\sigma}(\mu) < 0$ for $\mu$ near 0.

\noindent{\bf Case when $\mu \gg \sigma$:} Recall that 
$g_{\sigma}(\mu) =  \E_{z\sim \cN(\mu, \sigma^2)}\left[\lexp(z)\right]$ is the average of the entropy function over a Gaussian distribution. When $\mu$ is sufficiently large, most of the mass of the Gaussian distribution $\cN(\mu, \sigma^2)$ is on the positive side of the real line, where the function $\lexp$ is convex. For convex functions $f$, Jensen's inequality tells us $\Exp_{z\sim \cN(\mu, \sigma^2)}[f(\mu)] > f(\mu)$ if $\sigma > 0$. As decreasing $\sigma$ decreases the expected loss, we can see that $q_{\sigma}(\mu) > 0$. This is visualized in Figure~\ref{fig:q_sigma_mu} (left).

\begin{figure}
	\centering
	\includegraphics[width=0.45\textwidth]{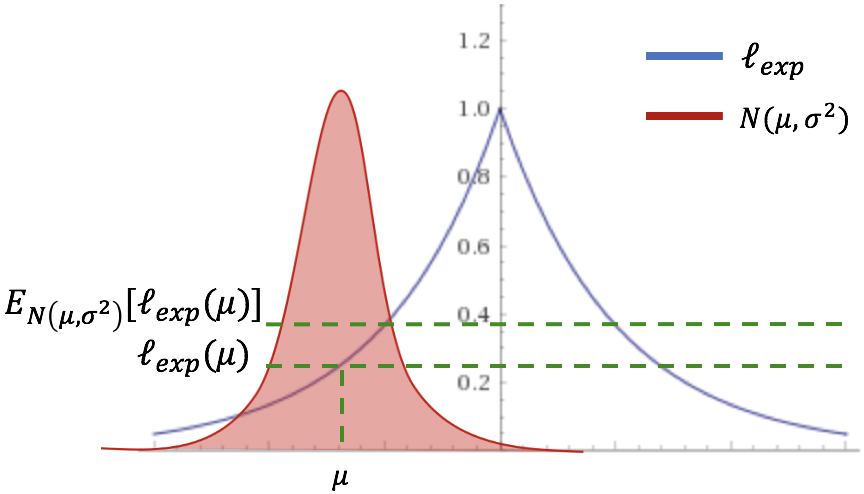}
	\includegraphics[width=0.45\textwidth]{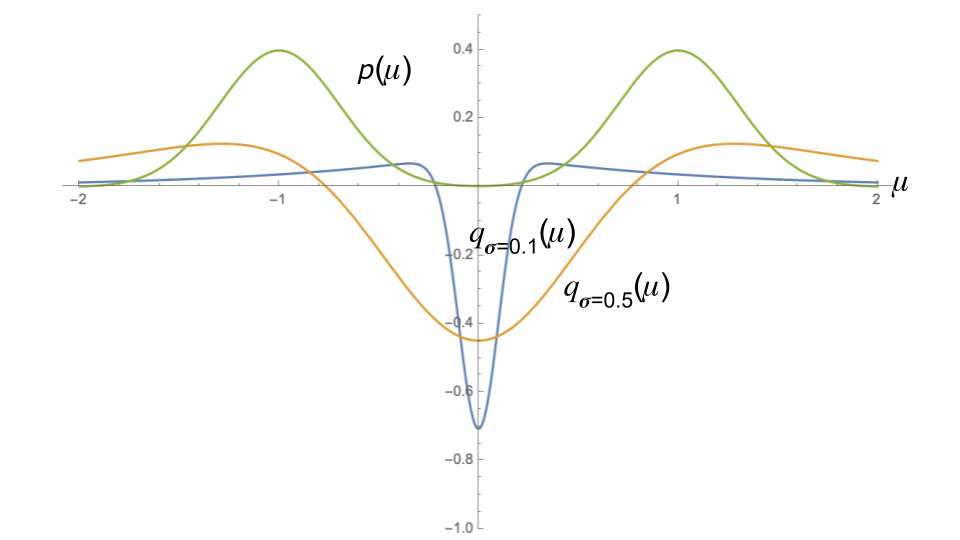}
    \caption{Analyzing dependence of $L$ on $\sigma$. {\bf Left: $\mu \gg \sigma$.} A visual depiction of why $q_{\sigma}(\mu) > 0$ when $\mu \gg \sigma$. Conditioned on $\mu = w_1^\top x_1$, $w^\top x$ is Gaussian with mean $\mu$. As $\mu \gg \sigma$, most of its mass is in the region where $\lexp$ is convex. By Jensen's inequality, driving $\sigma$ to 0 decreases the loss in this region. {\bf Right: plot of $q_{\sigma}(\mu)$.} The function $q_\sigma(\mu)$ will be convolved with $p$, the distribution over $\mu$. To guarantee $\E\left[q_\sigma(\mu)\right]\ge 0$, we would like to $\mu$ has as large amount of mass right to the positive root and left to the negative root of $q_\sigma(\cdot)$ as possible.} 
			\label{fig:q_sigma_mu}
\end{figure}

\noindent{\bf General case: } We plot $q_\sigma(\mu)$ as a function of $\mu$ for various choices of $\sigma$ in Figure~\ref{fig:q_sigma_mu} (right). We can see from the figure that for any $\sigma>0$, there is a threshold $r(\sigma)$ (defined in~\eqref{eq:r_def}) such that for any $|\mu| > r(\sigma)$, $q_\sigma(\mu) > 0$. In Lemma~\ref{lem:dg_dsigma}, we bound this value $r(\sigma)$ in terms of $\sigma$.

For the Gaussian setting, we can compute $\E_{\mu}\left[q_\sigma(\mu)\right]$ exactly and show it is positive for sufficiently accurate classifiers. For the general case (Theorem~\ref{thm:general_main}), it is difficult to bound this expectation because the expression for $q_{\sigma}(\mu)$ is complicated. Intuitively, our argument for why $\E_{\mu}[q_{\sigma}(\mu)] > 0$ is as follows: log-concavity and smoothness of each mixture component in $\mu$ ensures that the densities are uni-modal and do not change too fast. Thus, when $L(w)$ is sufficiently small, the mass of each component is spread over the real line, with most of the mass in the middle where $q_{\sigma}(\mu)$ is significantly positive, guaranteeing $\E_{\mu}[q_{\sigma}(\mu)] > 0$.

To formalize this, we use a second order Taylor expansion of the log density of $\mu$ and bound the error incurred by the expansion using smoothness and concavity. This analysis is presented in Section~\ref{sec:general:app}. 

In Section~\ref{sec:finite_sample}, we prove finite sample guarantees by showing that the gradient updates on the population and sample loss are similar ($\nabla \hat{L}(w) \approx \nabla L(w)$ for all $w$, where $\hat{L}(w)$ is the empirical loss).

\section{Experiments}
\label{sec:experiments}
We validate our theory in a variety of empirical settings. We study a more general setting with nonlinear models where the signal $x_1$ and spurious feature $x_2$ are not distinct dimensions of the data. Using a semi-synthetic colored MNIST dataset, we verify that 1. self-training avoids using spurious features in a manner consistent with our theory and 2. as our theory predicts, self-training can harm performance when the source classifier is not sufficiently accurate. We also confirm our theoretical conclusions on a celebA dataset modified to have spurious correlations in training data but not in test. 

Next, we investigate the connection between entropy minimization~\eqref{eqn:alg} and a stochastic variant of pseudo-labeling where pseudo-labels are frequently updated~\eqref{eq:pseudo_alg}. We demonstrate that entropy minimization can converge to better target accuracy within a fixed wall clock time-budget, suggesting that practitioners may benefit from pseudo-labeling with more rounds and fewer epochs per round (Section~\ref{sec:entropy_min_pseudo}). Finally, in Section~\ref{sec:toy_gaussian_exp} we verify that the conclusions of our theory also hold for more common variants of pseudo-labeling using simulations in a toy Gaussian setting. 

\subsection{Datasets}

\noindent{\bf Colored MNIST.} We create colored variants of the MNIST dataset~\cite{MNIST} inspired by~\cite{kim2019learning, IRM} where the shape of the digit is the signal feature and the color is the spurious feature. In the first variant, denoted CMNIST10, there are 10 classes. Color correlates with the label in the source with probability $p = 0.95$, but is uncorrelated with the label in the target. In the second variant, denoted CMNIST2, we group digits into two classes: 0-4 and 5-9, which allows detailed investigation of our theory. Color correlates with the class label in the source but not in the target using a construction described in detail in Section~\ref{sec:colored_mnist}. We train 3-layer feed-forward network on the source, and use this to initialize entropy minimization (Algorithm~\ref{eqn:alg}) on unlabeled target data. Evaluation is performed on held-out target samples.

\begin{figure}[t]
	\centering
	\begin{subfigure}{.32\textwidth}
		\centering
		\includegraphics[width=0.95\linewidth]{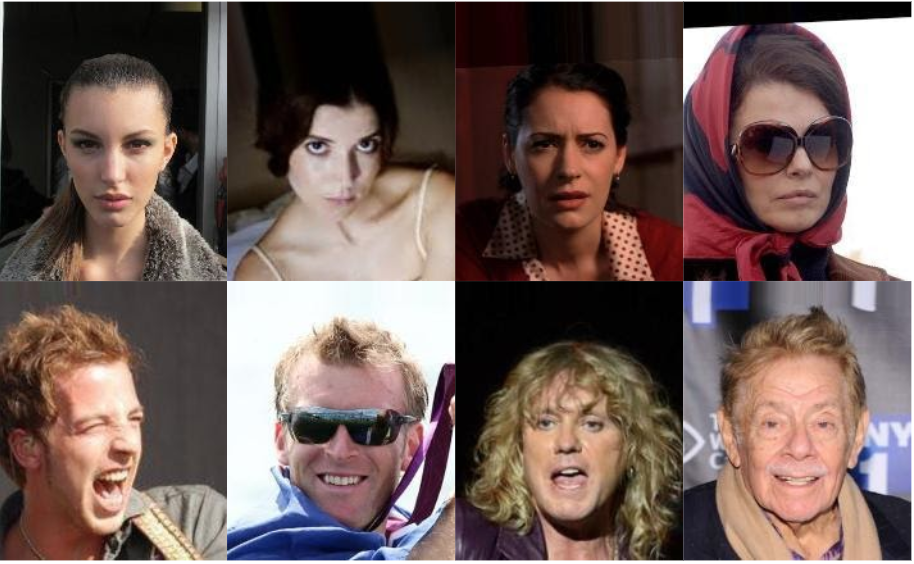}
		\caption{{\small Synthetic source data: blondness perfectly correlates with the male gender. }}\label{celeb_source}
	\end{subfigure}~
	\begin{subfigure}{.32\textwidth}
		\centering
		\includegraphics[width=0.95\linewidth]{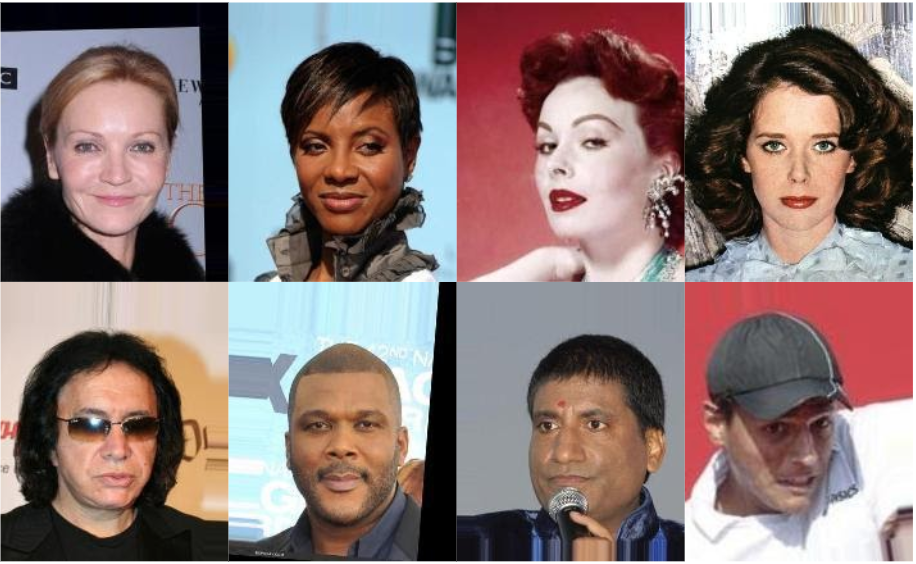}
		\caption{{\small Synthetic target data: each gender has a variety of hair colors.\vspace{\baselineskip}}}\label{celeb_target}
	\end{subfigure}~
	\begin{subfigure}{.32\textwidth}
		\centering
		\includegraphics[width=0.95\linewidth]{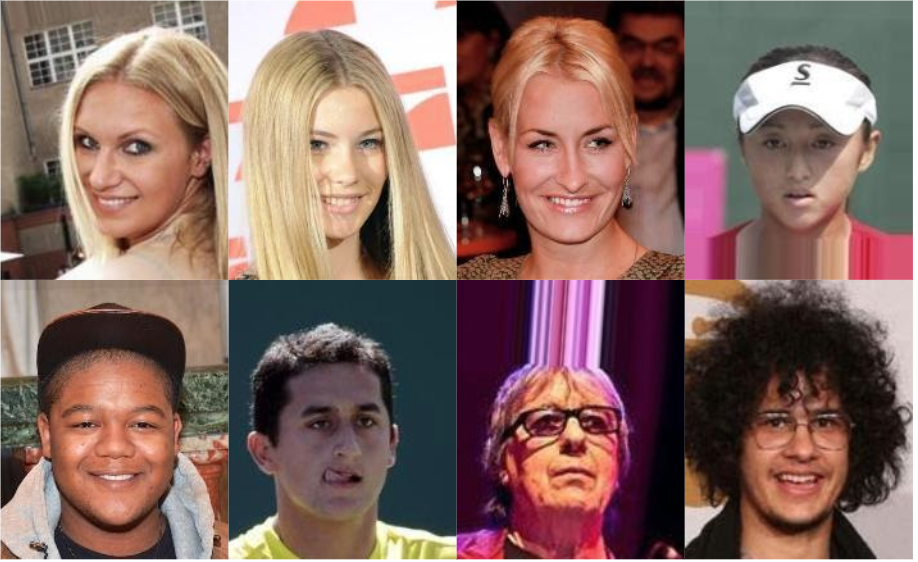}
		\caption{{\small Predictions corrected by self-training were mostly mistaken due to the spurious correlation.}}\label{fig:corrected}
	\end{subfigure}
	\caption{\small{In the synthetic CelebA experiment, the source has perfect correlation between hair color and gender, and the target does not. A classifier trained only on the source domain uses the spurious correlation. However, continuing to self-train on the unlabeled target domain reduces reliance on the spurious feature.}\label{fig:spurious}}
\end{figure}

\noindent{\bf CelebA dataset.} Inspired by~\cite{heinze2017conditional}, we partition the celebA dataset~\citep{liu2015faceattributes} so that gender correlates perfectly with hair color in source data (Figure~\ref{celeb_source}) but not in the target (Figure~\ref{celeb_target}). We train a neural net to predict gender by first training on source data alone and then performing self-training with unlabeled target data. During self-training, we add the labeled source loss to the min-entropy loss on target data. (Section~\ref{sec:celeb_a} has more details.) 

\begin{table}[t]
\begin{center}
\begin{small}
\begin{sc}
\begin{tabular}{lcccc}
\toprule
    & CelebA & CMNIST10 (p = 0.95) & CMNIST2 &  CMNIST10 (p = 0.97)\\
\midrule
    Trained on source & 81\% & 82\% & 94\% & 72\%\\
    After self-training & 88\% &  91\%  & 96\%& 67\%\\ 
\bottomrule
\end{tabular}
\end{sc}
\end{small}
\end{center}
\caption{Accuracy of models on the target before/after self-training, demonstrating that self-training can boost target accuracy under our structured domain shift. The exception is CMNIST10 with 0.97 probability of correlation between color and class. Here self-training decreases accuracy because initial accuracy is poor (only 72\%), justifying our assumption of a decently accurate source classifier.}
\label{tb:acc_improvement}
\end{table}

\subsection{Results}

\noindent{\bf Self-training improves target accuracy.} Table~\ref{tb:acc_improvement} shows that with a decently accurate source classifier, self-training on unlabeled target data leads to substantial improvements in the target domain. For example, on celebA the classifier achieves 81\% accuracy before self-training and 88\% after. This suggests that practitioners can potentially avoid overfitting to spurious correlations by self-training on large unlabeled datasets in the target domain. 

\noindent{\bf Self-training requires decent source classifier accuracy to succeed.} We test whether self-training is effective when the source classifier is bad by increasing the correlation between label and spurious color feature from 0.95 to 0.97 for CMNIST10. The resulting source classifier only obtains 72\% initial accuracy on target data, which \textit{drops} to 67\% after self-training (see Table~\ref{tb:acc_improvement}, last column, and plots in Section~\ref{sec:colored_mnist}). This shows that our assumption that the source classifier has to obtain non-trivial target accuracy (with bounded usage of the spurious feature) is also necessary in practice. 

\noindent {\bf Self-training reduces reliance on the spurious features.} In the CelebA experiment, test predictions corrected by self-training were mostly mistaken due to the spurious correlation. Figure~\ref{fig:corrected}, a random sample of the corrected examples, consists of mostly blond females, non-blond males, and subjects with hats or irregular hairstyles.

For 2-class colored MNIST, let $\mu_S(x_1)$, $\sigma_S(x_1)$ denote the mean and standard deviation of the source classifier conditioned on grayscale image $x_1$, with color distributed independently of $x_1$. Define $\mu_T(x_1)$, $\sigma_T(x_1)$ similarly for the classifier after self-training. Our theory suggests that $\sign{\mu_S(x_1)}=\sign{\mu_T(x_1)}$, $|\mu_S(x_1)|<|\mu_T(x_1)|$, and $\sigma_S(x_1) > \sigma_T(x_1)$, and we say a test example is {\em explainable} by our theory if this holds. We divide the test examples into four categories: ``-/+", ``+/-", ``+/+", ``-/-", where, for example, ``-/+" indicates source classifier was wrong but corrected by self-training. Table \ref{tb:statistics} summarizes the number of explainable examples in each category, showing that for the majority ($>90\%$) of examples, entropy minimization works due to the reason we hypothesized. In Section~\ref{sec:colored_mnist}, we provide additional detailed analyses of the influence of the spurious color feature on the prediction before and after self-training.

\begin{table}[t]
\begin{center}
\begin{small}
\begin{sc}
\begin{tabular}{lcccccr}
\toprule
    & -/+ & +/- & +/+ & -/- & Total \\
\midrule
    Explainable & 271 & 45 & 8785 & 150 & 9251\\
    Total & 349 & 86 &  9286 & 279 & 10000\\ 
\bottomrule
\end{tabular}
\end{sc}
\end{small}
\end{center}
\caption{Number of test examples explainable by our theory. See text for definitions and interpretation.}
\label{tb:statistics}
\end{table}

\section{Conclusion}
We study the impact of self-training under domain shift. We show that when there are spurious correlations in the source domain which are not present in the target, self-training leverages the unlabeled target data to avoid relying on these spurious correlations. Our analysis highlights several conditions for self-training to work in theory, such as good separation between classes and a decently accurate source classifier. Our experiments support that 1) these theoretical requirements can capture the initial conditions needed for self-training to work and 2) under these initial conditions, self-training can indeed prevent the model from using spurious features in ways predicted by our theory. It is an interesting question for future work to explore other settings we can analyze with our framework.  
\section{Acknowledgement}
We are grateful to Rui Shu, Shiori Sagawa, and Pang Wei Koh for insightful discussions. The authors would like to thank the Stanford Graduate Fellowship program for funding. CW acknowledges support from a NSF Graduate Research Fellowship. TM is also partially supported by the Google Faculty Award, Stanford Data Science Initiative, and the Stanford Artificial Intelligence Laboratory.

\bibliography{references}

\begin{thebibliography}{47}
\providecommand{\natexlab}[1]{#1}
\providecommand{\url}[1]{\texttt{#1}}
\expandafter\ifx\csname urlstyle\endcsname\relax
  \providecommand{\doi}[1]{doi: #1}\else
  \providecommand{\doi}{doi: \begingroup \urlstyle{rm}\Url}\fi

\bibitem[Agarwal et~al.(2017)Agarwal, Allen-Zhu, Bullins, Hazan, and
  Ma]{agarwal2017finding}
Naman Agarwal, Zeyuan Allen-Zhu, Brian Bullins, Elad Hazan, and Tengyu Ma.
\newblock Finding approximate local minima faster than gradient descent.
\newblock In \emph{Proceedings of the 49th Annual ACM SIGACT Symposium on
  Theory of Computing}, pages 1195--1199, 2017.

\bibitem[Amini and Gallinari(2003)]{amini2003semisupervised}
M.~Amini and P.~Gallinari.
\newblock Semi-supervised learning with explicit misclassification modeling.
\newblock In \emph{International Joint Conference on Artificial Intelligence
  (IJCAI)}, 2003.

\bibitem[Arjovsky et~al.(2019)Arjovsky, Bottou, Gulrajani, and Lopez-Paz]{IRM}
Martin Arjovsky, Léon Bottou, Ishaan Gulrajani, and David Lopez-Paz.
\newblock Invariant risk minimization, 2019.

\bibitem[Balcan and Blum(2010)]{balcan2010discriminative}
Maria-Florina Balcan and Avrim Blum.
\newblock A discriminative model for semi-supervised learning.
\newblock \emph{Journal of the ACM (JACM)}, 57\penalty0 (3):\penalty0 1--46,
  2010.

\bibitem[Ben-David et~al.(2008)Ben-David, Lu, and Pal]{shai2008unlabeled}
S.~Ben-David, T.~Lu, and D.~Pal.
\newblock Does unlabeled data provably help? worst-case analysis of the sample
  complexity of semi-supervised learning.
\newblock In \emph{Conference on Learning Theory (COLT)}, 2008.

\bibitem[Ben-David et~al.(2010)Ben-David, Blitzer, Crammer, Kulesza, Pereira,
  and Vaughan]{ben2010theory}
Shai Ben-David, John Blitzer, Koby Crammer, Alex Kulesza, Fernando Pereira, and
  Jennifer~Wortman Vaughan.
\newblock A theory of learning from different domains.
\newblock \emph{Machine learning}, 79\penalty0 (1-2):\penalty0 151--175, 2010.

\bibitem[Berthelot et~al.(2020)Berthelot, Carlini, Cubuk, Kurakin, Sohn, Zhang,
  and Raffel]{berthelot2020remixmatch}
David Berthelot, Nicholas Carlini, Ekin~D. Cubuk, Alex Kurakin, Kihyuk Sohn,
  Han Zhang, and Colin Raffel.
\newblock Remixmatch: Semi-supervised learning with distribution matching and
  augmentation anchoring.
\newblock In \emph{International Conference on Learning Representations}, 2020.

\bibitem[Blum and Mitchell(1998)]{blum98cotraining}
A.~Blum and T.~Mitchell.
\newblock Combining labeled and unlabeled data with co-training.
\newblock In \emph{Conference on Learning Theory (COLT)}, 1998.

\bibitem[Carmon et~al.(2019)Carmon, Raghunathan, Schmidt, Liang, and
  Duchi]{carmon2019unlabeled}
Y.~Carmon, A.~Raghunathan, L.~Schmidt, P.~Liang, and J.~C. Duchi.
\newblock Unlabeled data improves adversarial robustness.
\newblock In \emph{Advances in Neural Information Processing Systems
  (NeurIPS)}, 2019.

\bibitem[Ganin and Lempitsky(2015)]{ganin2015domain}
Y.~Ganin and V.~Lempitsky.
\newblock Unsupervised domain adaptation by backpropagation.
\newblock In \emph{International Conference on Machine Learning (ICML)}, pages
  1180--1189, 2015.

\bibitem[Ge et~al.(2015)Ge, Huang, Jin, and Yuan]{ge2015escaping}
Rong Ge, Furong Huang, Chi Jin, and Yang Yuan.
\newblock Escaping from saddle points—online stochastic gradient for tensor
  decomposition.
\newblock In \emph{Conference on Learning Theory}, pages 797--842, 2015.

\bibitem[Gong et~al.(2012)Gong, Sha, and Grauman]{Gong12overcomingdataset}
Boqing Gong, Fei Sha, and Kristen Grauman.
\newblock Overcoming dataset bias: An unsupervised domain adaptation approach.
\newblock In \emph{In NIPS Workshop on Large Scale Visual Recognition and
  Retrieval}, 2012.

\bibitem[Grandvalet and Bengio(2005)]{grandvalet05entropy}
Y.~Grandvalet and Y.~Bengio.
\newblock Entropy regularization.
\newblock In \emph{Semi-Supervised Learning}, 2005.

\bibitem[Gururangan et~al.(2018)Gururangan, Swayamdipta, Levy, Schwartz,
  Bowman, and Smith]{gururangan-etal-2018-annotation}
Suchin Gururangan, Swabha Swayamdipta, Omer Levy, Roy Schwartz, Samuel Bowman,
  and Noah~A. Smith.
\newblock Annotation artifacts in natural language inference data.
\newblock In \emph{Proceedings of the 2018 Conference of the North {A}merican
  Chapter of the Association for Computational Linguistics: Human Language
  Technologies, Volume 2 (Short Papers)}, pages 107--112, June 2018.

\bibitem[Heinze-Deml and Meinshausen(2017)]{heinze2017conditional}
C.~Heinze-Deml and N.~Meinshausen.
\newblock Conditional variance penalties and domain shift robustness.
\newblock \emph{arXiv preprint arXiv:1710.11469}, 2017.

\bibitem[Ilyas et~al.(2019)Ilyas, Santurkar, Tsipras, Engstrom, Tran, and
  Madry]{ilyas2019adversarial}
A.~Ilyas, S.~Santurkar, D.~Tsipras, L.~Engstrom, B.~Tran, and A.~Madry.
\newblock Adversarial examples are not bugs, they are features.
\newblock In \emph{Advances in Neural Information Processing Systems
  (NeurIPS)}, 2019.

\bibitem[Jiayuan et~al.(2006)Jiayuan, J., Arthur, M., and
  Bernhard]{huang2006correcting}
H.~Jiayuan, S.~A. J., G.~Arthur, B.~K. M., and S.~Bernhard.
\newblock Correcting sample selection bias by unlabeled data.
\newblock In \emph{Advances in Neural Information Processing Systems
  (NeurIPS)}, 2006.

\bibitem[Kim et~al.(2019)Kim, Kim, Kim, Kim, and Kim]{kim2019learning}
Byungju Kim, Hyunwoo Kim, Kyungsu Kim, Sungjin Kim, and Junmo Kim.
\newblock Learning not to learn: Training deep neural networks with biased
  data.
\newblock In \emph{Computer Vision and Pattern Recognition (CVPR)}, 2019.

\bibitem[Kumar et~al.(2020)Kumar, Ma, and Liang]{kumar2020gradual}
A.~Kumar, T.~Ma, and P.~Liang.
\newblock Understanding self-training for gradual domain adaptation.
\newblock \emph{arXiv preprint arXiv:2002.11361}, 2020.

\bibitem[LeCun et~al.(2010)LeCun, Cortes, and Burges]{MNIST}
Yann LeCun, Corinna Cortes, and CJ~Burges.
\newblock Mnist handwritten digit database.
\newblock \emph{ATT Labs [Online]. Available: http://yann. lecun.
  com/exdb/mnist}, 2, 2010.

\bibitem[Lee(2013)]{lee2013pseudo}
D.~Lee.
\newblock Pseudo-label: The simple and efficient semi-supervised learning
  method for deep neural networks.
\newblock In \emph{International Conference on Machine Learning (ICML)
  Workshop}, 2013.

\bibitem[Liu et~al.(2015)Liu, Luo, Wang, and Tang]{liu2015faceattributes}
Ziwei Liu, Ping Luo, Xiaogang Wang, and Xiaoou Tang.
\newblock Deep learning face attributes in the wild.
\newblock In \emph{Proceedings of International Conference on Computer Vision
  (ICCV)}, December 2015.

\bibitem[Long et~al.(2013)Long, Wang, Ding, Sun, and Yu]{long2013transfer}
M.~Long, J.~Wang, G.~Ding, J.~Sun, and P.~S. Yu.
\newblock Transfer feature learning with joint distribution adaptation.
\newblock In \emph{Proceedings of the IEEE international conference on computer
  vision}, pages 2200--2207, 2013.

\bibitem[McCoy et~al.(2019)McCoy, Pavlick, and Linzen]{mccoy-etal-2019-right}
Tom McCoy, Ellie Pavlick, and Tal Linzen.
\newblock Right for the wrong reasons: Diagnosing syntactic heuristics in
  natural language inference.
\newblock In \emph{Proceedings of the 57th Annual Meeting of the Association
  for Computational Linguistics}, pages 3428--3448, July 2019.

\bibitem[Miyato et~al.(2018)Miyato, Maeda, Koyama, and
  Ishii]{miyato2018virtual}
Takeru Miyato, Shin-ichi Maeda, Masanori Koyama, and Shin Ishii.
\newblock Virtual adversarial training: a regularization method for supervised
  and semi-supervised learning.
\newblock \emph{IEEE transactions on pattern analysis and machine
  intelligence}, 41\penalty0 (8):\penalty0 1979--1993, 2018.

\bibitem[Nesterov and Polyak(2006)]{nesterov2006cubic}
Yurii Nesterov and Boris~T Polyak.
\newblock Cubic regularization of newton method and its global performance.
\newblock \emph{Mathematical Programming}, 108\penalty0 (1):\penalty0 177--205,
  2006.

\bibitem[Peters et~al.(2015)Peters, Bühlmann, and
  Meinshausen]{peters2015causal}
Jonas Peters, Peter Bühlmann, and Nicolai Meinshausen.
\newblock Causal inference using invariant prediction: identification and
  confidence intervals, 2015.

\bibitem[Rigollet(2007)]{rigollet2007generalization}
P.~Rigollet.
\newblock Generalization error bounds in semi-supervised classification under
  the cluster assumption.
\newblock \emph{Journal of Machine Learning Research (JMLR)}, 8:\penalty0
  1369--1392, 2007.

\bibitem[Sagawa et~al.(2020)Sagawa, Raghunathan, Koh, and
  Liang]{sagawa2020investigation}
Shiori Sagawa, Aditi Raghunathan, Pang~Wei Koh, and Percy Liang.
\newblock An investigation of why overparameterization exacerbates spurious
  correlations.
\newblock \emph{ArXiv}, abs/2005.04345, 2020.

\bibitem[Schmidt et~al.(2018)Schmidt, Santurkar, Tsipras, Talwar, and
  Madry]{schmidt2018adversarially}
Ludwig Schmidt, Shibani Santurkar, Dimitris Tsipras, Kunal Talwar, and
  Aleksander Madry.
\newblock Adversarially robust generalization requires more data.
\newblock In \emph{Advances in Neural Information Processing Systems}, pages
  5014--5026, 2018.

\bibitem[Shimodaira(2000)]{shimodaira2000improving}
H.~Shimodaira.
\newblock Improving predictive inference under covariate shift by weighting the
  log-likelihood function.
\newblock \emph{Journal of Statistical Planning and Inference}, 90:\penalty0
  227--244, 2000.

\bibitem[Shu et~al.(2018)Shu, Bui, Narui, and Ermon]{shu2018dirtt}
R.~Shu, H.~H. Bui, H.~Narui, and S.~Ermon.
\newblock A {DIRT}-{T} approach to unsupervised domain adaptation.
\newblock In \emph{International Conference on Learning Representations
  (ICLR)}, 2018.

\bibitem[Singh et~al.(2008)Singh, Nowak, and Zhu]{singh2008unlabeled}
A.~Singh, R.~Nowak, and J.~Zhu.
\newblock Unlabeled data: Now it helps, now it doesn't.
\newblock In \emph{Advances in Neural Information Processing Systems
  (NeurIPS)}, 2008.

\bibitem[Sohn et~al.(2020)Sohn, Berthelot, Li, Zhang, Carlini, Cubuk, Kurakin,
  Zhang, and Raffel]{sohn2020fixmatch}
K.~Sohn, D.~Berthelot, C.~Li, Z.~Zhang, N.~Carlini, E.~D. Cubuk, A.~Kurakin,
  H.~Zhang, and C.~Raffel.
\newblock Fixmatch: Simplifying semi-supervised learning with consistency and
  confidence.
\newblock \emph{arXiv}, 2020.

\bibitem[Soudry et~al.(2018)Soudry, Hoffer, and Srebro]{soudry2018the}
Daniel Soudry, Elad Hoffer, and Nathan Srebro.
\newblock The implicit bias of gradient descent on separable data.
\newblock In \emph{International Conference on Learning Representations}, 2018.

\bibitem[Sugiyama et~al.(2007)Sugiyama, Krauledat, and
  Muller]{sugiyama2007covariate}
M.~Sugiyama, M.~Krauledat, and K.~Muller.
\newblock Covariate shift adaptation by importance weighted cross validation.
\newblock \emph{Journal of Machine Learning Research (JMLR)}, 8:\penalty0
  985--1005, 2007.

\bibitem[Tommasi et~al.(2015)Tommasi, Patricia, Caputo, and
  Tuytelaars]{Tommasi2017}
Tatiana Tommasi, Novi Patricia, Barbara Caputo, and Tinne Tuytelaars.
\newblock A deeper look at dataset bias.
\newblock In \emph{GCPR}, 2015.

\bibitem[Tsipras et~al.(2018)Tsipras, Santurkar, Engstrom, Turner, and
  Madry]{tsipras2018robustness}
Dimitris Tsipras, Shibani Santurkar, Logan Engstrom, Alexander Turner, and
  Aleksander Madry.
\newblock Robustness may be at odds with accuracy.
\newblock \emph{arXiv preprint arXiv:1805.12152}, 2018.

\bibitem[Tzeng et~al.(2014)Tzeng, Hoffman, Zhang, Saenko, and
  Darrell]{tzeng2014domain}
E.~Tzeng, J.~Hoffman, N.~Zhang, K.~Saenko, and T.~Darrell.
\newblock Deep domain confusion: Maximizing for domain invariance.
\newblock \emph{arXiv preprint arXiv:1412.3474}, 2014.

\bibitem[Tzeng et~al.(2017)Tzeng, Hoffman, Saenko, and
  Darrell]{tzeng2017domain}
E.~Tzeng, J.~Hoffman, K.~Saenko, and T.~Darrell.
\newblock Adversarial discriminative domain adaptation.
\newblock In \emph{Computer Vision and Pattern Recognition (CVPR)}, 2017.

\bibitem[Uesato et~al.(2019)Uesato, Alayrac, Huang, Stanforth, Fawzi, and
  Kohli]{uesato2019are}
J.~Uesato, J.~Alayrac, P.~Huang, R.~Stanforth, A.~Fawzi, and P.~Kohli.
\newblock Are labels required for improving adversarial robustness?
\newblock In \emph{Advances in Neural Information Processing Systems
  (NeurIPS)}, 2019.

\bibitem[Vershynin(2018)]{vershynin2018high}
Roman Vershynin.
\newblock \emph{High-Dimensional Probability}.
\newblock Cambridge University Press, 2018.

\bibitem[Wang et~al.(2019)Wang, He, Lipton, and Xing]{wang2019learning}
Haohan Wang, Zexue He, Zachary~C. Lipton, and Eric~P. Xing.
\newblock Learning robust representations by projecting superficial statistics
  out.
\newblock In \emph{International Conference on Learning Representations
  (ICLR)}, 2019.

\bibitem[Weisstein(2020)]{wolfram}
Eric~W. Weisstein.
\newblock \emph{Erfc. From MathWorld--A Wolfram Web Resource.}, 2020.
\newblock \url{http://mathworld.wolfram.com/Erfc.html}.

\bibitem[Xie et~al.(2020)Xie, Luong, Hovy, and Le]{xie2020selftraining}
Q.~Xie, M.~Luong, E.~Hovy, and Q.~V. Le.
\newblock Self-training with noisy student improves imagenet classification.
\newblock \emph{arXiv}, 2020.

\bibitem[Zhang et~al.(2019)Zhang, Liu, Long, and Jordan]{zhang2019bridging}
Yuchen Zhang, Tianle Liu, Mingsheng Long, and Michael~I Jordan.
\newblock Bridging theory and algorithm for domain adaptation.
\newblock \emph{arXiv preprint arXiv:1904.05801}, 2019.

\bibitem[Zou et~al.(2019)Zou, Yu, Liu, Kumar, and Wang]{zou2019confidence}
Y.~Zou, Z.~Yu, X.~Liu, B.~Kumar, and J.~Wang.
\newblock Confidence regularized self-training.
\newblock \emph{arXiv preprint arXiv:1908.09822}, 2019.

\end{thebibliography}
\bibliographystyle{plainnat}

\newpage
\appendix

\section{Warmup: Proofs for Gaussian Setting (Theorem~\ref{thm:mixture})}
\label{sec:gaussian_app}

We will define the function $r$ in Theorem~\ref{thm:mixture} as follows:
\begin{align}
	r(\sigma) = \begin{cases}
  \sigma^2+ \sigma\sqrt{2\log{\frac{4\sqrt{2}}{\sqrt{\pi}\sigma}}}, & \text{if } 0 < \sigma \le \frac{4\sqrt{2}}{\sqrt{\pi}}. \\
  2\sigma^2, & \text{if } \sigma > \frac{4\sqrt{2}}{\sqrt{\pi}}.
\end{cases}
\label{eq:r_def}
\end{align}
The algorithm that we consider is a variant of Algorithm~\ref{eqn:alg} which more generally projects to the $R$-norm ball rather than unit ball: $w^{t + 1} = R \frac{w^t - \eta \nabla L(w^t)}{\|w^t - \eta \nabla L(w^t)\|}$.
Now define
\begin{align}
    a&=\sqrt{2}R\tilde{\sigma}_{min} \erfc^{-1}{(2\rho)} = r(R \tilde{\sigma}_{max})\\
    S&=\{w: w_1 \gamma \ge a, ||w||_2\le R\}. 
    \label{eq:S_def}
\end{align}
where the function $r$ is defined in Lemma~\ref{lem:dg_dsigma}. We first observe that the condition that classifier $w$ has at most $1 - \rho$ error corresponds directly to $w \in S$. 

\begin{lemma}
	\label{lem:acc_to_a}
	In the setting of Theorem~\ref{thm:mixture}, suppose some classifier $w$ has at least $1-\rho$ accuracy in the sense that $\Pr_{x,y}{\left[\sign{(w^\top x)}=y\right]} \ge 1-\rho$ and $||w||_2\le R$. Then $w_1  \gamma\ge a$.
\end{lemma}

\begin{proof}[Proof of Lemma~\ref{lem:acc_to_a}]
	Let $\tilde{\sigma}^2=w^\top\tilde{\Sigma}w$. We have $R\tilde{\sigma}_{min}\le \tilde{\sigma} \le R\tilde{\sigma}_{max}$.
	\begin{align}
	&\Pr_{y \sim \{\pm 1\},z \sim \cN(\Vec{0}, \tilde{\sigma}^2)}{\left[\sign{(y w_1 \gamma+z)}=y\right]} \ge 1-\rho \\
	&\iff \Pr_{z \sim \cN(\Vec{0}, \tilde{\sigma}^2)}{\left[w_1  \gamma+z \ge 0\right]} \ge 1-\rho \\
	&\iff \erfc{\left(\frac{w_1  \gamma}{\sqrt{2} \tilde{\sigma}}\right)} \le 2\rho \\
	&\implies w_1  \gamma\ge a
	\end{align}
\end{proof}

Next, our proof of Theorem~\ref{thm:mixture} will be based on the following two lemmas. The first lemma shows that $w_1$ is increasing, and the second shows that $\|w_2\|_2$ is decreasing at a fast enough rate. 

\begin{lemma} \label{lem:w1increase}
	\label{lem:w1_inc}In the setting of Theorems~\ref{thm:general_main} and~\ref{thm:mixture}, for any $w \in S$, $$\langle \nabla_{w_1}L(w), w_1 \rangle <0.$$
\end{lemma}

\begin{proof}
	Recall the definition $$g_\sigma(\mu)=\mathbb{E}_{z \sim \cN(\mu, \sigma^2)}\left[\ell_{exp}(z)\right]$$
	Now we can compute
	\begin{align}
	\langle \nabla_{w_1}L(w), w_1 \rangle &=  \mathbb{E}_{x_1}\left[g_\sigma'(w_1^\top x_1)x_1^\top w_1 \right] \\
	&= \mathbb{E}_{x_1}\left[g_\sigma'(w_1^\top x_1)w_1^\top x_1\right] \\
	&<0
	\end{align} since $g_\sigma'(\mu)$ and $\mu$ always have opposite signs.
\end{proof}

\begin{lemma}
	\label{lem:w2_dec} In the same setting as in Lemma~\ref{lem:w1_inc}, we have that 
	for any $w \in S$, $$\langle \nabla_{w_2}L(w), w_2\rangle \ge c||w_2||_2^2$$ for some constant $c$ dependent only on $R, \gamma, \tilde{\sigma}_{min}, \tilde{\sigma}_{max}$.
\end{lemma}

This lemma relies on the following bound stating that for $|\mu| > r(\sigma)$, $q_{\sigma}(\mu) > 0$. 

\begin{lemma}
	\label{lem:dg_dsigma}
	Define $r(\sigma)$ as in~\eqref{eq:r_def}. 
    Then for $|\mu| \ge r(\sigma)$, $q_{\sigma}(\mu) \ge \frac{1}{4} \sigma \lexp(\mu) > 0$. 
\end{lemma}

We prove this lemma in Section~\ref{sec:lexp_bounds}. We also require the following claim that $q_{\tilde{\sigma}}(w_1 \gamma)$ is lower bounded by some positive constant for all $w \in S$. 

\begin{claim}
	\label{claim:dg_dsigma_lb_gaussian}
	Define the function $r(\sigma)$ as in Lemma~\ref{lem:dg_dsigma}. 
	Suppose $a \le \mu \le R \gamma$. Then for all $w \in S$, $$\frac{\partial g_{\tilde{\sigma}}(w_1 \gamma)}{\partial \tilde{\sigma}}\ge c_1$$ for some constant $c_1$ dependent only on $R, \gamma,  \tilde{\sigma}_{min}, \tilde{\sigma}_{max}$.
\end{claim}

\begin{proof}[Proof of Claim~\ref{claim:dg_dsigma_lb_gaussian}]
	We note that $r$ and $q$ satisfy the following properties:
	\begin{enumerate}
	    \item $r(\sigma)$ is a monotonically increasing  increasing function.
		\item $q_{\sigma}(\mu) = \frac{\partial g_{\tilde{\sigma}}(w_1 \gamma)}{\partial \tilde{\sigma}} >0$ for all $\mu \ge r(\sigma)$. (See Lemma~\ref{lem:dg_dsigma} for proof.)
	\end{enumerate}
	For arbitrary $\tilde{\sigma} \in [R\tilde{\sigma}_{min}, R\tilde{\sigma}_{max}]$, $\mu \ge a=r(R\tilde{\sigma}_{max})$ ensures that $\mu \ge r(\tilde{\sigma})$ by property 1. By property 2, $q(\tilde{\sigma}, \mu)>0$. Setting $c_1 = \min_{\tilde{\sigma} \in [R\tilde{\sigma}_{min}, R\tilde{\sigma}_{max}], \mu \in [a, R\gamma]}{q_{\sigma}(\mu)}$ finishes the proof. $c_1$ is dependent only on $R, \gamma, \tilde{\sigma}_{min}, \tilde{\sigma}_{max}$.
\end{proof}

\begin{proof} [Proof of Lemma~\ref{lem:w2_dec}]
    Using the fact that $x_1$ is a uni-variate mixture of Gaussians, and therefore $w^\top x$ is itself a mixture of two Gaussians with variance $\tilde{\sigma}^2$, we have
	\begin{align}
	L(w) &=\mathbb{E}_{y \sim \{\pm 1\},z \sim \cN(0, \tilde{\sigma}^2)}{\left[l_{exp}\left(y w_1 \gamma+z\right)\right]} \\
	&= \frac{1}{2}\left(g_{\tilde{\sigma}}(w_1 \gamma)+g_{\tilde{\sigma}}(-w_1 \gamma)\right) \\
	&= g_{\tilde{\sigma}}(w_1 \gamma) 
	\end{align}
	Now we differentiate with respect to $w_2$ to obtain
	\begin{align}
	\nabla_{w_2}L(w) &= \frac{\partial L(w)}{\partial \tilde{\sigma}} \cdot \frac{\partial \tilde{\sigma}}{\partial w_2} \\
	&= \frac{\partial g_{\tilde{\sigma}}(w_1 \gamma)}{\partial \tilde{\sigma}} \cdot 2\Sigma_2 w_2
	\end{align} 
	We now use the lower bound on $\frac{\partial g_{\tilde{\sigma}}(w_1 \gamma)}{\partial \tilde{\sigma}}$ given by Claim~\ref{claim:dg_dsigma_lb_gaussian}. This gives
	\begin{align}
	\langle \nabla_{w_2}L(w), w_2\rangle \ge 2c_1 \tilde{\sigma}_{min} ||w_2||_2^2.
	\end{align} Setting $c=2c_1 \tilde{\sigma}_{min}$ finishes the proof.
\end{proof}

We can now complete the proof of Theorem~\ref{thm:mixture}.

\begin{proof}[Proof of Theorem~\ref{thm:mixture}] 
    Define $\tilde{w}^{t + 1} = w^{t} - \eta \nabla L(w^t)$. By Lemma \ref{lem:acc_to_a}, $w^0_1  \gamma \ge a$. Note that by assumption $a>r(R\tilde{\sigma}_{max})$. By Lemma \ref{lem:w1_inc} and \ref{lem:w2_dec}, at iteration $t\ge 0$, taking constant step size $\eta$,
	\begin{align}
	|\tilde{w}_1^{t+1}| &>|w_1^{t}| \\
	||\tilde{w}_2^{t+1}||_2^2 &= ||w_2^{t}-\eta \nabla_{w_2}L(w)|_{w=w^{t}}||_2^2 \\
	&= ||w_2^{t}||_2^2 + \eta^2 || \nabla_{w_2}L(w)|_{w=w^{t}}||_2^2 \\
	&-2\eta \langle \nabla_{w_2}L(w)|_{w=w^{t}}, w_2^{t} \rangle
	\end{align}
	By Lemma \ref{lem:w2_dec}, $$\nabla_{w_2}L(w)=q_{\sigma(w)}(\mu(w)) \cdot 2 \Sigma_2 w_2$$ for some continuous function $q_{\sigma}(\mu)$ where $\sigma(w)^2 =w_2^\top \Sigma_2 w_2$, $\mu(w)=w_1 \gamma$ over compact set $S$. Therefore $q$ is bounded. Suppose $|q| \le c_2$ for all $w \in S$, then $$||\nabla_{w_2}L(w)|_{w=w^{t}}||_2^2\le 4 c_2^2 \tilde{\sigma}_{max}^2 ||w_2^{t}||_2^2.$$
	
	Therefore $||\tilde{w}_2^{t+1}||_2^2 \le c'||w_2^{t}||_2^2$ for some constant $c'<1$ for appropriate choice of $\eta$. As $|\tilde{w}_1^{t + 1}| > |w_1^t|$, renormalization results in some constant factor decrease in $\|w_2^{t}\|_2$. Therefore $||w_2^{t}||_2^2 \le \epsilon$ when $t \ge K=O(\log{\left(\frac{1}{\epsilon}\right)})$.
\end{proof}

\newpage

\section{Missing Proofs for Theorem~\ref{thm:general_main}}\label{sec:general:app}


Define the following function depending on parameters $\smooth, \concave$ defined later: 
\begin{align} \label{eq:kappa_def}
\widetilde{\kappa}(\smooth, \concave) \triangleq \min\left \{ \frac{\sqrt{\pi}}{4 \sqrt{\smooth}} (p^\star(\smooth, \concave))^{1 - \frac{\concave}{4\smooth}} \left( \frac{\concave}{2 \sqrt{\pi}}\right)^{\frac{\concave}{4 \smooth}}, \frac{\sqrt{\concave}}{8 \sqrt{2\pi}(\sqrt{\smooth} + \sqrt{2})}\exp\left(-\left( \frac{\sqrt{\smooth} + 4}{2 \sqrt{\concave}}\right)^2\right) \right\}
\end{align}
for 
$$p^\star(\smooth, \concave) \triangleq \frac{\sqrt{\concave}}{2\sqrt{\pi}} \min\left \{ 1, \frac{\sqrt{\concave} }{\sqrt{\smooth}} \left( \frac{\sqrt{\pi}}{44 \sqrt{2\smooth}} \right)^{\frac{8\smooth}{\concave}} \right \}$$

Now we choose the constant $\kappa$ in Assumptions~\ref{ass:seperation} and~\ref{ass:boundinit} to be $\kappa(\beta, \alpha) \triangleq \min_{a \in [1, 4]} \widetilde{\kappa}(a \beta, a \alpha)$. 

Throughout the proof, we will use $p$ to refer to the density of $\mu = w_1^\top x_1$, and define $s(\mu) \triangleq \frac{\partial}{\partial \mu} \log p(\mu)  = \frac{p'(\mu)}{p(\mu)}$. Throughout our proofs, we use $\concave$, $\smooth$ to refer to parameters such that $p$ is $\concave$-log-concave and $\smooth$-log-smooth. By Assumption~\ref{ass:seperation}, we have that the density of $\frac{w_1^\top x_1}{\|w_1\|_2}$ is $\alpha$-log-concave and $\beta$-log-smooth, so we can choose $\concave = \alpha/\|w_1\|_2^2$ and $\smooth = \beta/\|w_1\|_2^2$ by the linear transformation formula of a probability density. We use $\sigma = w_2^\top \Sigma_2 w_2$ to be the variance of the output of the current classifier restricted to the spurious coordinates.

\subsection{Proof overview}
\label{sec:general_case_overview}

We will argue that if the initial conditions in Assumption~\ref{ass:boundinit} hold, then they will continue to hold throughout training. Furthermore, under these initial conditions, the loss gradient will force $\|w_2\|_2$ to decrease. 

The following three lemmas which analyze a single update of the algorithm will form the main technical core of our proof. They will be used to show that when the loss is sufficiently small, the norm of $w_2$ is always decreasing. The first lemma states that if the loss is small, then $p$ cannot have a large density at 0. 
\begin{lemma}\label{lem:Ltop}
	In the setting of Theorem~\ref{thm:general_main}, suppose that $K=1$. When $L(w) \le \widetilde{\kappa}(\rho, \nu)$, we must have
	\begin{align}
	L(w) 
	&\ge \frac{\sqrt{\pi}}{4\sqrt{\smooth}} \density(0)^{1 - \frac{\concave}{4\smooth}} \left( \frac{\sqrt{\concave}}{2 \sqrt{\pi}}\right)^{\frac{\concave}{4\smooth}} \label{eq:sigma_decrease:4}
	\end{align}
	As a consequence, when $L(w)\le \widetilde{\kappa}(\rho, \nu)$, we have:
		\begin{align}
	p(0) \le p^\star(\rho, \nu)
	\end{align}
\end{lemma}

Next, observe that $w_2$ will decrease if $\frac{\partial}{\partial \sigma} L(w) > 0$. The following lemma lower bounds $\frac{\partial}{\partial \sigma} L(w) > 0$ in terms of $p(0)$, showing that if $p(0)$ is small, $\frac{\partial}{\partial \sigma} L(w)$ will be positive. 
\begin{lemma}\label{lem:dLdstop}
	In the setting of Lemma~\ref{lem:Ltop}, we have
	\begin{align} \label{eq:dLdstop:1}
	\frac{\partial}{\partial \sigma} L(w) &\ge \density(0) \sigma  \left(\frac{ \sqrt{\pi}}{11\sqrt{\smooth}}  \left(\frac{\sqrt{\concave}}{2 \density(0) \sqrt{\pi}}\right)^{\frac{\concave}{4\smooth}} -  2\sqrt{2} \max\left \{1, \left(\frac{\sqrt{\smooth}}{2 \density(0) \sqrt{\pi}} \right)^{\frac{\concave}{8\smooth}} \right \}  \right)
\end{align}
\end{lemma}

Finally, this next lemma combines the two lemmas above to show that when the loss is sufficiently small, $w_2$ is always shrinking. 
\begin{lemma}\label{lem:dldsigma}
	In the setting of Lemma~\ref{lem:Ltop}, when $L(w) \le \widetilde{\kappa}(\rho, \nu)$ for $\widetilde{\kappa}(\rho, \nu)$ defined in~\eqref{eq:kappa_def}, we have
	\begin{align}
	\frac{\partial}{\partial \sigma} L(w) \ge \sigma \density(0)^{1 - \frac{\concave}{4\smooth}} \frac{ \sqrt{\pi}}{22\sqrt{\smooth}}  \left(\frac{\sqrt{\concave}}{2  \sqrt{\pi}}\right)^{\frac{\concave}{4\smooth}} > 0
	\end{align}
\end{lemma}
\begin{proof}[Proof of Lemma~\ref{lem:dldsigma}]
	For simplicity, define $a_1 \triangleq \frac{ \sqrt{\pi}}{11\sqrt{\smooth}}  \left(\frac{\sqrt{\concave}}{2 \density(0) \sqrt{\pi}}\right)^{\frac{\concave}{4\smooth}}$ and $a_2 \triangleq 2\sqrt{2} \max\left \{1, \left(\frac{\sqrt{\smooth}}{2 \density(0) \sqrt{\pi}} \right)^{\frac{\concave}{8\smooth}} \right\}$, so that the right hand side of~\eqref{eq:dLdstop:1} becomes $\density(0) \sigma(a_1 - a_2)$. 
	
	Now we apply Lemma~\ref{lem:Ltop} to conclude that when $L(w) \le \widetilde{\kappa}(\smooth, \concave)$, $\density(0) \le p^\star(\rho, \nu)$. Note that when $\density(0) \le  \frac{\sqrt{\concave}}{2\sqrt{\pi}} \min\left \{ \left(\frac{\sqrt{\pi}}{44 \sqrt{2\smooth}}\right)^{\frac{4\smooth}{\concave}}, \frac{\sqrt{\concave} }{\sqrt{\smooth}} \left( \frac{\sqrt{\pi}}{44 \sqrt{2\smooth}} \right)^{\frac{8\smooth}{\concave}} \right \}$, we must have $a_1 \ge 2 a_2$ by the definitions of $a_1, a_2$. Furthermore, the r.h.s. of this bound is lower-bounded by $p^\star(\rho, \nu)$. As a result, when $\density(0) \le p^\star$, by Lemma~\ref{lem:dLdstop}, we have $\frac{\partial}{\partial \sigma} L(w) \ge \density(0) \sigma \frac{a_1}{2}$. Applying the definition of $a_1$ gives the desired result.
\end{proof}

\subsection{Proof of Lemmas~\ref{lem:Ltop}}

The following claim will be useful for proving both Lemma~\ref{lem:Ltop} and Lemma~\ref{lem:dLdstop}.
\begin{claim}\label{claim:integral_s0}
	Recall that we
	defined $\dratio(\mu) = \frac{\partial}{\partial \mu} \log \density(\mu) = \frac{\density'(\mu)}{\density(\mu)}$. The following bound holds:
	\begin{align}
	\int_{0}^{\infty}  \exp \left( \frac{|\dratio(0)|}{2} \delta - \frac{\smooth}{2} \delta^2 \right) d\delta \ge \frac{\sqrt{\pi}}{\sqrt{\smooth}} \left(\frac{\sqrt{\concave}}{2 \density(0) \sqrt{\pi}}\right)^{\frac{\concave}{4\smooth}}
	\end{align}
\end{claim}

\begin{proof}[Proof of Claim~\ref{claim:integral_s0}]
	From Lemma~\ref{lem:dratio_bound}, we start with
	\begin{align}
	\int_{0}^{\infty}  \exp \left( \frac{|\dratio(0)|}{2} \delta - \frac{\smooth}{2} \delta^2 \right) d\delta \ge  \frac{\sqrt{\pi} \exp \left( \frac{\dratio(0)^2}{4\smooth}\right) }{\sqrt{\smooth}} 
	\end{align}
	Now we apply the lower bound $\dratio(0)^2 \ge \concave \log \left(\frac{\sqrt{\concave}}{2 \density(0) \sqrt{\pi}}\right)$ and obtain 
	\begin{align}
	\int_{0}^{\infty}  \exp \left( \frac{|\dratio(0)|}{2} \delta - \frac{\smooth}{2} \delta^2 \right) d\delta &\ge  \frac{\sqrt{\pi}}{\sqrt{\smooth}} \exp\left (\frac{\concave}{4 \smooth} \log \left(\frac{\sqrt{\concave}}{2 \density(0) \sqrt{\pi}}\right) \right ) \\
	&\ge \frac{\sqrt{\pi}} {\sqrt{\smooth}} \left(\frac{\sqrt{\concave}}{2 \density(0) \sqrt{\pi}}\right)^{\frac{\concave}{4\smooth}}
	\end{align}
\end{proof}

Our starting point is to first lower bound $L(w)$ in terms of $p(0)$ and $s(0)$. 

\begin{claim}\label{lem:loss_lb}
	The following lower bound on the loss $L(w)$ holds:
	\begin{align}
	L(w) \ge \constl \density(0)\int_{0}^{\infty} \exp \left ((|\dratio(0)| - 1) \delta - \frac{\smooth}{2} \delta^2\right) d\delta
	\end{align}
\end{claim}

\begin{proof}[Proof of Claim~\ref{lem:loss_lb}]
	Without loss of generality, assume that $\dratio(0) \ge 0$ (otherwise, by symmetry of $\lexp$ the same arguments hold). Then we have 
	\begin{align}
	L(w) &= \int_{-\infty}^{\infty} \density(\delta) \expent_{\sigma}(\delta) d\delta \\
	&\ge \density(0) \int_{-\infty}^{\infty} \exp\left (\dratio(0) \delta - \frac{\smooth}{2} \delta^2\right) \expent_{\sigma}(\delta) \tag{by log-smoothness}\\
	&\ge \constl \density(0)\int_{-\infty}^{\infty} \exp\left (\dratio(0) \delta - \frac{\smooth}{2} \delta^2\right) \lexp(\delta) d\delta \tag{by Lemma~\ref{lem:exp_ent_deriv_bound}}\\
	&\ge \constl \density(0) \int_{0}^{\infty} \exp \left ((|\dratio(0)| - 1) \delta - \frac{\smooth}{2} \delta^2\right) d\delta \tag{substituting $\lexp(\delta) = \exp(-\delta)$ for $\delta \ge 0$}
	\end{align}
	Now the loss is symmetric around $0$, so the same argument would also work for $\dratio(0) < 0$. Thus, we obtain the desired result. 
\end{proof}

Next, we argue that if $L(w)$ is bounded above by some threshold, then $s(0)$ will be large in absolute value.

\begin{claim} \label{lem:dratio_lb}
	Suppose that our classifier $w$ satisfies the following loss bound: 
	\begin{align}\label{eq:dratio_lb:1}
		L(w) \le \frac{\sqrt{\concave}}{8 \sqrt{2\pi}(\sqrt{\smooth} + \sqrt{2})}\exp\left(-\left( \frac{\sqrt{\smooth} + 4}{2 \sqrt{\concave}}\right)^2\right)
	\end{align} 
	Then $|\dratio(0)| \ge \frac{\sqrt{\smooth}}{2} + 2$.
\end{claim}

\begin{proof}[Proof of Claim~\ref{lem:dratio_lb}]
	Assume for the sake of contradiction that $|\dratio(0)| \le \frac{\sqrt{\smooth}}{2} + 2$. First, we consider the case when $\dratio(0) \in [1, \frac{\sqrt{\smooth}}{2} + 2]$. In this case, by Lemma~\ref{lem:dratio_bound}, we have 
	\begin{align}
	\density(0) \ge \frac{\sqrt{\concave}}{2 \sqrt{\pi}} \exp\left(-\left( \frac{\sqrt{\smooth} + 4}{2 \sqrt{\concave}}\right)^2\right)
	\end{align}
	Furthermore, in this case we also have $|s(0)|  - 1 > 0$, so we can apply~\eqref{eq:integral:3} from Claim~\ref{claim:integral}. Plugging into Claim~\ref{lem:loss_lb}, we obtain 
	\begin{align}
	L(w)  &\ge \constl \density(0) \frac{\sqrt{\pi} \exp\left(\frac{(|\dratio(0)| - 1)^2}{\smooth}\right)}{\sqrt{\smooth}}\\
	&\ge \frac{\sqrt{\concave }}{8 \sqrt{\smooth}}\exp\left(-\left( \frac{\sqrt{\smooth} + 4}{2 \sqrt{\concave}}\right)^2\right)
	\end{align}
	In the other case where $0 \le |s(0)| \le 1$, by Claim~\ref{claim:integral} and Claim~\ref{lem:loss_lb}, we first have
	\begin{align}
		L(w) \ge \constl \frac{\sqrt{\pi}}{\sqrt{\smooth}} \density(0) \exp\left (\frac{(|\dratio(0)| - 1)^2}{\smooth}\right) \left(\erf \left(\frac{\dratio(0) - 1}{\sqrt{\smooth}}\right) + 1 \right)
	\end{align} 
	Now applying the lower bound on $\density(0)$ from Lemma~\ref{lem:dratio_bound}, we have
	\begin{align}
	L(w) \ge  \frac{\sqrt{\concave}}{8 \sqrt{\smooth}} \exp (-\dratio(0)^2/\concave) \exp\left (\frac{(|\dratio(0)| - 1)^2}{\smooth}\right) \left(\erf \left(\frac{\dratio(0) - 1}{\sqrt{\smooth}}\right) + 1 \right)
	\end{align}
	Now by Claim~\ref{claim:erf_lb}, we have 
	\begin{align}
	\erf \left(\frac{|\dratio(0)| - 1}{\sqrt{\smooth}}\right) + 1 \ge \frac{1}{\sqrt{\pi}} \frac{\exp\left (-\frac{(|\dratio(0)| - 1)^2}{\smooth}\right)}{\sqrt{2} + 2(1 - |\dratio(0)|)/\sqrt{\smooth}} 
	\end{align}
	Thus, we have 
	\begin{align}
	L(w) &\ge \frac{\sqrt{\concave}}{8 \sqrt{\pi}(\sqrt{2 \smooth} + 2(1 - |\dratio(0)|))} \exp\left(-\frac{\dratio(0)^2}{\concave}\right)\\
	&\ge \frac{\sqrt{\concave}}{8 (\sqrt{2\pi \smooth} + 2\sqrt{\pi})} \exp(-1/\nu)
	\end{align}
	Combining the two cases allows us to conclude that if $|\dratio(0)| < \frac{\sqrt{\smooth}}{2} + 2$, the loss must satisfy 
	\begin{align}
	L(w) \ge \min \left \{ \frac{\sqrt{\concave }}{8 \sqrt{\smooth}}\exp\left(-\left( \frac{\sqrt{\smooth} + 4}{2 \sqrt{\concave}}\right)^2\right),\frac{\sqrt{\concave}}{8 (\sqrt{2\pi \smooth} + 2\sqrt{\pi})} \exp(-1/\nu) \right \}
	\end{align}
	Now we note that the r.h.s. of the above equation is lower bounded by the r.h.s of~\eqref{eq:dratio_lb:1}. Thus, the loss would violate~\eqref{eq:dratio_lb:1}, a contradiction. 
\end{proof}

\begin{proof}[Proof of Lemma~\ref{lem:Ltop}]
	First, by Claim~\ref{lem:dratio_lb}, when $L(w) \le \widetilde{\kappa}(\rho, \nu)$, we must have $|\dratio(0)| \ge \frac{\sqrt{\smooth}}{2} + 2$. Now we lower bound the loss in terms of $\density(0)$. Starting from Claim~\ref{lem:loss_lb}, we have
	\begin{align}
	L(w) &\ge \constl \density(0) \int_{0}^{\infty} \exp \left( (|\dratio(0)| - 1) \delta - \frac{\smooth}{2} \delta^2 \right) d\delta \\
	&\ge \constl \density(0) \int_{0}^{\infty}  \exp \left( \frac{|\dratio(0)|}{2} \delta - \frac{\smooth}{2} \delta^2 \right) d\delta \tag{since $|\dratio(0)| \ge 2$}
	\end{align}
	Now applying Claim~\ref{claim:integral_s0}, we obtain 
	\begin{align}
	L(w) &\ge \constl \frac{\sqrt{\pi}} {\sqrt{\smooth}}\density(0) \left(\frac{\sqrt{\concave}}{2 \density(0) \sqrt{\pi}}\right)^{\frac{\concave}{4\smooth}}\\
	&\ge \constl \frac{\sqrt{\pi}}{\sqrt{\smooth}} \density(0)^{1 - \frac{\concave}{4\smooth}} \left( \frac{\sqrt{\concave}}{2 \sqrt{\pi}}\right)^{\frac{\concave}{4\smooth}}
	\end{align}
	This completes the first part of the lemma. For the second part, we note that if $L(w) \le \widetilde{\kappa}(\smooth, \concave)$, then $L(w)$ is bounded above by the r.h.s. of~\eqref{eq:sigma_decrease:4} with $p^\star(\rho, \nu)$ substituted for $p(0)$ by the definition of $\widetilde{\kappa}(\smooth, \concave)$. Combined with the first part of the lemma, this immediately gives $p(0) \le p^\star(\rho, \nu)$.
\end{proof}

\subsection{Proof of Lemma~\ref{lem:dLdstop}}

We rely on the following lemma which lower bounds $\frac{\partial}{\partial \sigma} L(w)$. 
\begin{lemma} \label{lem:deriv_lb}
	Suppose $\sigma \le \frac{4 \sqrt{2}}{\sqrt{\pi}}$ satisfies $\sigopt \le \min\left \{1, \frac{1}{4\sqrt{\smooth}}\right\}$ for $\sigopt \triangleq \sigma^2 + \sigma \sqrt{2 \log \frac{4 \sqrt{2}}{\sqrt{\pi} \sigma}}$. Then the following lower bound on the derivative $\frac{\partial}{\partial \sigma} L(w)$ holds:
	\begin{align}
	&\frac{\partial}{\partial \sigma} L(w) \ge \nonumber \\ &\density(0) \left(\frac{\sigma}{11} \int_{0}^{\infty} \exp \left( \left(|\dratio(0)| - \frac{\sqrt{\smooth}}{4} - 1\right)\delta - \frac{\smooth}{2} \delta^2\right) d\delta- 2\sqrt{2}  \exp\left(\frac{\dratio(0)^2}{\concave + \sigma^{-2}}\right)\frac{1}{\sqrt{\concave + \sigma^{-2}}}\right) \label{eq:deriv_lb:4}
	\end{align}
\end{lemma}

\begin{proof}[Proof of Lemma~\ref{lem:deriv_lb}]
	We compute $\frac{\partial}{\partial \sigma} L(w)$ in two parts: 
	\begin{align} \label{eq:deriv_lb:1}
	\frac{\partial}{\partial \sigma} L(w) = \int_{|\mu| \le \sigopt} \density(\mu) \frac{\partial}{\partial \sigma} \expent_{\sigma}(\mu) + \int_{|\mu| > \sigopt} \density(\mu) \frac{\partial}{\partial \sigma} \expent_{\sigma}(\mu)
	\end{align}
	For $|\mu| \le \sigopt$, we lower bound the integral using~\eqref{eq:exp_ent_deriv_bound:1} in Lemma~\ref{lem:exp_ent_deriv_bound}. For $|\mu| > \sigopt$, we lower bound the integral using Lemma~\ref{lem:dg_dsigma}. 
	By~\eqref{eq:exp_ent_deriv_bound:1}, we have 
	\begin{align}
	\int_{|\mu| \le \sigopt} \density(\mu) \frac{\partial}{\partial \sigma} \expent_{\sigma}(\mu) &\ge -\sqrt{\frac{2}{\pi}} \int_{|\mu| \le \sigopt} \density(\mu) \exp\left( -\frac{\mu^2}{2\sigma^2} \right) d\mu \\
	&\ge -\sqrt{\frac{2}{\pi}} \int_{|\delta| \le \sigopt} \density(0) \exp\left(\dratio(0) \delta - \frac{\concave}{2} \delta^2\right) \exp\left( -\frac{\delta^2}{2\sigma^2} \right) d\delta \tag{by log-strong concavity}\\
	&\ge -\density(0) \sqrt{\frac{2}{\pi}} \int_{-\infty}^{\infty} \exp \left(\dratio(0) \delta - \left(\frac{\concave}{2} + \frac{1}{2 \sigma^2}\right) \delta^2\right) d\delta \\
	&= -p(0) 2\sqrt{2}  \exp\left(\frac{\dratio(0)^2}{\concave + \sigma^{-2}}\right)\frac{1}{\sqrt{\concave + \sigma^{-2}}} \label{eq:deriv_lb:2}
	\end{align}
	
	We obtained the last equation via Claim~\ref{claim:integral}. Now we lower bound the second integral in~\eqref{eq:deriv_lb:1}. By Lemma~\ref{lem:dg_dsigma}, $\frac{\partial}{\partial \sigma} \expent_{\sigma}(\mu) \ge \frac{\sigma}{4} \lexp(\mu) > 0$ for $|\mu| > \sigopt$. Assume without loss of generality that $\dratio(0) > 0$ (so we restrict our attention to $\mu > \sigopt > 0$). By symmetry, our arguments still hold if $\dratio(0) < 0$. Now we have 
	\begin{align}
	\int_{|\mu| > \sigopt} \density(\mu) \frac{\partial}{\partial \sigma} \expent_{\sigma}(\mu) &> \int_{\delta > 0} \density(\sigopt + \delta) \frac{\partial}{\partial \sigma} \expent_{\sigma}(\sigopt + \delta)\\
	&\ge \frac{\sigma}{4} \int_{\delta > 0} \density(\sigopt + \delta) \lexp(\sigopt + \delta)\\
	&\ge \frac{p(0) \sigma}{4} \int_{\delta > 0} \exp \left(\dratio(0) (\sigopt + \delta) - \frac{\smooth}{2} (\delta + \sigopt)^2 \right) \exp(-\delta - \sigopt) d\delta \label{eq:deriv_lb:3}
	\end{align}  
	Now we note that for $\sigopt$ satisfying $\sqrt{\smooth} \sigopt \le \frac{1}{4}$ and $\delta > 0$, we have 
	\begin{align}
	\dratio(0) (\sigopt + \delta) - \frac{\smooth}{2} (\delta + \sigopt)^2 > \dratio(0) \delta - \frac{\smooth}{2} \delta^2 - \smooth \delta \sigopt - \frac{\smooth}{2} {\sigopt}^2 \ge \left(\dratio(0) - \frac{\sqrt{\smooth}}{4}\right) \delta - \frac{\smooth}{2} \delta^2 - \frac{1}{32}
	\end{align}
	As a result, plugging this back into~\eqref{eq:deriv_lb:3} gives 
	\begin{align}
	\int_{|\mu| > \sigopt} \density(\mu) \frac{\partial}{\partial \sigma} \expent_{\sigma}(\mu) &> \exp \left( - \sigopt - \frac{1}{32}\right) \frac{\density(0) \sigma}{4} \int_{0}^{\infty} \exp \left( \left(\dratio(0) - \frac{\sqrt{\smooth}}{4} - 1\right)\delta - \frac{\smooth}{2} \delta^2\right) d\delta.
	\label{eq:deriv_lb:5}
	\end{align}
	Now we use the fact that $\sigopt \le 1$ to lower bound $\exp(-\sigopt)$. Finally, we obtain~\eqref{eq:deriv_lb:4} by combining~\eqref{eq:deriv_lb:2} and~\eqref{eq:deriv_lb:5}.
\end{proof}

Now we complete the proof of Lemma~\ref{lem:dLdstop}.
\begin{proof}[Proof of Lemma~\ref{lem:dLdstop}]
	Now we proceed to lower bound $\frac{\partial}{\partial \sigma} L(w)$. Our starting point is Lemma~\ref{lem:deriv_lb}. We will lower bound the first integral:
	\begin{align}
	\int_{0}^{\infty} \exp \left( \left(|\dratio(0)| - \frac{\sqrt{\smooth}}{4} - 1\right)\delta - \frac{\smooth}{2} \delta^2\right) d\delta &\ge \int_{0}^{\infty} \exp \left( \left(\frac{|\dratio(0)|}{2} \right)\delta - \frac{\smooth}{2} \delta^2\right) d\delta \tag{using $|\dratio(0)| \ge \frac{\sqrt{\smooth}}{2} + 2$}\\
	&\ge \frac{\sqrt{\pi}}{\sqrt{\smooth}} \left(\frac{\sqrt{\concave}}{2 \density(0) \sqrt{\pi}}\right)^{\frac{\concave}{4\smooth}} \tag{from Claim~\ref{claim:integral_s0}}
	\end{align}
	Applying this with equation~\eqref{eq:deriv_lb:4} in Lemma~\ref{lem:deriv_lb}, we obtain 
	\begin{align}
	\frac{\partial}{\partial \sigma} L(w) \ge \density(0) \left(\frac{\sigma \sqrt{\pi}}{11\sqrt{\smooth}}  \left(\frac{\sqrt{\concave}}{2 \density(0) \sqrt{\pi}}\right)^{\frac{\concave}{4\smooth}} -  2\sqrt{2}  \exp\left(\frac{\dratio(0)^2}{\concave + \sigma^{-2}}\right)\frac{1}{\sqrt{\concave + \sigma^{-2}}} \right) \label{eq:sigma_decrease:3}
	\end{align}
	Now we lower bound the second term in~\eqref{eq:sigma_decrease:3}. By applying the upper bound on $\dratio(0)$ in Lemma~\ref{lem:dratio_bound}, we obtain 
	\begin{align}
	\exp\left(\frac{\dratio(0)^2}{\concave + \sigma^{-2}}\right) &\le \exp \left(\frac{\smooth}{\concave + \sigma^{-2}} \log \left(\frac{\sqrt{\smooth}}{2 \density(0) \sqrt{\pi}} \right)\right)\\
	&\le \left(\frac{\sqrt{\smooth}}{2 \density(0) \sqrt{\pi}} \right)^{\frac{\smooth}{\concave + \sigma^{-2}}}
	\end{align}
	Plugging this back into~\eqref{eq:sigma_decrease:3} and observing that $\frac{1}{\sqrt{\concave + \sigma^{-2}}} \le \sigma$, we obtain
	\begin{align}
	\frac{\partial}{\partial \sigma} L(w) &\ge \density(0) \sigma  \left(\frac{ \sqrt{\pi}}{11\sqrt{\smooth}}  \left(\frac{\sqrt{\concave}}{2 \density(0) \sqrt{\pi}}\right)^{\frac{\concave}{4\smooth}} -  2\sqrt{2} \left(\frac{\sqrt{\smooth}}{2 \density(0) \sqrt{\pi}} \right)^{\frac{\smooth}{\concave + \sigma^{-2}}}  \right)
	\end{align}
	
	Now suppose that the condition $\sigma^2 {\smooth}^2 / \concave \le 1/8$ holds. Then $\frac{\smooth}{\concave + \sigma^{-2}} < \frac{\concave}{8 \smooth}$, so $\left(\frac{\sqrt{\smooth}}{2 \density(0) \sqrt{\pi}} \right)^{\frac{\smooth}{\concave + \sigma^{-2}}} \le \max\left\{1, \left(\frac{\sqrt{\smooth}}{2 \density(0) \sqrt{\pi}} \right)^{\frac{\concave}{8\smooth}}\right \}$. 
	It follows that 
	\begin{align}
	\frac{\partial}{\partial \sigma} L(w) &\ge \density(0) \sigma  \left(\frac{ \sqrt{\pi}}{11\sqrt{\smooth}}  \left(\frac{\sqrt{\concave}}{2 \density(0) \sqrt{\pi}}\right)^{\frac{\concave}{4\smooth}} -  2\sqrt{2}\max\left \{1, \left(\frac{\sqrt{\smooth}}{2 \density(0) \sqrt{\pi}} \right)^{\frac{\concave}{8\smooth}} \right \}  \right)
	\end{align}
\end{proof}

\subsection{Proof of Theorem~\ref{thm:general_main}}
\label{sec:general_proof}
We first show that the loss must be lower-bounded by some constant depending only on the distribution over $x_1$ if $w_1$ is bounded. 
\begin{lemma} \label{lem:smallest_possible_loss}
	For any $w$ with $\|w_1\|_2 \le R$, as long as $\sigma \le 1$ the following holds: 
	\begin{align}
	L(w) \ge \constl \exp(-R \E_{x_1}[\|x_1\|_2])
	\end{align}
\end{lemma}
\begin{proof}
	We have 
	\begin{align}
	L(w) &= \E_{x_1}[\expent_{\sigma}(w_1^\top x_1)]\\
	&\ge \E_{x_1}[\expent_{\sigma}(\|w_1\|_2 \|x_1\|_2)]\\
	&\ge 0.25 \E_{x_1} [\lexp(\|w_1\|_2 \|x_1\|_2)] \tag{by Lemma~\ref{lem:exp_ent_deriv_bound}}\\
	&\ge 0.25 \exp(-R \E_{x_1}[\|x_1\|_2])
	\end{align}
	The last line followed because $\lexp(\mu) = \exp(-\mu)$ for positive $\mu$. Since $\exp(-\mu)$ is convex, we applied Jensen's inequality.
\end{proof}

Next, we argue that $p(0)$ must be lower-bounded by some constant depending only on the distribution over $x_1$ if $w_1$ is bounded.
\begin{lemma}\label{lem:loss_ub}
	There exists some constant $c_1$ which only depends on $\E_{x_1}[\|x_1\|_2], \alpha, \beta$ such that for all $w$ satisfying $\sigma \le 1/2$ and $1/2 \le \|w_1\|_2 \le 1$, we have
	\begin{align}
	\density(0) \ge c_1. 
	\end{align}
\end{lemma}
\begin{proof}
	Fix $\bar{\mu} = \log (4/L(w))$. Then note that we must have $\int_{-\bar{\mu}}^{\bar{\mu}} \density(\mu) d\mu + \int_{|\mu| > \bar{\mu}} \density(\mu) \max_{|\mu| > \bar{\mu}} \expent_{\sigma}(\mu) \ge L(w)$. Now note that since $\sigma \le 1/2$, Lemma~\ref{lem:loss_ub} tells us that $\expent_{\sigma}(\mu) \le 2 \exp(-\mu)$. Thus, $\max_{|\mu| > \bar{\mu}} \expent_{\sigma}(\mu) \le L(w)/2$. Thus, we obtain 
	\begin{align}
	\int_{-\bar{\mu}}^{\bar{\mu}} \density(\mu) d\mu + \left (1 - \int_{-\bar{\mu}}^{\bar{\mu}} \density(\mu) d\mu \right) \frac{L(w)}{2} \ge L(w)
	\end{align}
	This gives 
	\begin{align}
	\int_{-\bar{\mu}}^{\bar{\mu}} \density(\mu) d\mu \ge \frac{L(w)/2}{1 - L(w)/2}
	\end{align}
	Thus, we can conclude that there exists $\mu' \in [-\bar{\mu}, \bar{\mu}]$ such that $\density(\mu') \ge \frac{L(w)}{2 \bar{\mu}(2 - L(w))}$. 
	
	Now we apply Lemma~\ref{lem:dratio_bound} to obtain 
	\begin{align}
	|\dratio(\mu')| \le \sqrt{\smooth \log \left(\frac{\sqrt{\smooth} \bar{\mu}(2 - L(w))}{L(w) \sqrt{\pi}}\right)}
	\end{align}
	Now we apply Claim~\ref{claim:concave_smooth} to conclude that 
	\begin{align}
	\density(0) &\ge \density(\mu') \exp\left(-|\dratio(\mu')| \bar{\mu} - \frac{\smooth}{2} {\bar{\mu}}^2\right)
	\end{align}
	Now note that $s(\mu')$, $p(\mu')$, $\bar{\mu}$ depend only on $L(w)$ which is upper bounded by 1 and lower bounded by some function of $\E_{x_1}[\|x_1\|_2]$ by Lemma~\ref{lem:smallest_possible_loss}. Furthermore, $\smooth \in [\beta, 2\beta]$. Thus, $s(\mu')$, $p(\mu')$, $\bar{\mu}$ are all upper and lower bounded by some function of $\E_{x_1}[\|x_1\|_2]$. As a result, the same applies to $p(0)$, giving  us the desired statement. 
\end{proof}

\begin{proof}[Proof of Theorem~\ref{thm:general_main}]
	We start with proving the case when $K=1$. First, we note that $w_1^\top \nabla_{w_1} L(w) < 0$ by using the same argument as Lemma~\ref{lem:w1increase}. Furthermore, for all $\|w_1\|_2 \in [1/2, 1]$, the upper bound on $\sigma$ required for Lemmas~\ref{lem:Ltop},~\ref{lem:dLdstop}, and~\ref{lem:dldsigma} are all satisfied by Assumption~\ref{ass:boundinit}. Thus, if $L(w^s) \le \kappa(\beta, \alpha) = \min_{a \in [1, 4]} \widetilde{\kappa}(a \beta, a\alpha)$, then initially the loss upper bound for Lemmas~\ref{lem:Ltop},~\ref{lem:dLdstop}, and~\ref{lem:dldsigma} are satisfied. Combining this with Lemma~\ref{lem:loss_ub}, we get that $\frac{\partial}{\partial \sigma} L(w) \ge c_1 \sigma$ for the constant $c_1$ defined in Lemma~\ref{lem:loss_ub}. Furthermore, the loss $L(w)$ is also always decreasing for sufficiently small choice of step size. As a result, the following invariants hold throughout the optimization algorithm: 
		 $\|w_1\|_2$ is non-decreasing, $\sigma$ is non-increasing, and $L(w) \le \kappa(\beta, \alpha)$. 
	Thus, the initial conditions Lemmas~\ref{lem:Ltop},~\ref{lem:dLdstop}, and~\ref{lem:dldsigma} will always hold, so we can conclude using the same argument as in Lemma~\ref{lem:w2_dec} that $\|w_2\|_2$ is always decreasing with rate $c_2 \|w_2\|_2$, where $c_2$ is some value depending on $\alpha, \beta$, and the data distribution. This implies that $w_2$ converges to 0, providing the first statement in Theorem~\ref{thm:general_main}.

	 Finally, in the case that $K > 1$, we observe that when Assumption~\ref{ass:boundinit} is satisfied, we must have $L_i(w) \le \kappa( \beta, \alpha)$ for all $i$, where $L_i(w)$ is the expectation of the loss conditioned on the $i$-th mixture component. Thus, this immediately reduces to the $K = 1$ case. 
\end{proof}

To prove convergence of noisy gradient descent to an approximate local minimum of the objective 
\begin{align}
    \min_{\|w_1\|_2 \le 1}{L((w_1, 0))}.
    \label{eq:purified}
\end{align} 
we also assume that $L((w_1, w_2))$ is twice-differentiable, and furthermore there exists $C$ such that $\nabla_{w_1} L((w_1, w_2))$, $\nabla_{w_1}^2 L((w_1, w_2))$ are Lipschitz in $w_2$ for $\|w_2\|_2 \le C$, $\|w_1\|_2 \le 1$. 

We will first formally define an $(\epsilon, \gamma)$-approximate local minimum of~\eqref{eq:purified}. Define $P_{w_1^{\perp}} \triangleq I - \frac{w_1 w_1^\top}{\|w_1\|_2^2}$ to be the projection onto the space orthogonal to $w_1$. Then an $(\epsilon, \gamma)$-approximate local minimum of~\eqref{eq:purified} is a point $w_1 : \|w_1\|_2 \le 1$ satisfying: 
\begin{enumerate}
	\item $\|w_1\|_2 \ge 1 - \epsilon$.
	\item $\|P_{w_1^\perp} \nabla_{w_1} L((w_1, 0))\|_2 \le \epsilon$. 
	\item $P_{w_1^\perp} \nabla^2_{w_1} L((w_1, 0))P_{w_1^\perp} - (w_1^\top \nabla_{w_1} L((w_1, 0))) P_{w_1^\perp} \succeq -\gamma I$. 
\end{enumerate}
Note that the first condition simply reflects the fact that all true local minima of~\eqref{eq:purified} will satisfy $\|w_1\|_2 = 1$ and therefore lie on the unit sphere $\mathbb{S}^{d_1 - 1}$, as scaling up the weights only decreases the objective. The second two conditions essentially adapt the classical conditions for approximate local minima (see~\citep{nesterov2006cubic,agarwal2017finding}) to the setting where the domain is a Riemannian manifold (in our case, the unit sphere $\mathbb{S}^{d_1 - 1}$). In particular, they replace the standard gradient and Hessian with the gradient and Hessian on a Riemannian manifold, in the special case when the manifold is the unit sphere. In other words, they capture the intuition in order for $w_1$ to be a local minimizer of the constrained objective, the only local change to $w_1$ which decreases the loss should be increasing its norm. 

To conclude convergence to an approximate local minimum, we note that by the argument of~\citep{ge2015escaping}, there are sufficiently small step size and additive noise such that for any choice of $\epsilon$, the algorithm converges to an $(\epsilon, \gamma)$-approximate local minimizer of the objective $\min_{\|w\|_2 \le 1} L(w_1, w_2)$ (defined in the same manner) satisfying $\|w_2\|_2 \le \epsilon$. By the regularity conditions on $L$, this is also a $(C' \epsilon, C'\gamma)$-approximate local minimizer of the purified objective for some $C'$ depending on the regularity conditions.

\newpage

\section{Proofs for finite sample setting} \label{sec:finite_sample}

Given $n$ samples $X_1, ..., X_n$ define the empirical loss $\hat{L}$ on \emph{unlabeled} data as:
\begin{equation}
\hat{L}(w) = \frac{1}{n}  \sum_{i=1}^n \lexp(w^T X_i)
\end{equation}
We analyze self-training on the empirical loss, which begins with a classifier $\ws$, and does projected gradient descent on $\hat{L}$ with learning rate $\eta$.
\begin{align}
w^0 & = \ws \nonumber\\
w^{t+1} & = R\frac{w^t - \eta \nabla \hat{L}(w^t)}{\|w^t - \eta \nabla \hat{L}(w^t)\|_2} \nonumber
\end{align}
\textbf{Recap of analysis in infinite setting}: In the infinite sample case, gradient descent moves in direction $-\nabla L(w)$, and we show that $\|w_2\|_2 \to 0$ as we self-train. If the loss were convex, we could just analyze the minima and show it had the desired property that $\|w\|_2 = 0$. Standard results in convex analysis would then show convergence. Since the loss is non-convex, the core of the proof is to bound certain directional gradients. In particular, we showed that $\langle \nabla_{w1} L(w), w_1 \rangle < 0$ , $\langle \nabla_{w2} L(w), w_2 \rangle \geq c_1^2 \| w_2 \|_2^2$, and $\| \nabla L(w) \|_2^2 \leq c_2^2 \|w\|_2^2$. Using this, we analyzed the gradient descent iterates and showed that $\|w_2\|_2$ decreased by a multiplicative factor at each step of self-training.

\textbf{Finite sample proof overview}: With finite samples, gradient descent instead moves in direction $-\nabla \hat{L}(w)$ where $\hat{L}(w)$ is the empirical loss on $n$ samples. Here $w_2$ won't go to exactly 0, but to a very small value: we will show $||w_2||_2 \to \tau$ with high probability if we use $\widetilde{O}(1/\tau^4)$ samples. At a high level, we will show a uniform concentration bound which shows that the empirical gradient $\nabla \hat{L}(w)$ and population gradient $\nabla L(w)$ are close for all $w$ (Lemma~\ref{lem:emp_grad_pop_grad_close}). In Theorem~\ref{thm:main-finite_sample}, this lets us show that  $\langle \nabla_{w1} \hat{L}(w), w_1 \rangle$ , $\langle \nabla_{w2} \hat{L}(w), w_2 \rangle$, and $\| \nabla \hat{L}(w) \|_2^2$  are similar to the population versions above with $L(w)$ instead of $\hat{L}(w)$. We use this to show that $\|w\|_2$ will keep decreasing until $\|w\|_2 \leq \tau$, and will then stabilize and stay below $\|\tau\|$ forever.

\textbf{Notation}: To avoid defining too many constants, we use big-O notation in the following sense that is different from the standard computer science usage but common in learning theory proofs. When we use $O(e)$ in an expression, we mean that expression $e$ can be replaced by $c_e e$ for some universal constant $c_e$ that does not depend on \emph{any} problem parameters (like $\delta, \sigma_{\min}, \gamma$, etc)---it is literally just some number like say $67/32$, but explicitly putting the numbers everywhere makes expressions messy. For sub-Gaussian, sub-exponential, we will use notation and standard results from~\cite{vershynin2018high}, in particular the \emph{norm} of a sub-Gaussian random variable (Definition 2.5.6), equivalent properties of sub-Gaussian random variables (Proposition 2.5.2), the relationship between sub-Gaussian and sub-exponential random variables (Lemma 2.7.6), and Bernstein's inequality for sub-exponential random variables (Theorem 2.8.1).

We define the empirical expectation $\hat{E}$ for any function $f$ over the $n$ samples $X_1, \ldots X_n$:

\begin{equation}
\hat{E}[f(X)] = \frac{1}{n} \sum_{i=1}^n f(X_i) 
\end{equation}

\subsection{Results}

\newcommand{\maxeigen}{\sigma_{\text{max}}}
\newcommand{\mineigen}{\sigma_{\text{min}}}

Our first Lemma shows that the empirical gradients $\nabla \hat{L}(w)$ and population gradient $\nabla L(w)$ are close for all $w$, if the distribution is sub-Gaussian. We will show later that Gaussian distributions and mixtures of $K$ log-concave distributions are indeed sub-Gaussian. Data that is normalized will also satisfy the sub-Gaussian assumption.

\begin{lemma}
\label{lem:emp_grad_pop_grad_close}
Let $\pi : \mathbb{R}^d \to \mathbb{R}^{d'}$ be a projection operator with $d' \leq d$, that is $\pi$ is a $d'$-by-$d$ matrix where each row of $\pi$ is orthnormal, and suppose the distribution $X \sim p(x)$ satisfies that $\|X\|_2$ is sub-Gaussian with norm $B$ (equivalently, variance parameter $B^2$). Suppose we choose:
\begin{equation}
n = \widetilde{O}\Big( \frac{d}{\epsilon^2} R^2 B^2 \log{1/\delta} \Big)
\end{equation}
Where we hide terms that are \emph{logarithmic} in $\frac{1}{\epsilon^2}$, $B$, $R$, and $d$ in the big-O here to highlight the prominent terms, but give the full version in the proof. Then, with probability $\geq 1 - \delta$, for all $w$ with $||w||_2 \leq R$, we have:
\begin{equation}
\label{eqn:emp_deriv_conc}
\big\lvert \hat{E}[l'(w^{\top} X)\pi(w)^{\top} \pi(X)] - E[l'(w^{\top} X)\pi(w)^{\top} \pi(X)]  \big\rvert \leq \epsilon
\end{equation}
\end{lemma}

\begin{proof}
We will use a discretization argument. We first show the concentration in Equation~\ref{eqn:emp_deriv_conc} for fixed $w$ using the fact that the distribution is sub-Gaussian and then applying Hoeffding's inequality. We will then construct an $\epsilon$-cover of the $R$-ball (in $\ell_2$ norm) and use union bound so that the concentration holds for each member of the $\epsilon$-cover. Finally, we will show that the concentration holds for all $w$ with $||w||_2 \leq R$. Let $h(w, X) = l'(w^{\top} X)\pi(w)^{\top} \pi(X)$. 

\textbf{Step 1: Concentration for single $w$}: First, we have:
\begin{align}
|h(w, X)| &= |l'(w^{\top} X)\pi(w)^{\top} \pi(X)| \\
&\leq |\pi(w)^{\top} \pi(X)| \\
&\leq ||\pi(w)||_2 ||\pi(X)||_2 \\
&\leq ||w||_2 ||X||_2 \\
&\leq R ||X||_2
\end{align}
We want to bound $\hat{\E}[h(w, X)] - \E[h(w, X)]$. Since $||X||_2$ is sub-Gaussian with norm $B$, $h(w, X)$ is sub-Gaussian with norm $RB$, it then follows that $h(w, X) - E[h(w, X)]$ is a mean 0  sub-Gaussian random variable with norm $2RB$. So the average, $\hat{E}[h(w, X)] - E[h(w, X)]$ is mean 0 and sub-Gaussian with norm $2RB/\sqrt{n}$. By the sub-Gaussian tail bound, we then get that with probability at least $1 - \delta/5$:
\begin{equation}
\big\lvert \hat{\E}[h(w, X)] - \E[h(w, X)] \big\rvert \leq O\Big(\frac{1}{\sqrt{n}}RB\sqrt{\log(1/\delta)}\Big)
\end{equation}
To control the RHS to be less than $\epsilon / 3$ with probability at least $1 - \delta/5$, it suffices to choose:
\begin{equation}
n = O\Big( \frac{1}{\epsilon^2} R^2 B^2\log(1/\delta)\Big)
\end{equation}
Note that this is only for a single $w$.

\textbf{Step 2: $\kappa$-covering}: We will now construct a $\kappa$-covering consisting of $M$ vectors $w$. We will want the above inequality to hold for all $M$ vectors---to do this we will apply union bound. More precisely, a standard covering argument tells us that we can choose $M$ with $\log{M} \leq d \log(1 + (2R)/\kappa)$ and $M$ vectors $w_1, \ldots, w_M$ s.t. for any $w$ with $||w||_2 \leq R$, there exists $w_i$ with $||w_i||_2 \leq R$ and $||w_i - w||_2 \leq \kappa$. By union bound, we have that if we choose:
\begin{equation}
n = O\Big( \frac{1}{\epsilon^2} R^2 B^2(\log{M} + \log(1/\delta))\Big)
\end{equation}
Then for all $w_i$ the empirical concentration holds, that is with probability at least $1 - \delta/5$, for all $w = w_i$:
\begin{equation}
\big\lvert \hat{\E}[h(w, X)] - \E[h(w, X)]  \big\rvert \leq \epsilon / 3
\end{equation}
It now remains to choose $\kappa$ so that we can show this for all $w$ (not just $w = w_i$).

\textbf{Step 3: Handling all $w$ by hoosing $\kappa$ small}: To extend the result to all $w$ (not just $w = w_i$), we consider arbitrary $w, w'$ with $||w||_2, ||w'||_2 \leq R$ and $||w-w'||_2 \leq \kappa$. We want to show that the difference in their directional derivatives is not too large. More precisely, we would like to show that with probability $\geq 1 - \delta/5$, for all such $w, w'$:
\begin{equation}
\label{eps-cover-continuity}
|\E[h(w', X)] - \E[h(w, X)]| \leq \epsilon/3
\end{equation}
And similarly for its empirical counterpart, $\hat{\E}$. This then proves the main claim, because for any $w$ with $||w||_2 \leq R$, we can choose some $w_i$ in the $\kappa$-cover above. We then have:
\begin{align}
\hat{\E}[h(w, X)] - \E[h(w, X)] \leq &|\hat{\E}[h(w, X)] - \hat{\E}[h(w_i, X)]| \\
+ &|\hat{\E}[h(w_i, X)] - \E[h(w_i, X)]| \\
+ &|\E[h(w_i, X)] - \E[h(w, X)]|
\end{align}
And each of the terms in the RHS will be bounded above by $\epsilon/3$, so the LHS will be bounded above by $\epsilon$. To show Equation~\ref{eps-cover-continuity}, we first write:
\begin{align}
h(w', X) - h(w, X) =& (l'({w'}^{\top} X)\pi(w')^{\top} \pi(X) - l'(w^{\top} X)\pi(w')^{\top} \pi(X)) \\
+& (l'(w^{\top} X)\pi(w')^{\top} \pi(X) - l'(w^{\top} X)\pi(w)^{\top} \pi(X))
\end{align}
Using Cauchy-Schwarz, and using the fact that $l'(r) \leq 1$ for all $r$ and that $l'$ is 1-Lipschitz, we can show for the first term in the RHS:
\begin{equation}
|l'({w'}^{\top} X)\pi(w')^{\top} \pi(X) - l'(w^{\top} X)\pi(w')^{\top} \pi(X)| \leq \kappa R ||X||_2^2
\end{equation}
And for the second term in the RHS:
\begin{equation}
|l'(w^{\top} X)\pi(w')^{\top} \pi(X) - l'(w^{\top} X)\pi(w)^{\top} \pi(X)| \leq \kappa ||X||_2
\end{equation}
Combining the above 2 inequalities:
\begin{equation}
|h(w', X) - h(w, X)| \leq \kappa R ||X||_2^2 + \kappa ||X||_2
\end{equation}
We will show below that $\E[||X||_2]  = O(B)$, $\hat{\E}[||X||_2]  = O(B)$, $\E[||X||_2^2]  = O(B^2)$, $\hat{\E}[||X||_2^2]  = O(B^2)$ (for the empirical expectations, this will hold with probability at least $1 - \delta/5$.
Assuming this for now, this gives us that it suffices to choose $\kappa$ such that:
\begin{equation}
\frac{1}{\kappa} \geq \frac{1}{\epsilon}[RB^2 + B] 
\end{equation}
In which case, Equation~\ref{eps-cover-continuity} and its empirical counterpart hold. In total, this means we require $n$ to be:
\begin{equation}
n = O\Big(  \frac{1}{\epsilon^2} R^2 B^2\Big(d \log\big[1 + \frac{R^2B^2 + RB}{\epsilon}\big] + \log(1/\delta) \Big) \Big)
\end{equation}
Or omitting $\log$ terms except in $1/\delta$ (we keep $1/\delta$ to make the dependence on the probability explicit):
\begin{equation}
n = \widetilde{O}\Big( \frac{d}{\epsilon^2} R^2 B^2 \log(1/\delta) \Big)
\end{equation}

\textbf{Bounding the norm and norm-squared}: Finally, we bound the expectations of the norm and norm-squared of $X$, which we used above. By taking integrals, since $X$ is sub-Gaussian with norm $B$, we can show that:
\begin{equation}
\E[||X||_2] \leq O(B)
\end{equation}
\begin{equation}
\E[||X||_2^2] \leq O(B^2)
\end{equation}
Next, we will like to bound the empirical means of these quantities. Since $||X||_2$ is sub-Gaussian with norm $B$. This means that $||X||_2 - \E[||X||_2]$ is mean 0 and sub-Gaussian with norm $2B$. So for the average of $n$ iid samples, we have that with probability $\geq 1 - \delta$:
\begin{equation}
\hat{\E}[||X||_2] \leq \E[||X||_2] + O\big(\frac{B}{\sqrt{n}} \sqrt{\log{\frac{1}{\delta}}}\big)
\end{equation}
As long as we choose $n \geq O(\log(1/\delta))$, we have with probability at least $1 - \delta/5$:
\begin{equation}
\hat{\E}[||X||_2] \leq O(B)
\end{equation}
Squares of sub-Gaussian random variables are sub-exponential, so $||X||_2^2$ is sub-exponential with sub-exponential norm $O(B^2)$. Then, $||X||_2^2 - \E[||X||_2^2]$ is sub-exponential with sub-exponential norm $O(B^2)$. Then by Bernstein's inequality for sub-exponentials, as long as we choose $n \geq O(\log(1/\delta))$, we have with probability at least $1 - \delta/5$:
\begin{equation}
\hat{\E}[||X||_2^2] \leq \E[||X||_2^2] + O(B^2) \leq O(B^2)
\end{equation}
\end{proof}

The main theorem of this section shows that for a sub-Gaussian distribution, if we have bounds on $\langle \nabla_{w1} L(w), w_1 \rangle$ , $\langle \nabla_{w2} L(w), w_2 \rangle$, and $\| \nabla L(w) \|_2^2$   for the \emph{population}, but do entropy minimization on the \emph{empirical samples}, we will still converge with $||w_2|| \leq \tau$. We will later instantiate these bounds for the Gaussian setting and the more general log-concave setting.

\begin{theorem}
\label{thm:main-finite_sample}
Suppose that for all $w$, $\langle \nabla_{w1} L(w), w_1 \rangle < 0$ , $\langle \nabla_{w2} L(w), w_2 \rangle \geq c_1 \| w_2 \|_2^2$, and $\| \nabla L(w) \|_2^2 \leq c_2^2 \|w\|_2^2$, for some $c_1, c_2 > 0$ where $c_1, c_2$ are not a function of $w$. Let $\tau < 0.5$ be the desired norm for the spurious feature $w_2$, that is, we want $||w_2||_2 \leq \tau$ after running self-training. Let $\epsilon = O( c_1 \tau^2)$, and choose  $n = \widetilde{O}\Big( \frac{1}{\epsilon^2} R^2 B^2\log(1/\delta)\Big)$ according to the Lemma~\ref{lem:emp_grad_pop_grad_close} such that with probability $\geq 1- \delta$, for all $w$ with $||w||_2 \leq \max(R, 1)$ the empirical gradients along both $w_1$ and $w_2$ are near the true gradient:
\begin{equation}
\big\lvert \hat{E}[l'(w^{\top} X)w_1^{\top} x_1] - E[l'(w^{\top} X)w_1^{\top} x_1]  \big\rvert \leq \epsilon
\end{equation}
\begin{equation}
\big\lvert \hat{E}[l'(w^{\top} X)w_2^{\top} x_2] - E[l'(w^{\top} X)w_2^{\top} x_2]  \big\rvert \leq \epsilon
\end{equation}
Then if initially $||w_2^0||_2 \leq 0.5$, self-training with step size $\eta = O(\frac{c_1}{ c_2^2 })$, will converge to $||w_2||_2 \leq \tau$. Specifically, if at step $t$, $||w_2^{t}||_2 \geq \tau / 2$, then the norm of $w_2$ shrinks by a multiplicative factor and rapidly reduces to less than $\tau/2$:
\begin{equation}
||w_2^{t+1}||_2^2 < \Big(1 - O\Big(\frac{c_1}{c_2}\Big)^2\Big) ||w_2^t||_2^2
\end{equation}
Furthermore, once this has happened, the norm stabilizes: if $||w_2^t||_2 < \tau / 2$, then $||w_2^{t+1}||_2 \leq \tau$.
\end{theorem}
\begin{proof}
We note that $\langle \nabla_{w_2}L(w), w_2\rangle = \E[l'(w^{\top} X)w_2^{\top} x_2]$ and $\langle \nabla_{w_2}\hat{L}(w), w_2\rangle = \hat{E}[l'(w^{\top} X)w_2^{\top} x_2]$, and similarly for $w_1$. So we have for all $w$ with $||w||_2 \leq \max(R, 1)$:
\begin{equation}
\label{eqn:empirical-w1-bound}
\lvert \langle \nabla_{w_1}L(w), w_1\rangle - \langle \nabla_{w_1}\hat{L}(w), w_1\rangle\rvert \leq \epsilon
\end{equation}
\begin{equation}
\label{eqn:empirical-w2-bound}
\lvert \langle \nabla_{w_2}L(w), w_2\rangle - \langle \nabla_{w_2}\hat{L}(w), w_2\rangle\rvert \leq \epsilon
\end{equation}
\textbf{Step 1: Bounding empirical gradients}: The main optimization analysis requires us to bound 3 quantities: $\langle \nabla_{w_2}\hat{L}(w), w_2\rangle$, $\langle \nabla_{w_1}\hat{L}(w), w_1\rangle$, and $||\nabla_{w_2}\hat{L}(w)||_2^2$, which we first do.

Equation~\ref{eqn:empirical-w2-bound} gives us a bound on the empirical gradient along $w_2$:
\begin{equation}
\langle \nabla_{w_2}\hat{L}(w), w_2\rangle \ge c_1 ||w_2||_2^2 - \epsilon
\end{equation}
Equation~\ref{eqn:empirical-w1-bound} this gives us a bound on the empirical gradient along $w_1$:
\begin{equation}
\langle \nabla_{w_1}\hat{L}(w), w_1\rangle < \epsilon
\end{equation}
Finally, we bound $||\nabla_{w_2}\hat{L}(w)||_2^2$:
\begin{align}
||\nabla_{w_2}\hat{L}(w)||_2 &= \big(\max_{||v||_2 \leq 1} \langle v, \nabla_{w_2}\hat{L}(w) \rangle\big)^2 \\
&\leq \big( \max_{||v||_2 \leq 1} \langle v, \nabla_{w_2}L(w) \rangle + \epsilon \big)^2 \\
&= \big(||\nabla_{w_2}L(w)||_2 + \epsilon\big)^2 \\
&= 2||\nabla_{w_2}L(w)||_2^2 + 2\epsilon^2 \\
&\leq 2c_2^2 ||w_2||_2^2 + 2\epsilon^2 \\
\end{align}

Where in the first line we used the variational form of the 2-norm, second line we used Equation~\ref{eqn:empirical-w2-bound}, in the third line we used the variational form of the 2-norm again, fourth line we used the fact that $(a+b)^2 \leq 2a^2 + 2b^2$, and in the fifth line we used the bound on $||\nabla_{w_2}L(w)||_2^2$ in the asssumption of the theorem.

\textbf{Step 2: Show $w_2$ decreases and stabilizes}: 
Our updates involve taking a gradient descent step, and then projecting back to the sphere, $\|w\|_2 \leq R$. Define $\tilde{w}^{t + 1} = w^{t} - \eta \nabla L(w^t)$  to be the iterate before projecting. Then, we have:
\begin{align}
	||\tilde{w}_2^{t+1}||_2^2 &= ||w_2^{t}-\eta \nabla_{w_2}\hat{L}(w)|_{w=w^{t}}||_2^2 \\
	&= ||w_2^{t}||_2^2 + \eta^2 || \nabla_{w_2}\hat{L}(w)|_{w=w^{t}}||_2^2 \\
	&-2\eta \langle \nabla_{w_2}\hat{L}(w)|_{w=w^{t}}, w_2^{t} \rangle \\
    &\leq (1 + 2\eta^2 c_2^2 - 2\eta c_1 )||w_2^{t}||_2^2 + (2 \eta^2 \epsilon^2 + 2 \eta \epsilon) \\
\end{align}
We choose $\eta$ as:
\begin{equation}
\eta = \frac{c_1}{c_2^2}
\end{equation}
Which gives us:
\begin{equation}
||\tilde{w}_2^{t+1}||_2^2 \leq \Big(1 - \frac{1}{2} \frac{c_1^2}{c_2^2}\Big)||w_2^{t}||_2^2 + (2 \eta^2 \epsilon^2 + 2 \eta \epsilon)
\end{equation}
Since the norm is always non-negative, we note that $c_1^2/c_2^2 \leq 2$. To control the error terms, we choose $\epsilon$ as:
\begin{equation}
\epsilon = \frac{1}{48} c_1 \tau^2
\end{equation}
Then, we get,
\begin{equation}
2 \eta^2 \epsilon^2 + 2 \eta \epsilon \leq \frac{1}{4} \frac{c_1^2}{c_2^2}(\tau/2)^2
\end{equation}
In other words, if $||w_2^t||_2 \geq \tau/2$, then the norm decreases:
\begin{equation}
||\tilde{w}_2^{t+1}||_2^2 \leq \Big(1 - \frac{1}{4} \frac{c_1^2}{c_2^2}\Big)||w_2^{t}||_2^2
\end{equation}
And if $||w_2^t||_2 < \tau/2$, then the norm stabilizes:
\begin{equation}
||\tilde{w}_2^{t+1}||_2^2 \leq ||w_2^t||_2^2 + \frac{1}{2} (\tau/2)^2 \leq \frac{3}{2} (\tau/2)^2
\end{equation}

\textbf{Step 3: Show $w_1$ does not decrease much}: We have shown that $\tilde{w_2}^{t+1}$ is smaller than $w^t$. Next, we need to deal with the renormalization step to show that $w_2^{t+1}$ is also smaller than $w_2^t$. We will show that $\tilde{w}_1$ cannot decrease by too much, so that after renormalization, the norm of $w_2$ is still decreasing sufficiently. We have:
\begin{align}
	||\tilde{w}_1^{t+1}||_2^2 &= ||w_1^{t}-\eta \nabla_{w_1}\hat{L}(w)|_{w=w^{t}}||_2^2 \\
	&\geq ||w_1^{t}||_2^2 - 2\eta \langle \nabla_{w_1}\hat{L}(w)|_{w=w^{t}}, w_1^{t} \rangle \\
    &\geq  ||w_1^{t}||_2^2 - 2\eta\epsilon
\end{align}
From our choise of $\eta$ and $\epsilon$, we can show that $\eta \epsilon$ is actually very small. In particular, with some algebra, we can show that
\begin{equation}
2 \eta \epsilon < \frac{1}{24} \frac{c_1^2}{c_2^2} \leq \frac{1}{12}
\end{equation}
In effect, the decrease in the norm of $w_1$ is at least 10 times smaller than the decrease in the norm of $w_2$. Now we note that since at all times $t$, $||w_2^t|| \leq 0.5$, we have $||w_1^t|| \geq ||w_2^t||$. So $w_1$ is larger and decreases by a much smaller amount, which means that after renormalizing $w_2$ still decreases by around the same amount. Formally, with a bit of algebra, we get that if $||w_2^t||_2 \geq \tau/2$, then \emph{after renormalizing},
\begin{equation}
||w_2^{t+1}||_2^2 \leq \Big(1 - \frac{1}{10} \frac{c_1^2}{c_2^2} \Big)||w_2^{t}||_2^2
\end{equation}
And on the other hand, if  $||w_2^t||_2 < \tau/2$ then \emph{after renormalizing}, $||w_2^{t+1}||_2^2 < \tau^2$. This completes the proof.

\end{proof}

\subsection{Applying the finite sample results}

We now instantiate the above Theorem~\ref{thm:main-finite_sample} for the mixture of $K$ sliced log-concave setting:


\begin{proof}[Proof of Theorem~\ref{thm:general_main} finite sample guarantee] 
In the population case proof of Theorem~\ref{thm:general_main}, we showed that $\langle \nabla_{w2} L(w), w_2 \rangle \geq O(c_1 ||w_2||_2^2)$, $\langle \nabla_{w1} L(w), w_1 \rangle < 0$ , and $\| \nabla L(w) \|_2^2 \leq O(c_2^2 \|w\|_2^2)$. Additionally, we note that a mixture of $K$ sliced log-concave distributions is sub-Gaussian. So we get that there exists some $c, c'$ that depends on the data distribution, such that for all $\tau, \delta$ if we choose $n \geq (c/\tau^4) [\log(1/\tau) + \log(1/\delta)]$ samples then after $t$ iterations, if $t \geq c'\log(1/\tau)$, we will have $\|w_2^t\| \leq \tau$ with probability at least $1 - \delta$, where the probability is over the empirical / training data.
\end{proof}

For the Gaussian setting, besides showing that $\|w_2\|_2 \to 0$, we can also show that $w_1 \to 1$ and we achieve the Bayes opt classifier.

\begin{proof}[Proof of Theorem~\ref{thm:mixture} finite sample guarantee] 
Gaussian distributions are sliced log-concave distributions, so from the proof of Theorem~\ref{thm:general_main} finite sample guarantee, there exists some $c, c'$ that depends on the data distribution, such that for all $\tau, \delta$ if we choose $n \geq (c/\tau^4) [\log(1/\tau) + \log(1/\delta)]$ samples then after $t$ iterations, if $t \geq c'\log(1/\tau)$, we will have $\|w_2^t\| \leq \tau$ with probability at least $1 - \delta$, where the probability is over the empirical / training data.

Since $w_1$ is 1-dimensional in this case and at least $a / \gamma$ initially (Lemma ~\ref{lem:acc_to_a}), from the proof of Theorem~\ref{thm:main-finite_sample}, as long as $n \geq c''$ for some constant $c''$, $w_1$ after re-normalizing stays non-negative. As such, $w_1 \geq \sqrt{R^2 - \tau^2}$ with probability at least $1 - \delta$.
\end{proof}

\newpage

\section{Additional Missing Proofs}

\subsection{Bounds on $\lexp$ and derivatives}
\label{sec:lexp_bounds}
\begin{proof} [Proof of Lemma~\ref{lem:dg_dsigma}]
We can exactly compute the expression 
\begin{dmath}
		q_{\sigma}(\mu)=\frac{\partial g_{\sigma}(\mu)}{\partial \sigma}= \frac{1}{2}\exp{\left(\frac{\sigma^2}{2}\right)} \sigma \left[\exp{(\mu)}\erfc{\left(\frac{\sigma}{\sqrt{2}}+\frac{\mu}{\sqrt{2}\sigma}\right)} +\exp{(-\mu)}\erfc{\left(\frac{\sigma}{\sqrt{2}}-\frac{\mu}{\sqrt{2}\sigma}\right)}\right]- \sqrt{\frac{2}{\pi}}\exp{\left(-\frac{\mu^2}{2\sigma^2}\right)} \label{eq:q_def}
\end{dmath}
As $q_{\sigma}( \mu)$ and $\lexp$ are both symmetric around 0, we assume w.l.o.g. that $\mu \ge 0$.
We first consider $\sigma \le \frac{4\sqrt{2}}{\sqrt{\pi}}$. Since $\mu \ge \sigma^2$, $\sigma - \mu/\sigma \le 0$ so $\erfc((\sigma - \mu/\sigma)/\sqrt{2}) \ge 1$. Thus, we have 
	\begin{align}
		q_{\sigma}( \mu) \ge -\sqrt{\frac{2}{\pi}} \exp\left(-\frac{\mu^2}{2\sigma^2} \right) + \frac{1}{2} \sigma \exp \left(-\mu + \frac{\sigma^2}{2}\right)
	\end{align}
	Note that for $\mu \ge \sigopt$, we have 
	\begin{align}
		\frac{1}{4} \sigma \exp\left(-\mu + \frac{\sigma^2}{2}\right) \ge \sqrt{\frac{2}{\pi}} \exp\left(-\frac{\mu^2}{2\sigma^2} \right)
	\end{align} 
	in which case we can obtain $q_{\sigma}( \mu) \ge \frac{1}{4} \sigma \exp\left(-\mu + \frac{\sigma^2}{2}\right)$ by rearranging. 

Now we consider $\sigma > \frac{4\sqrt{2}}{\sqrt{\pi}}$. Since $\mu \ge 2\sigma^2$, we have $\erfc((\sigma - \mu/\sigma)/\sqrt{2}) \ge 1$ and $-\sqrt{\frac{2}{\pi}}\exp\left(-\frac{\mu^2}{2\sigma^2}\right) \ge -\sqrt{\frac{2}{\pi}}\lexp{(\mu)}$. Therefore \begin{align}
	q_{\sigma}( \mu) \ge -\sqrt{\frac{2}{\pi}}\lexp(\mu)+\frac{1}{2} \sigma \lexp(\mu) \ge \frac{1}{4} \sigma \lexp(\mu).
	\end{align} 
\end{proof}

\begin{lemma} \label{lem:exp_ent_deriv_bound}
	For all $\mu$, the following holds:
	\begin{align}\label{eq:exp_ent_deriv_bound:1}
	q_{\sigma}( \mu)= \frac{\partial}{\partial \sigma} \expent_{\sigma}(\mu) \ge -\sqrt{\frac{2}{\pi}} \exp\left(-\frac{\mu^2}{2\sigma^2}\right)
	\end{align}
	Furthermore, for $\sigma \le 1$, we also have 
	\begin{align} \label{eq:exp_ent_deriv_bound:2}
		\expent_{\sigma}(\mu) \ge 0.25 \lexp(\mu)
	\end{align}
\end{lemma}
\begin{proof}
	To conclude~\eqref{eq:exp_ent_deriv_bound:1}, we simply use the fact that $\erfc$ is always positive, so only the last term in~\eqref{eq:q_def} can be negative. 
	
	To conclude the second statement, assume without loss of generality that $\mu > 0$. We first note that with probability at least $0.68$, $1 \ge Z \ge -1$, and additionally, $\lexp(\mu + \sigma Z) \ge \exp(-\sigma ) \lexp(\mu)$ for $1\ge Z \ge -1$. When $\sigma \le 1$, we thus have $\expent_{\sigma} \ge 0.25 \lexp(\mu)$.
\end{proof}

\begin{lemma}\label{lem:exp_ent_ub}
	For all $\mu$ and $\sigma \le 1/2$, the following holds: 
	\begin{align}
		\expent_{\sigma}(\mu) \le 2 \lexp(\mu)
	\end{align}
\end{lemma}
\begin{proof}
	Without loss of generality, assume that $\mu > 0$. We note that we can upper bound $\lexp(\mu)$ by the loss function $\exp(-\mu)$. It follows that 
	\begin{align}
		\expent_\sigma(\mu) &\le \frac{1}{\sqrt{2\pi}}\int_{-\infty}^{\infty} \exp\left(-\mu - \sigma Z - \frac{Z^2}{2}\right) dZ\\
		&= \exp(-\mu)	\exp(\sigma^2) \sqrt{2}
	\end{align}
	As $\lexp$ is symmetric around 0, we could also apply the same argument to $\mu < 0$ using $\exp(\mu)$ as the loss upper bound. This gives the desired result. 
\end{proof}

\subsection{Log-concave and smooth densities}
\begin{claim}\label{claim:concave_smooth}
	Let $\density : \R \to \R$ be any density such that $\log \density$ is differentiable, $\concave$-strongly concave, and $\smooth$-smooth. Define $\dratio(\mu) \triangleq \frac{\partial}{\partial \mu} \log \density$. Then for all $\mu \in \R$, the following hold:
	\begin{align}
	\density(\mu + \delta) \ge \density(\mu) \exp\left(\dratio(\mu)\delta  - \frac{\smooth}{2} \delta^2 \right)\\
	\density(\mu + \delta) \le \density(\mu) \exp \left( \dratio(\mu)\delta  - \frac{\concave}{2} \delta^2 \right)
	\end{align}
\end{claim}
\begin{proof}
	By strong concavity and smoothness, we have 
	\begin{align}
	\log \density(\mu + \delta) \ge \log \density(\mu) + \frac{\partial}{\partial \mu} \log \density(\mu)  \delta - \frac{\smooth}{2} \delta^2\\
	\log \density(\mu + \delta) \le \log \density(\mu) + \frac{\partial}{\partial \mu} \log \density(\mu)  \delta - \frac{\concave}{2} \delta^2
	\end{align}
	Exponentiating both sides and using the definition of $\dratio$ gives the desired result. 
\end{proof}

\begin{lemma} \label{lem:dratio_bound}
	In the setting of Claim~\ref{claim:concave_smooth}, we have the following upper and lower bounds for $\dratio$ in terms of $\density$: 
	\begin{align}
	\dratio^2(\mu) \ge \concave \log \left( \frac{\sqrt{\concave}}{2 \density(\mu)\sqrt{\pi}}\right)\\
	\dratio^2(\mu) \le \smooth \log \left( \frac{\sqrt{\smooth}}{2 \density(\mu)\sqrt{\pi}}\right)
	\end{align}
	In other words, by rearranging,
	\begin{align}
	\density(\mu) \ge \frac{\sqrt{\concave}}{2 \sqrt{\pi}} \exp (-\dratio(\mu)^2/\concave)\\
	\density(\mu) \le \frac{\sqrt{\smooth}}{2 \sqrt{\pi}} \exp (-\dratio(\mu)^2/\smooth)
	\end{align}
\end{lemma}
\begin{proof}
	First, from Claim~\ref{claim:concave_smooth}, we have
	\begin{align}
	1 = \int_{-\infty}^{\infty} \density(\mu + \delta) d\delta &\le \int_{-\infty}^{\infty} \density(\mu) \exp\left(\dratio(\mu) \delta - \frac{\concave}{2} \delta^2 \right)\\
	& = \frac{2\sqrt{\pi}\density(\mu)  \exp(\dratio(\mu)^2/\concave)}{\sqrt{\concave}}
	\end{align}
	Solving, we obtain 
	\begin{align}
	\dratio^2(\mu) \ge \concave \log \left( \frac{\sqrt{\concave}}{2 \density(\mu)\sqrt{\pi}}\right)
	\end{align}
	Likewise, we can use the same reasoning to obtain
	\begin{align}
	\dratio^2(\mu) \le \smooth \log \left( \frac{\sqrt{\smooth}}{2 \density(\mu)\sqrt{\pi}}\right)
	\end{align}
\end{proof}

\begin{claim}\label{claim:integral}
	For any $a \in \R$, $b > 0$, we have 
	\begin{align}
		\int_{-\infty}^{\infty} \exp(ax - bx^2) &= \frac{\sqrt{2\pi} \exp(a^2/2b)}{\sqrt{b}} \label{eq:integral:1}\\
		\int_{0}^{\infty} \exp(ax - bx^2) &= \frac{\sqrt{\pi} \exp(a^2/2b) \left(\erf\left(\frac{a}{\sqrt{2b}}\right)  + 1\right)}{\sqrt{2b}}  \label{eq:integral:2}
	\end{align}
	Furthermore, when $a \ge 0$, we additionally have 
	\begin{align}
		\int_{0}^{\infty} \exp(ax - bx^2) \in \left[\frac{\sqrt{\pi} \exp(a^2/2b)}{\sqrt{2b}}, \frac{\sqrt{2\pi} \exp(a^2/2b)}{\sqrt{b}}  \right] \label{eq:integral:3}
	\end{align}
\end{claim}
\begin{proof}
	Equations~\eqref{eq:integral:1} and~\eqref{eq:integral:2} follow from direct computation. Equation~\eqref{eq:integral:3} follows because for $a \ge 0$, $1 \ge \erf(a/\sqrt{2b}) \ge 0$. 
\end{proof}

\begin{claim} \label{claim:erf_lb}
	For $a < 0$, 
	\begin{align}
		1 + \erf(a) > \frac{2}{\sqrt{\pi}} \frac{\exp(-a^2)}{-a + \sqrt{a^2 + 2}} \ge\frac{1}{\sqrt{\pi}(\sqrt{2} - 2a)}\exp(-a^2)
	\end{align}
\end{claim}
\begin{proof}
	We have $1 + \erf(a) = 1 - \erf(-a) = 1 - (1 - \erfc(-a)) = \erfc(-a)$. Now as $-a > 0$, we can apply the lower bound on $\erfc(-a)$ in~\cite{wolfram} to obtain the desired result. 
\end{proof}

\subsection{Equivalence between pseudo-labeling variant and entropy minimization}
\label{sec:pseudo=entropy}
\begin{proof} [Proof of Proposition~\ref{prop:pseudo=entropy}]
We compute
\begin{align}
\nabla_{w}L_{pseudo}^{t+1}(w)|_{w=w^t} &= \nabla_{w}\Exp_{x\sim \Dt} \ell_{exp}(w^\top x, \sign{({w^t}^\top x)}) |_{w=w^t} \\
&= -\Exp_{x\sim \Dt} \exp{(-w^\top x \cdot \sign{({w^t}^\top x)})} \cdot \sign{({w^t}^\top x)} x |_{w=w^t} \\
&= -\Exp_{x\sim \Dt} \exp{(-w^\top x \cdot \sign{(w^\top x)})} \cdot \sign{(w^\top x)} x |_{w=w^t} \\
&= \nabla_{w}L(w)|_{w=w^t}
\end{align}
Therefore for all $t \ge 0$, pseudo-labeling algorithm has the same iterate as entropy minimization~\eqref{eqn:alg}.
\end{proof}
\newpage
\section{Additional experiments and details}
\label{sec:exp_details} 
\subsection{Colored MNIST}
\label{sec:colored_mnist}
Among 70K MNIST images, we split the source training / source test / target training / target test into 2:1:3:1. The model architecture is 3-layer feed-forward network with hidden layer sizes 128 and 64. For training on source, we use SGD optimizer with learning rate 0.03, momentum 0.9, weight decay 0.002, and always train until convergence.

\noindent {\bf Additional construction and training details for 10-way MNIST.} For each source image, with probability $p$, we assign it a weight $w \stackrel{\mathclap{\normalfont\mbox{unif}}}{\sim}~ [0.1k, 0.1k+0.1)$ when image is digit $k$; with probability $1-p$, we assign $w \stackrel{\mathclap{\normalfont\mbox{unif}}}{\sim}~ [0,1)$. Each target image is assigned $w \stackrel{\mathclap{\normalfont\mbox{unif}}}{\sim}~ [0,1)$. We create two color channels by scaling the gray-scale image with weights $w$ and $1-w$.

In entropy minimization phase, we perform full gradient descent on target training set, with learning rate 0.03, momentum 0.9, weight decay 0.002, and train for 300 epochs when $p=0.95$ and 50 epochs when $p=0.97$.

\noindent {\bf Detailed construction of binary colored MNIST.}
In this setup, we assign digits 0-4 label 0 and 5-9 label 1. For each gray-scale image, we first draw a Gaussian random variable $\tilde{w} \sim \cN(0, (0.5/3)^2)$. In source domain, with probability $p=0.8$, example with label $k$ is assigned with $w = 0.5+(2k-1)|\tilde{w}|$; with probability 0.2,  $w = 0.5+\tilde{w}$. In the target domain, we always have $w = 0.5+\tilde{w}$. We create two color channels by rescaling the original image with weights $w$ and $1-w$.

For training, we keep all other hyper-parameters the same as the 10-way setting and only reduce the learning rate to 0.003.

\noindent {\bf Distribution of predictions conditioned on gray-scale image.} We examine the effect of entropy minimization on each test example in binary MNIST experiment. For each gray-scale test image $x_1$, we draw 1000 $x_2$, i.e., $\tilde{w} \sim N(0, (0.5/3)^2)$, and plot the distribution of logits $f(x_1, x_2)_{y}-f(x_1, x_2)_{1-y}$  where $y$ is the true label of $x_1$. According to our theory, if the distribution is concentrated around the positive side, entropy minimization would push the distribution to be more concentrated and positive.

Examples classified wrongly by source classifier \footnote{A particular $\hat{x_2}$ is drawn for each $x_1$ in the target test set. If $f(x_1, \hat{x_2})_{y}-f(x_1, \hat{x_2})_{1-y}<0$, $f$ makes a wrong prediction.} (because they were on the negative tail of the distribution) can be corrected due to this effect. Figure \ref{fig:good_1_before_after} is an example image where source classifier was wrong before training on target but corrected due to the explanation we provide. Conversely, examples classified right by source classifier (because they happen to be on the positive tail of the distribution) can turn wrong due to entropy minimization (see Figure \ref{fig:bad_1_before_after}). The success of entropy minimization relies on more examples concentrated on the positive than the negative side, i.e., source classifier has non-trivial target accuracy.

\begin{figure}
	\centering
	\includegraphics[width=0.49\columnwidth]{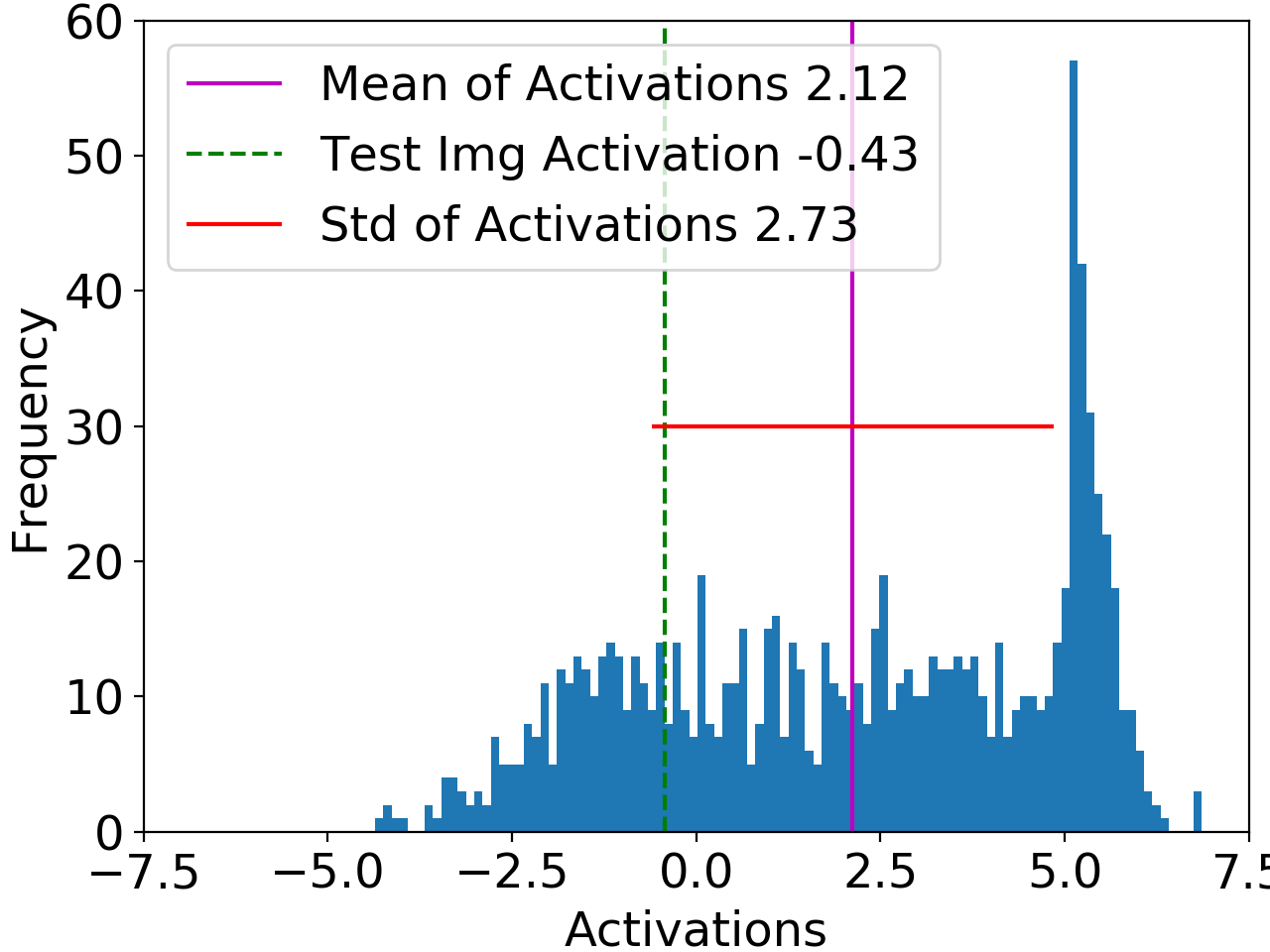}
	\includegraphics[width=0.49\columnwidth]{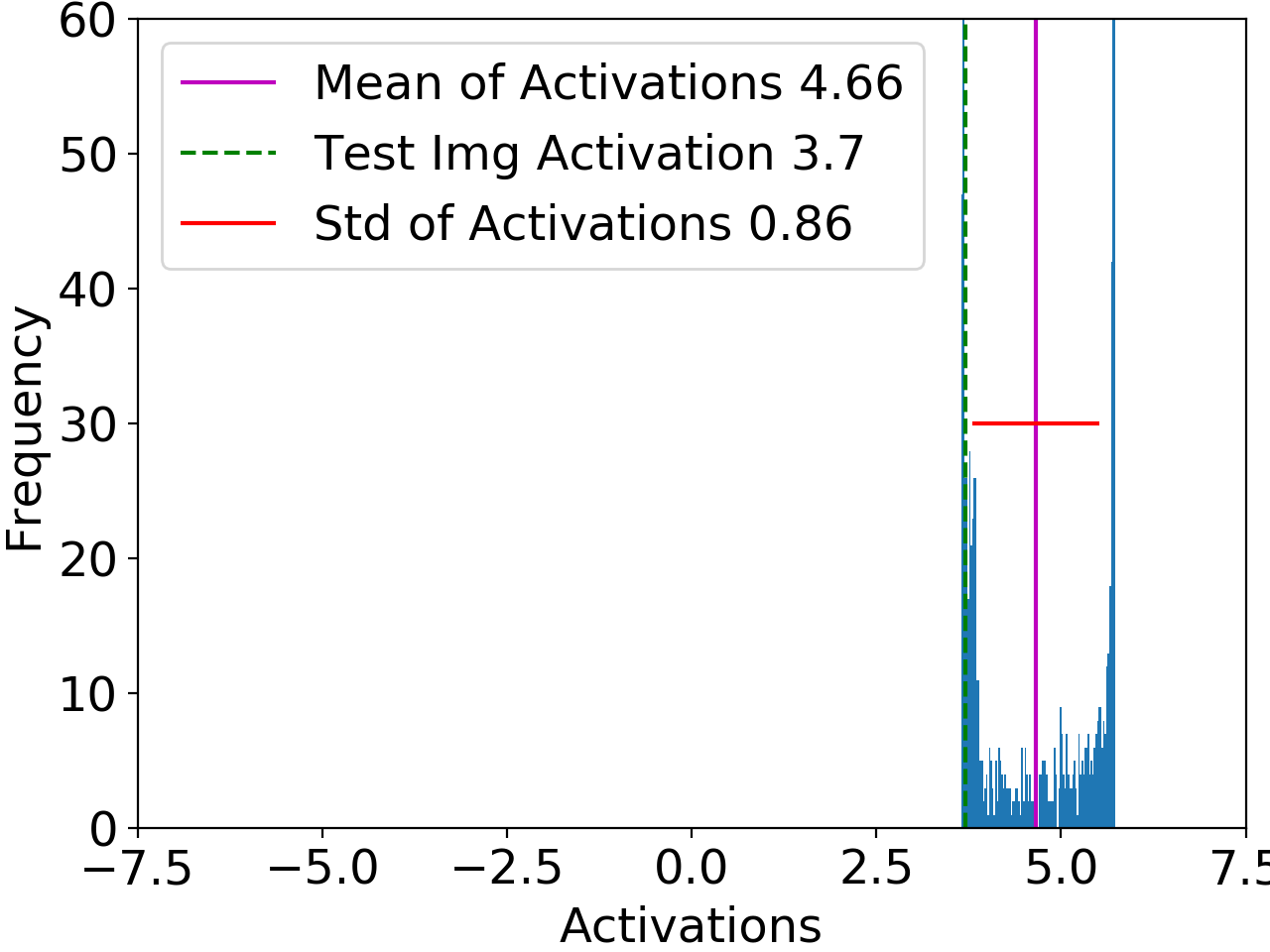}
	\caption{Distribution of $f(x_1, x_2)_y-f(x_1, x_2)_{1-y}$ before (\textbf{left}) and after (\textbf{right}) self-training for a test image whose prediction turned from wrong to correct. Green line shows $f(x_1, \hat{x_2})$ turning positive for the particular $\hat{x_2}$ in test set.}
	\label{fig:good_1_before_after}
\end{figure}
\begin{figure}
	\centering
	\includegraphics[width=0.49\columnwidth]{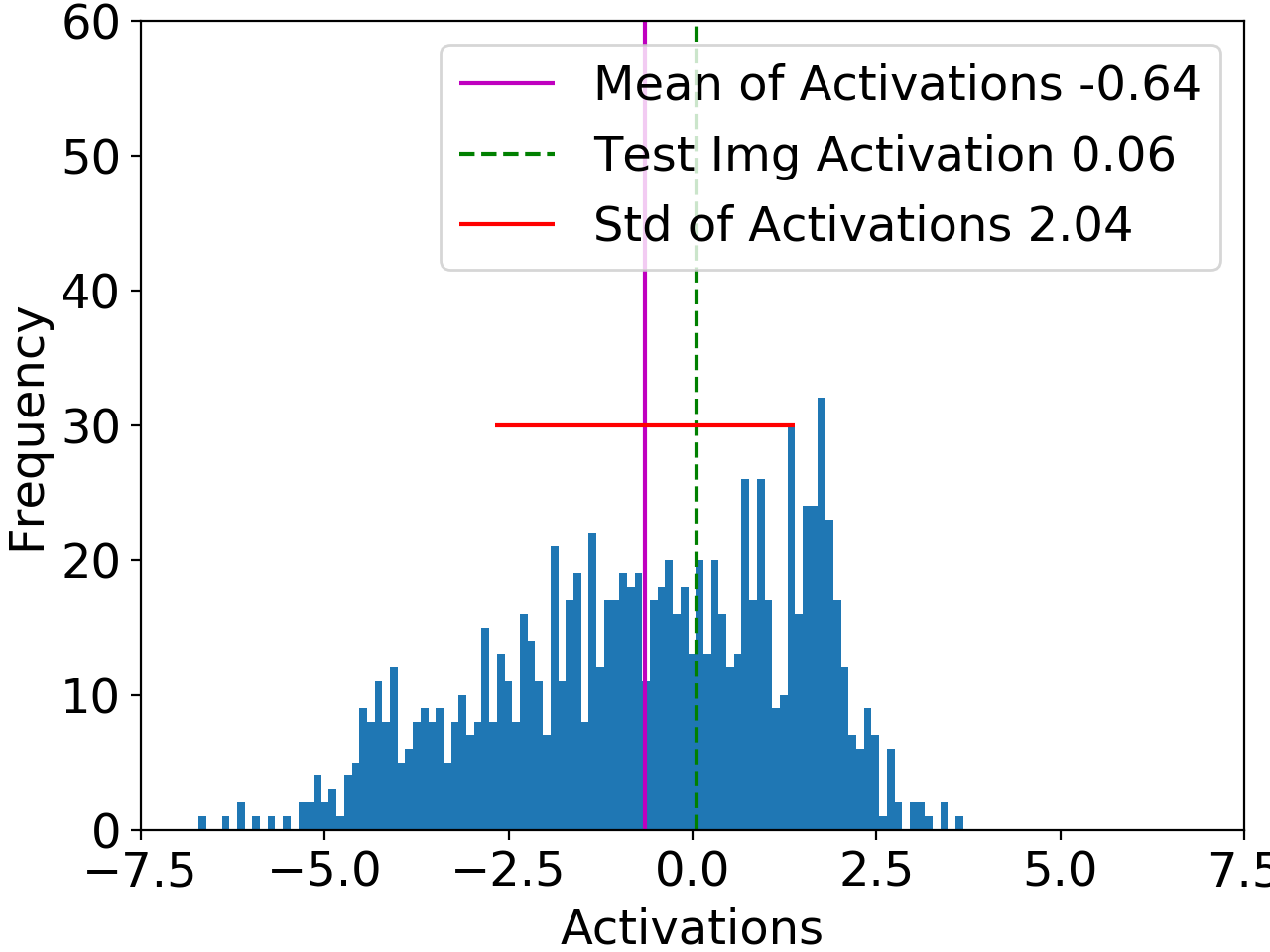}
	\includegraphics[width=0.49\columnwidth]{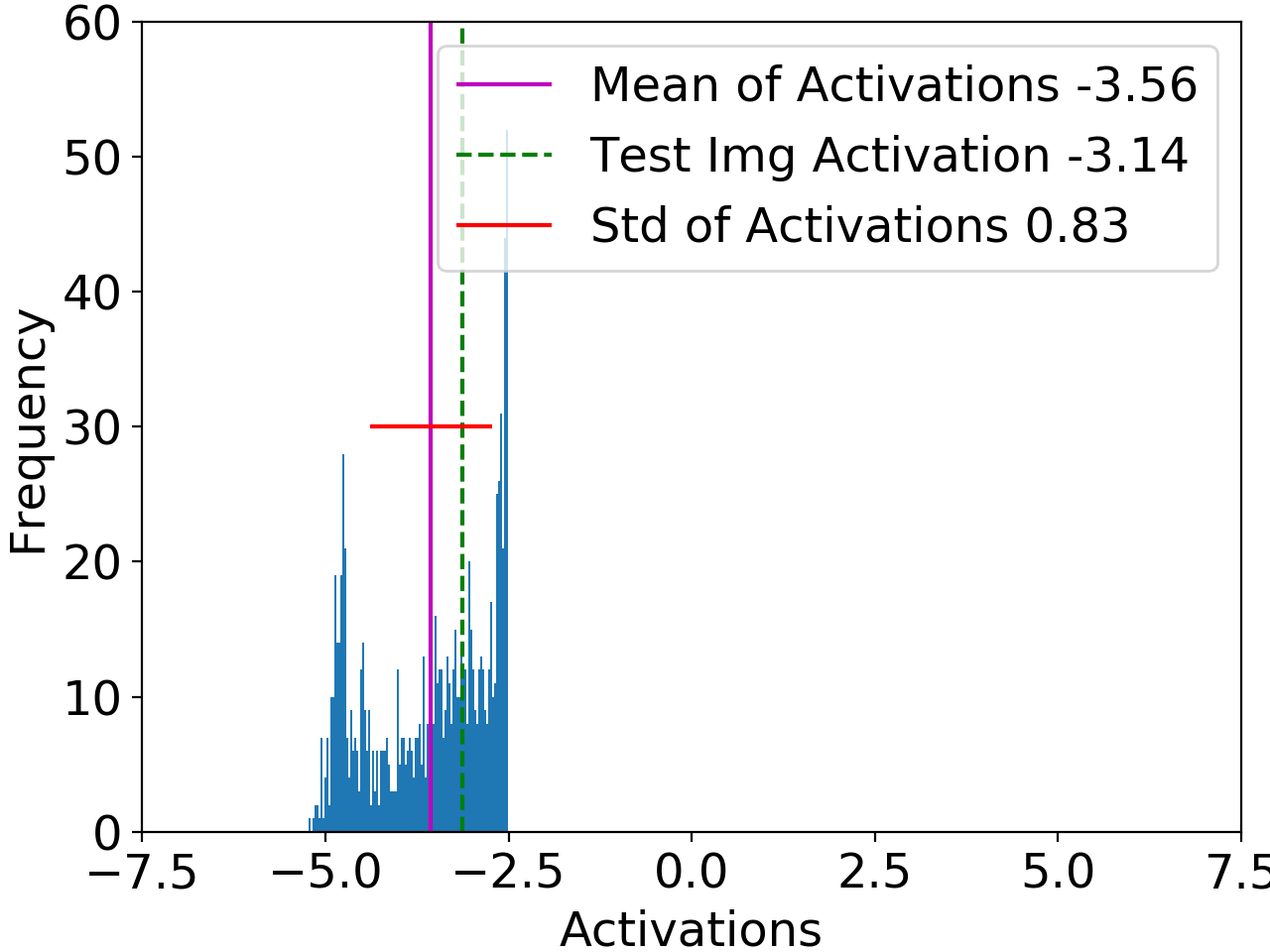}
	\caption{Distribution of $f(x_1, x_2)_y-f(x_1, x_2)_{1-y}$ before (\textbf{left}) and after (\textbf{right}) self-training for a test image whose prediction turned from right to wrong. Green line shows $f(x_1, \hat{x_2})$ turning negative for the particular $\hat{x_2}$ in test set.}
	\label{fig:bad_1_before_after}
\end{figure}

\noindent{\bf Distribution of mean activation.} Figure \ref{fig:plot_mu_before_after} shows the distribution of $f(x_1, \bar{x_2})$ before and after self-training for the binary MNIST experiment, where $\bar{x_2}$ indicates neutral color $(w=0.5)$. We see that qualitatively, the empirical distribution of $\mu$ has increasing mass far away from 0 throughout self-training, even for a multi-layer network. This is necessary for our theory, as seen in Figure~\ref{fig:q_sigma_mu}.
\begin{figure}
	\centering
	\includegraphics[width=0.49\columnwidth]{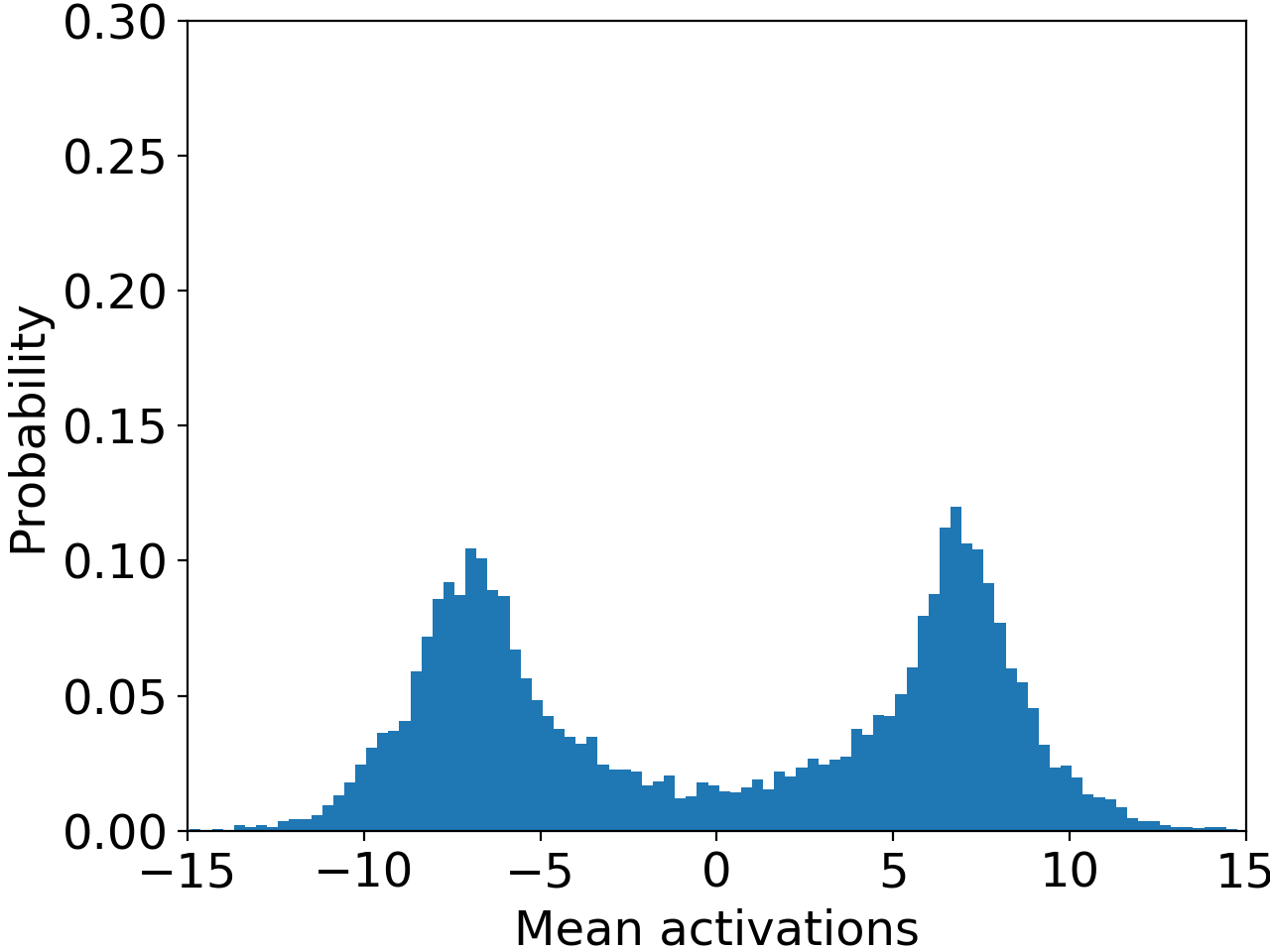}
	\includegraphics[width=0.49\columnwidth]{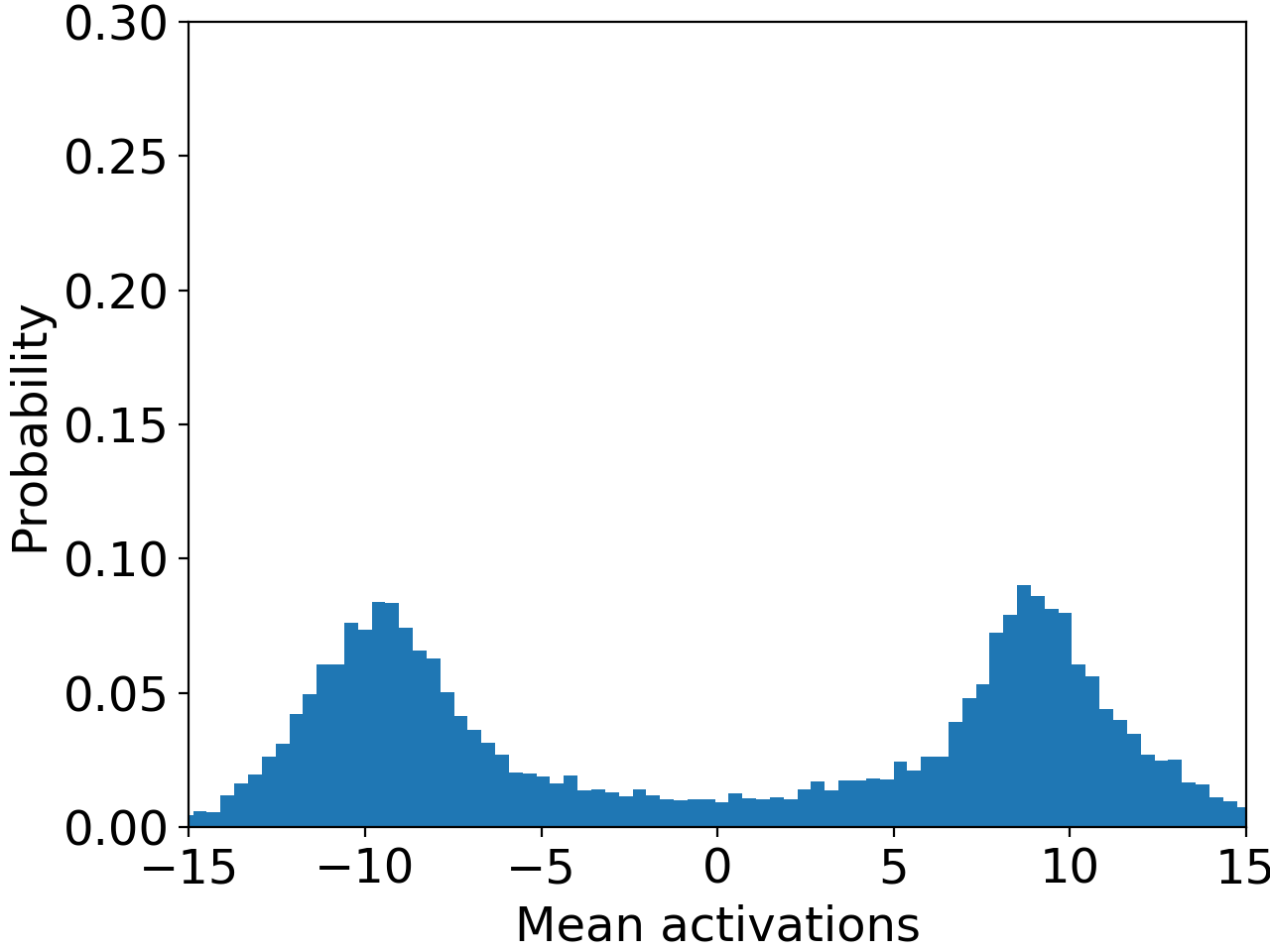}
	\caption{Distribution of $f(x_1, \bar{x_2})$ before (\textbf{left}) and after (\textbf{right}) self-training across all test images $x_1$ for neutral color $\bar{x_2}$. Qualitatively the empirical distribution of $\mu$ has more mass far away from 0 after self-training (which is the desired case for our theory, as seen in Figure~\ref{fig:q_sigma_mu}).}
	\label{fig:plot_mu_before_after}
\end{figure}

\noindent{\bf Importance of non-trivial source classifier accuracy.} We provide additional details on our study of 10-way colored MNIST when the spurious correlation probability is $p = 0.97$. In this setting, the source classifier has 98\% test accuracy on source but only 72\% on target. Entropy minimization initialized at $\tilde{f}$ causes target accuracy to \emph{drop} to 67\% (see Figure \ref{fig:train_plot_9597} right). 

\begin{figure}[ht]
\begin{center}
\includegraphics[width=0.4\columnwidth]{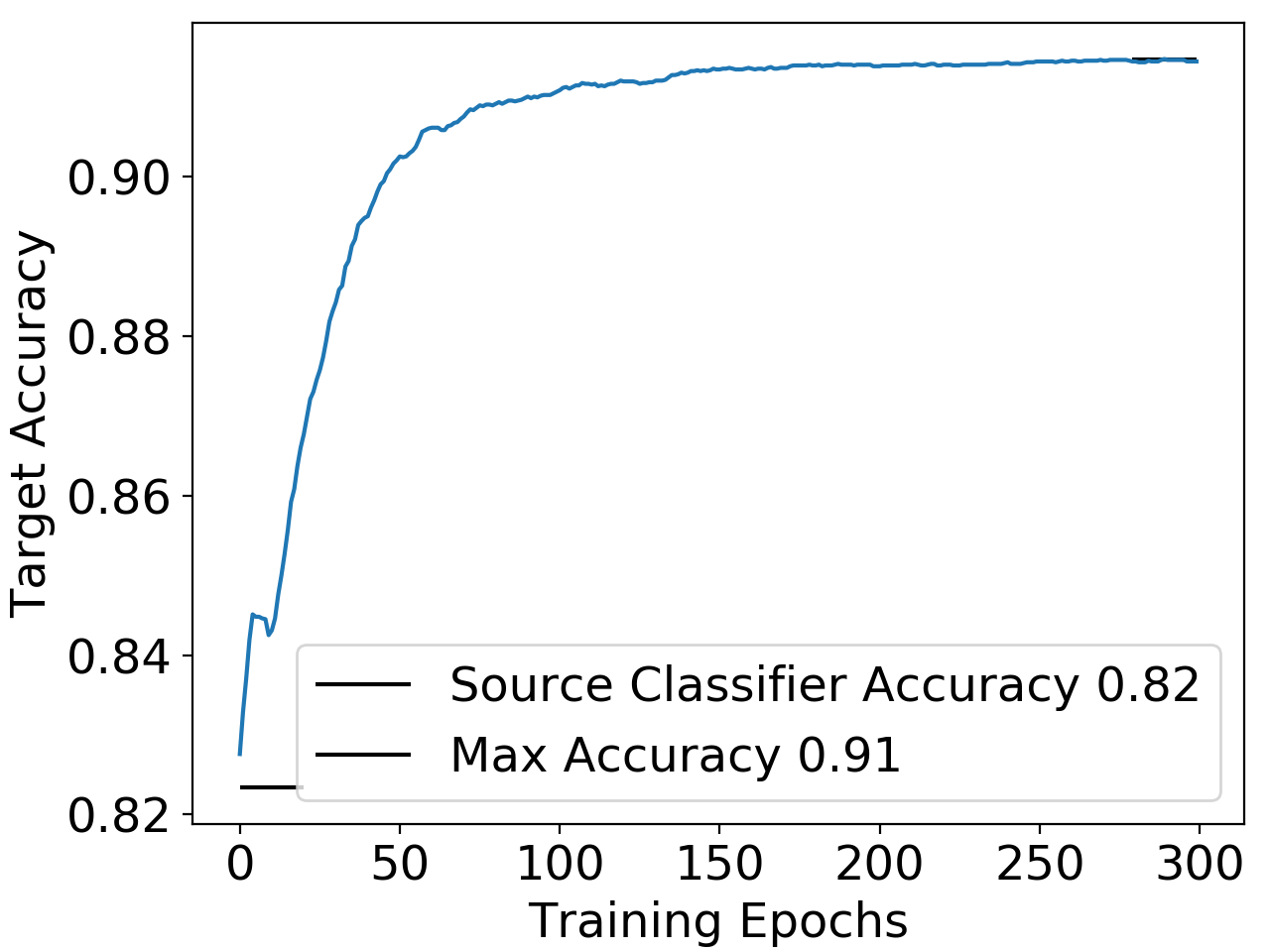}
\includegraphics[width=0.4\columnwidth]{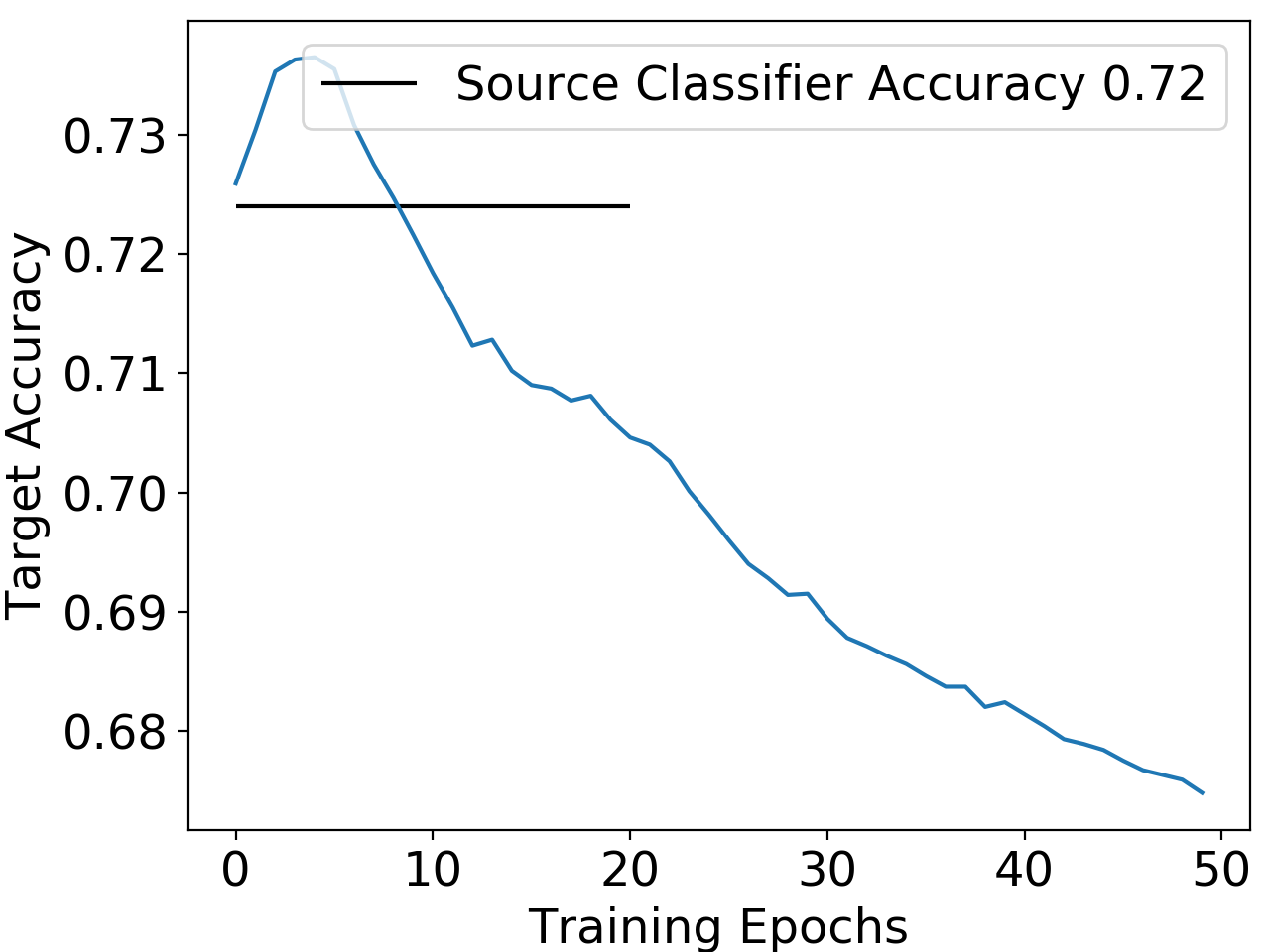}
\caption{In the 10-way MNIST experiment, entropy minimization raises target test accuracy by 9\% when we initialize with a good source classifier (\textbf{left}) and decreases target accuracy when we initialize with a bad source classifier (\textbf{right}). \textbf{Left}: Spurious correlation in source is $p=0.95$ so source classifier obtains high target accuracy; \textbf{Right}: Spurious correlation in source is $p=0.97$ so source classifier does not learn the right features.}
\label{fig:train_plot_9597}
\end{center}
\end{figure}

\subsection{CelebA dataset} 
\label{sec:celeb_a}
We partition the celebA dataset~\citep{liu2015faceattributes} so that the source domain has a perfect correlation between gender and hair color: 1250 blond males, 1749 non-blond females. The target domain has 57K unlabeled examples with the same correlation between gender and hair color as in the original dataset.

We use entropy minimization on this dataset with the Conv-Small model in~\citep{miyato2018virtual}. The source classifier has 94\% accuracy on source data and 81\% on target. After training on the sum of the source labeled loss and target entropy loss, the target accuracy increases to 88\%. 

\subsection{Connection between entropy minimization and stochastic pseudo-labeling}
\label{sec:entropy_min_pseudo}
In Equation~\ref{eq:pseudo_alg} we point out that entropy minimization is equivalent to a stochastic version of pseudo-labeling where we update the pseudo-labels after every SGD step. In practice, pseudo-labels are often updated for only a few rounds, and the student model is usually trained to convergence between rounds~\citep{xie2020selftraining}. In the 10-way MNIST experiment, we perform 3, 6, 30 rounds of pseudo-labeling with 100, 50, and 10 epochs of training per round, interpolating between more common versions of pseudo-labeling and entropy minimization. Figure~\ref{fig:interpolation} shows that entropy minimization converges to better target accuracy within the same clock-time, suggesting that practitioners may benefit from pseudo-labeling with more rounds and fewer epochs per round.

\begin{figure}[ht]
\begin{center}
\includegraphics[width=0.32\columnwidth]{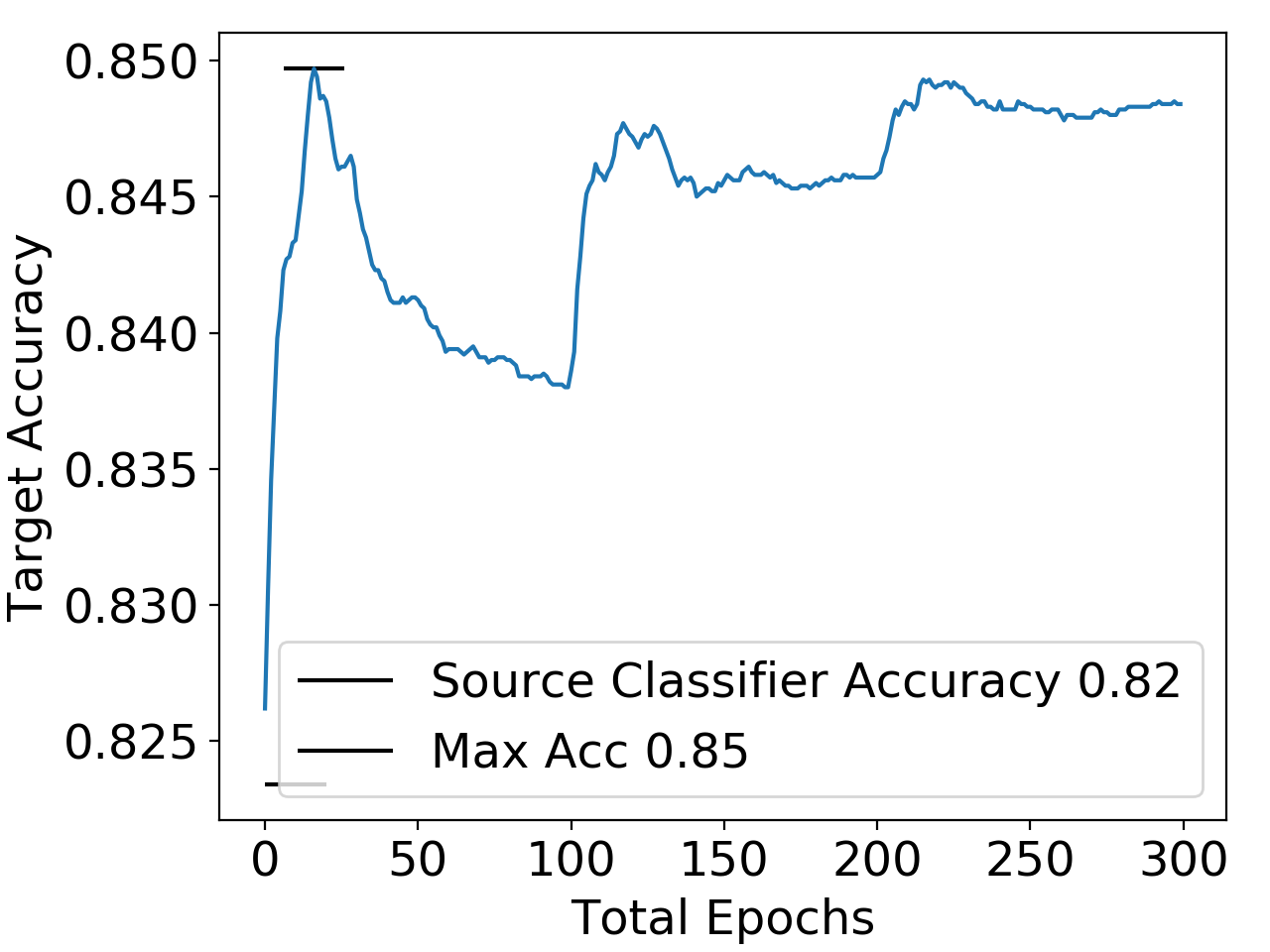}
\includegraphics[width=0.32\columnwidth]{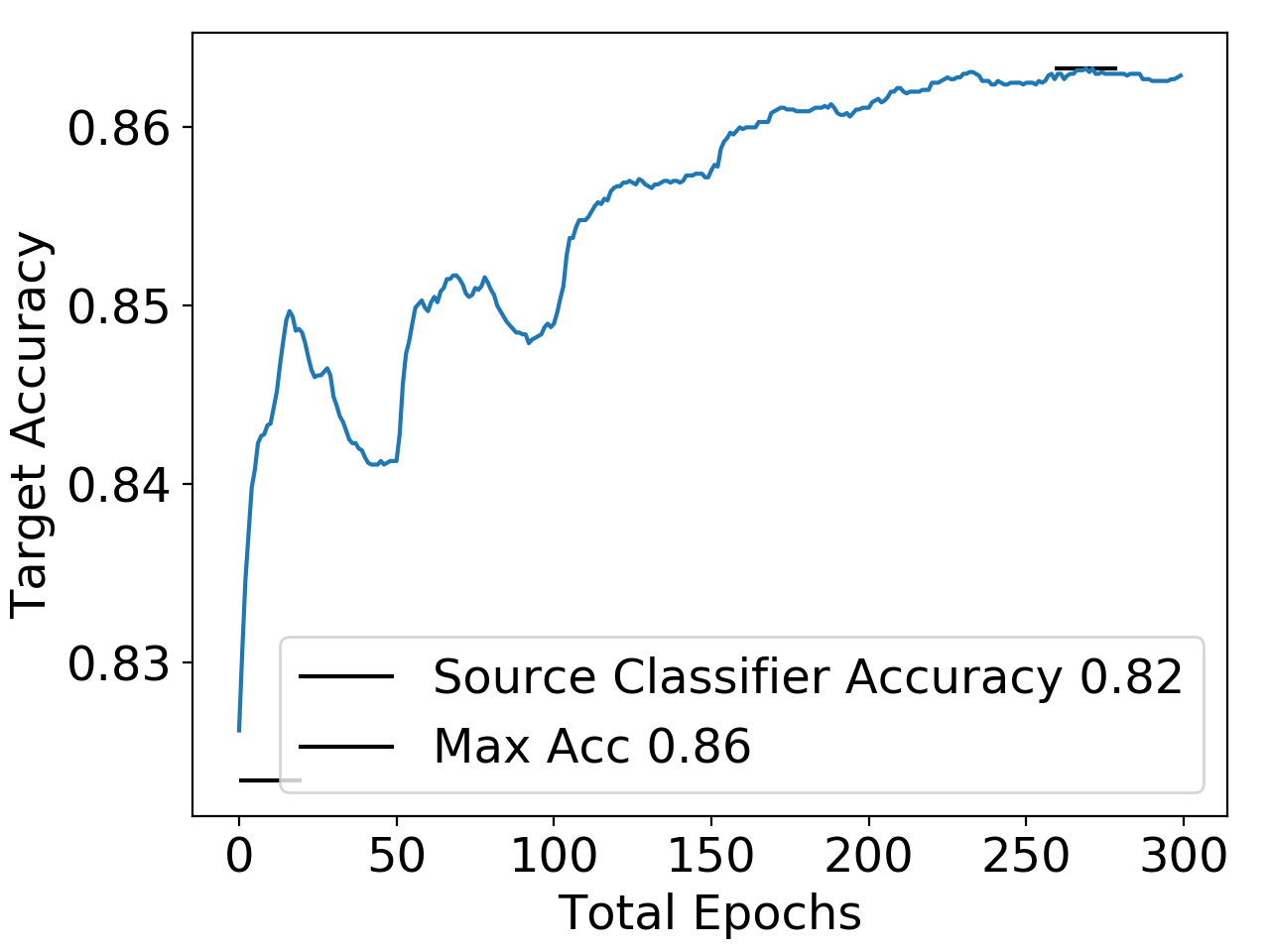}
\includegraphics[width=0.32\columnwidth]{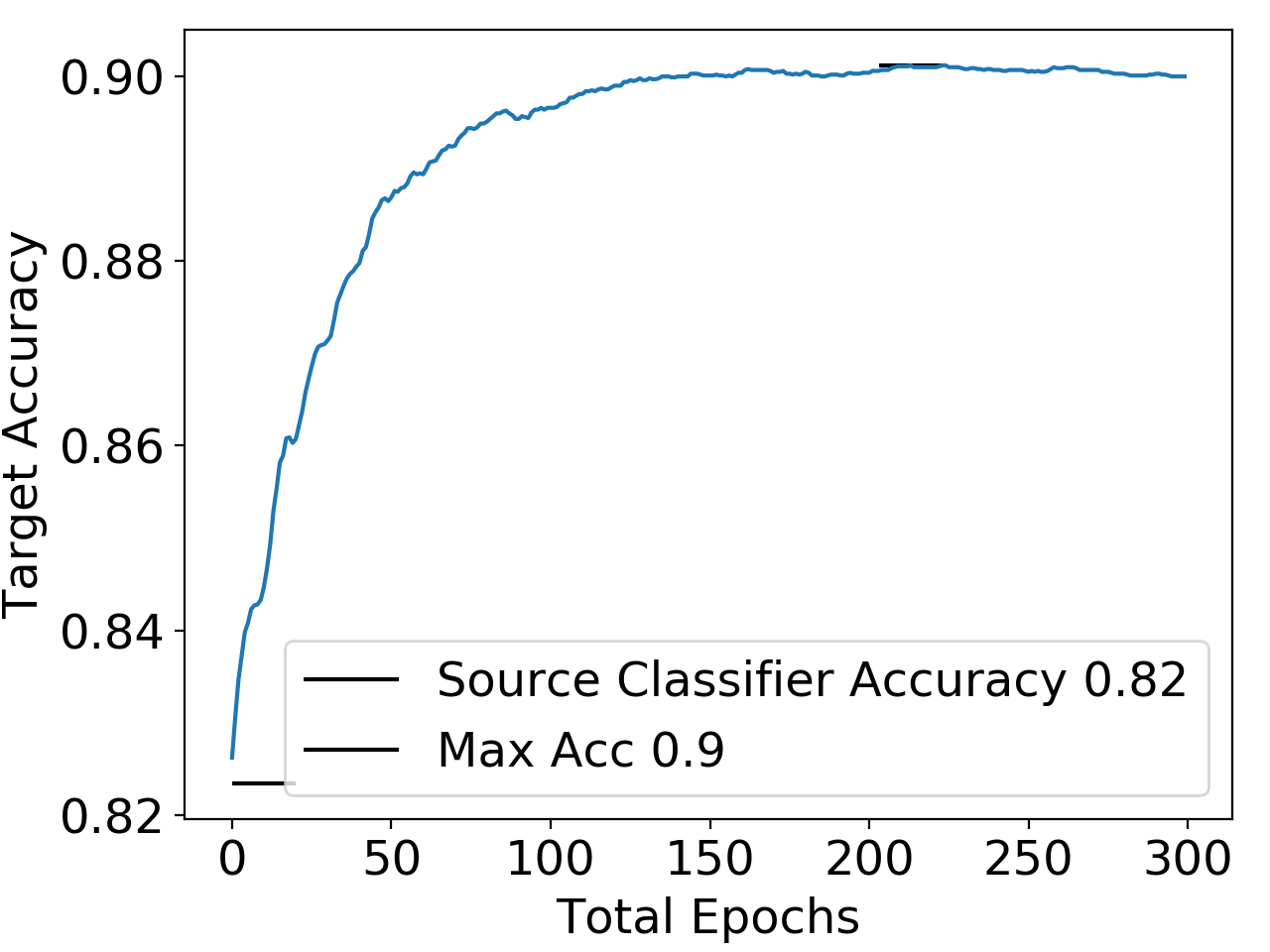}
\caption{In the 10-way MNIST experiment, 3 rounds of pseudo-labeling with 100 epochs per round (\textbf{left}), 6 rounds of pseudo-labeling with 50 epochs per round (\textbf{middle}), and 30 rounds of pseudo-labeling with 10 epochs per round (\textbf{right}) increase in target accuracy.}
\label{fig:interpolation}
\end{center}
\end{figure}

\subsection{Toy Gaussian mixture setting}
\label{sec:toy_gaussian_exp}
\noindent {\bf Generating data.} We generate source examples in the following fashion: For each example $(x_1, x_2) \in \R^4$, we first sample $y$ uniformly from $\{-1, 1\}$, and then $x_1 \in \R^{2} \sim \cN(\gamma y, I)$, where $\gamma$ is a random 2-dimensional vector. For the source examples, we then sample $\tilde{x_2} \sim N(\Vec{0}, I)$. For each coordinate $i$ of $x_2$ ($i \in \{1, 2\}$), with probability $0.8$, we set $(x_2)_i = y |(\tilde{x}_2)_i|$ (correlated); with probability $0.2$, we set $(x_2)_i = (\tilde{x}_2)_i$ (uncorrelated). For target examples, we sample $x_2 \sim N(\Vec{0}, I)$.

The source training dataset, source test set, and target test set all have 10K examples.

\noindent {\bf Algorithms.} We use entropy minimization as well as the following version of pseudo-labeling: starting with the source classifier, we perform 200 rounds of pseudo-labeling with 50 epochs of training in each round. We also set up a threshold $\tau=0.1$ where we throw out least-confident target example $x$ with $|w^\top x|<\tau$ in each round to mimic most popular pseudo-labeling algorithms used in practice \cite{xie2020selftraining}. We experiment on this version of pseudo-labeling algorithm because the version in equation \ref{eq:pseudo_alg} is equivalent to entropy minimization.

For entropy minimization, we use a new batch of 10K target training examples in each epoch. We use SGD optimizer with learning rate 1e-3 and normalize the linear model after each gradient step.

For pseudo-labeling, we use a new batch of 10K target examples in each round. Optimizer choices are the same as entropy minimization.

\begin{figure}[ht]
\begin{center}
\includegraphics[width=0.49\columnwidth, height=100pt]{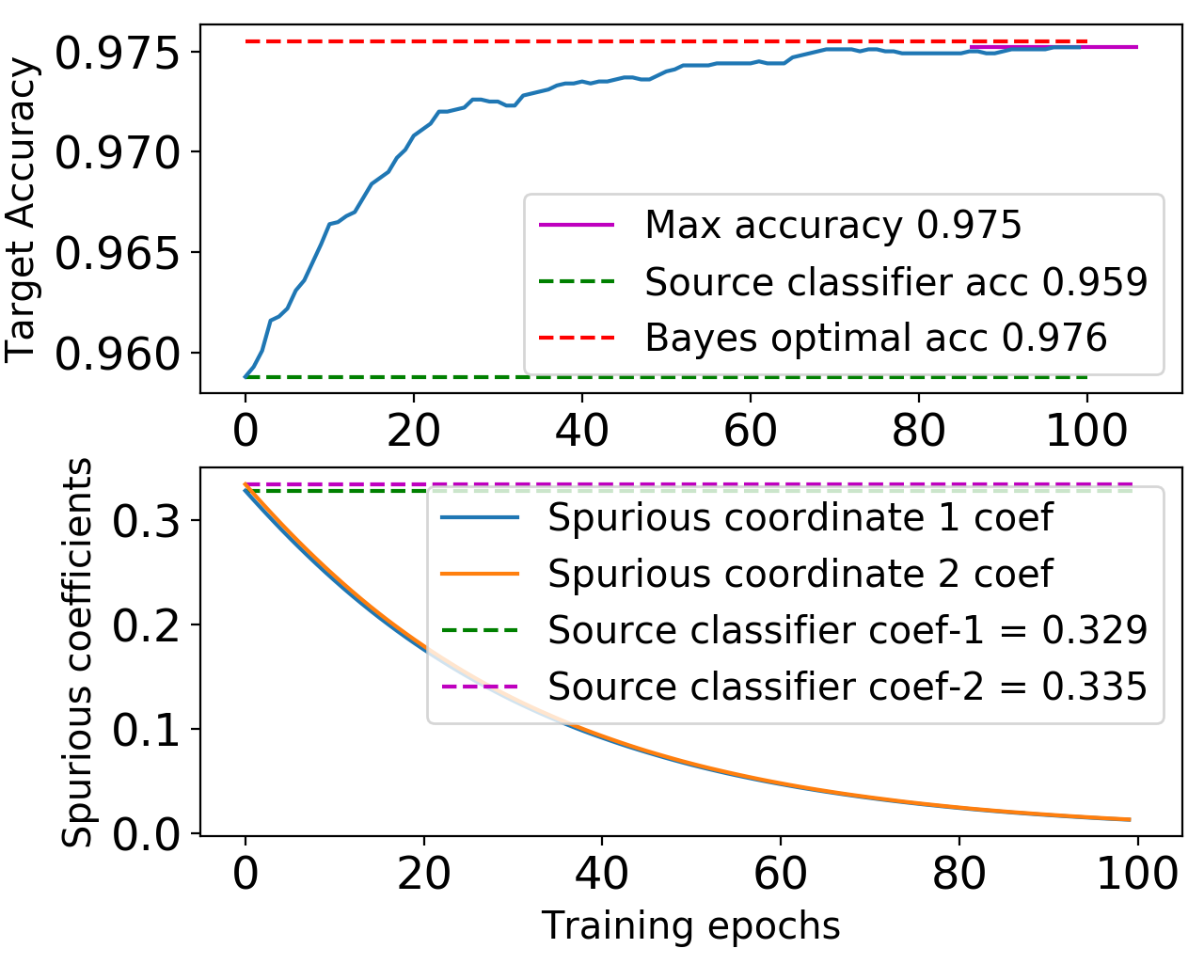}
\includegraphics[width=0.49\columnwidth,height=100pt]{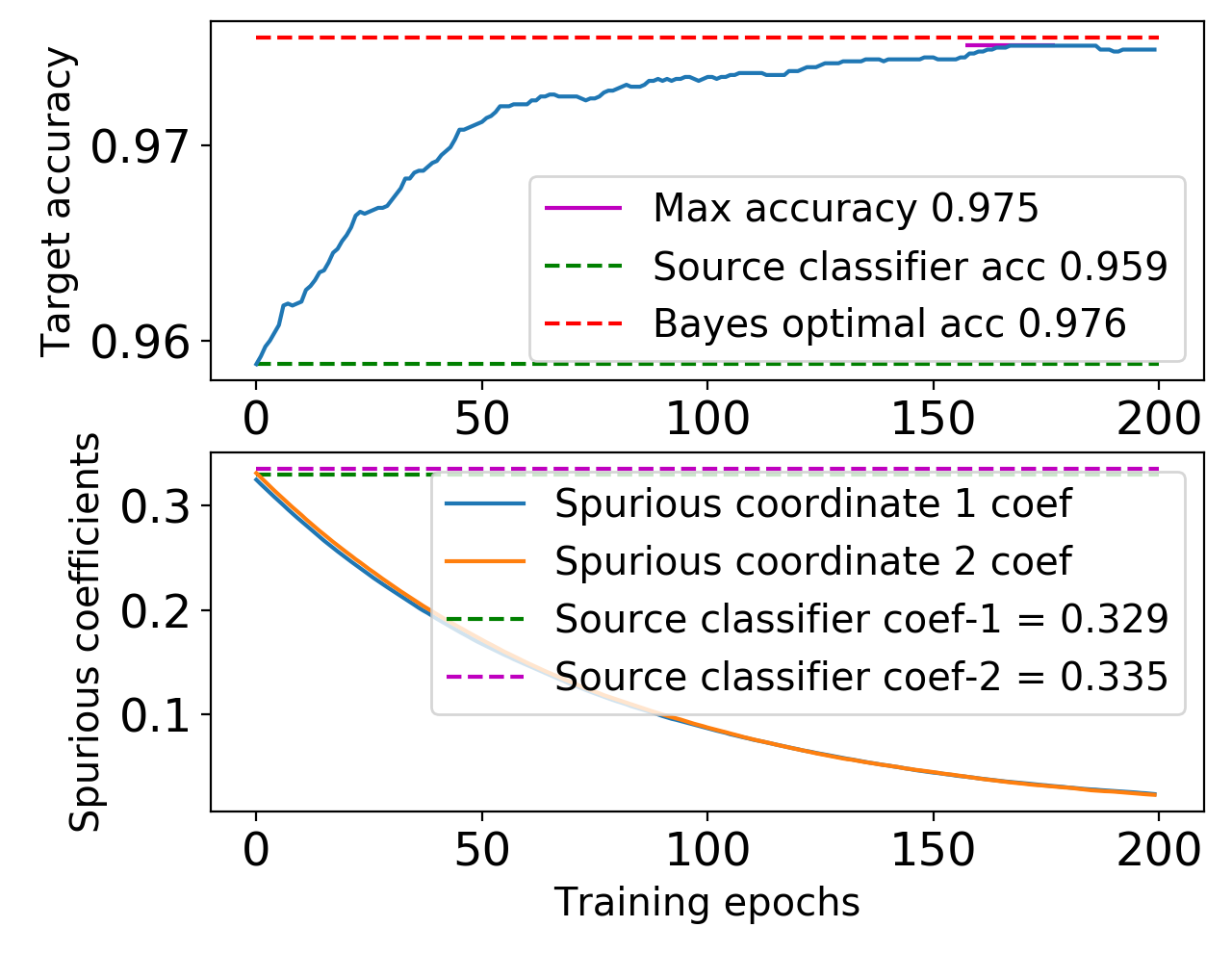}
\caption{Entropy minimization (\textbf{left}) and pseudo-labeling (\textbf{right}) increase target test accuracy from 95.9\% to 97.5\%, and reduce coefficients on two spurious coordinates from 0.33 to 0 in the Gaussian mixture experiment.}
\label{fig:gaussianmixture}
\end{center}
\end{figure}
\noindent{\bf Improvement of target test accuracy and de-emphasis of spurious features.} 
In the Gaussian mixture experiment, the source classifier gets an accuracy of 95.9\% on the target domain. Both entropy minimization and pseudo-labeling algorithms raise the target accuracy to Bayes-optimal while driving the coefficients $w_2$ on spurious features $x_2$ to 0 (Figure~\ref{fig:gaussianmixture}). Notably, even though we use confidence thresholding and train for 50 epochs in each round, the model behavior still closely tracks that of entropy minimization, as predicted by our theory.

\subsection{Justification of approximation $l(t)=\exp(-|t|)$}
\label{sec:justify_approx}

\begin{figure}
	\centering
	\includegraphics[width=0.7\textwidth]{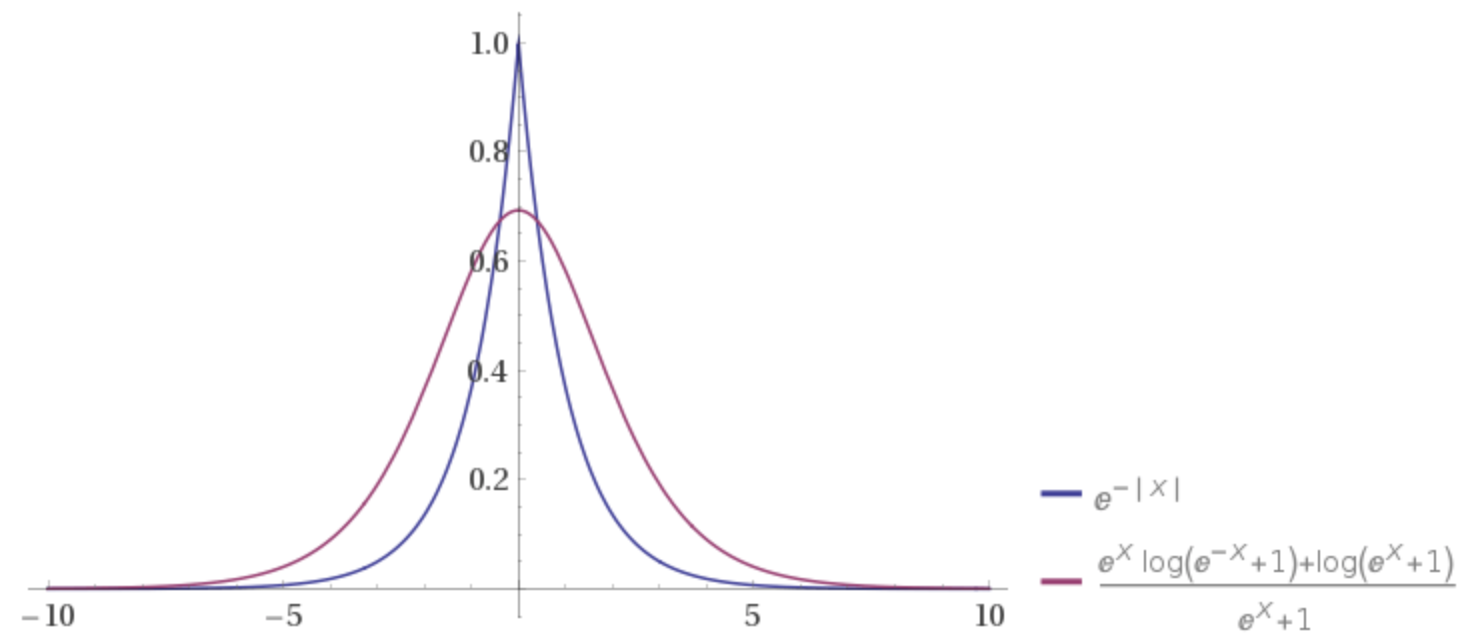}
	\caption{Plot of $exp(-|t|)$ and entropy loss. The losses are within a constant factor of each other and exhibit the same tail behavior.}
	\label{fig:ent}
\end{figure}

\begin{figure}[ht]
\begin{center}
\centerline{\includegraphics[width=0.6\textwidth]{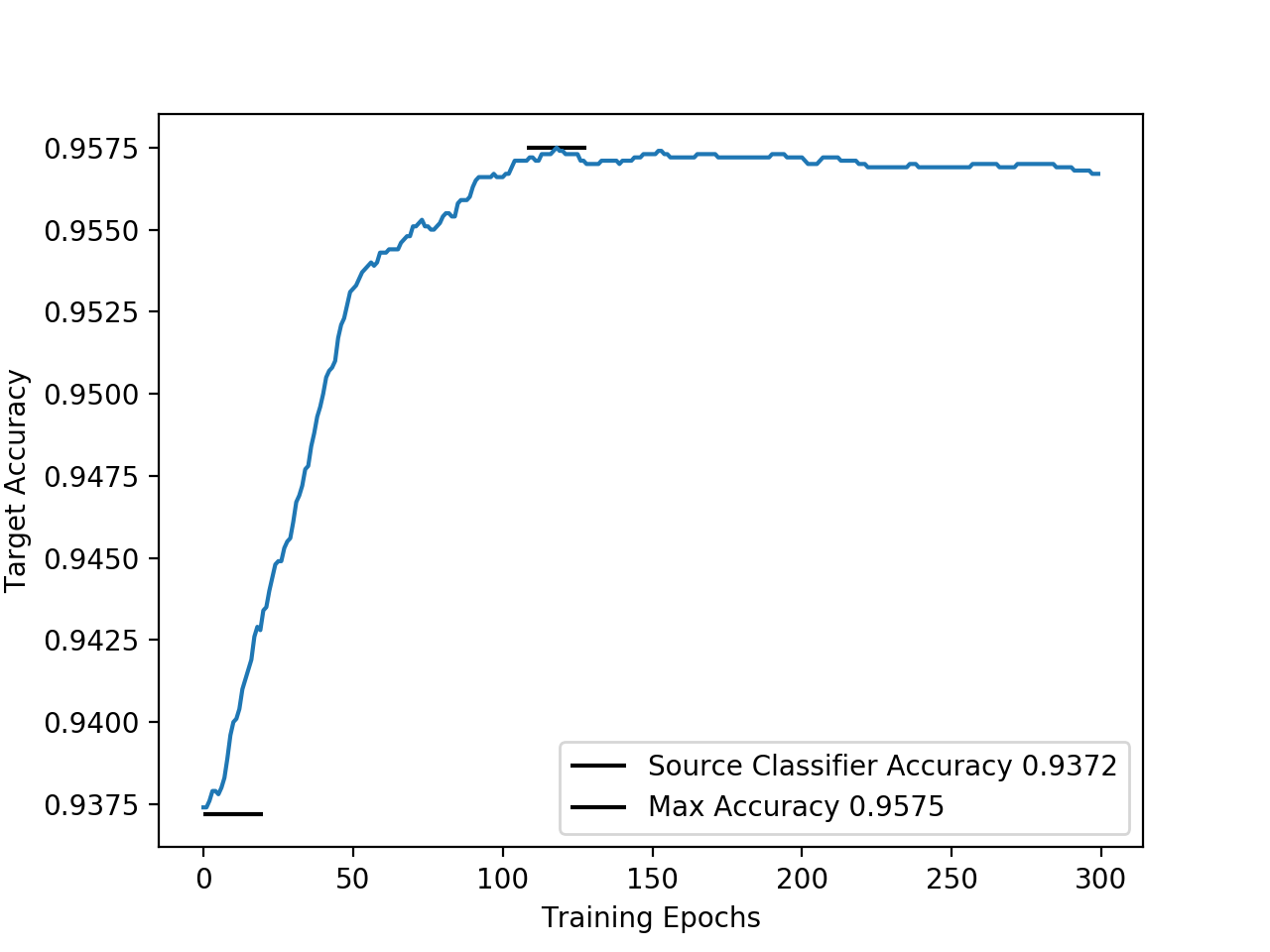}}
\caption{Target accuracy using $l(t)=\exp(-|t|)$ on binary MNIST dataset.}
\label{fig:exp_loss}
\end{center}
\end{figure}

\begin{figure}[ht]
\begin{center}
\includegraphics[width=0.45\textwidth]{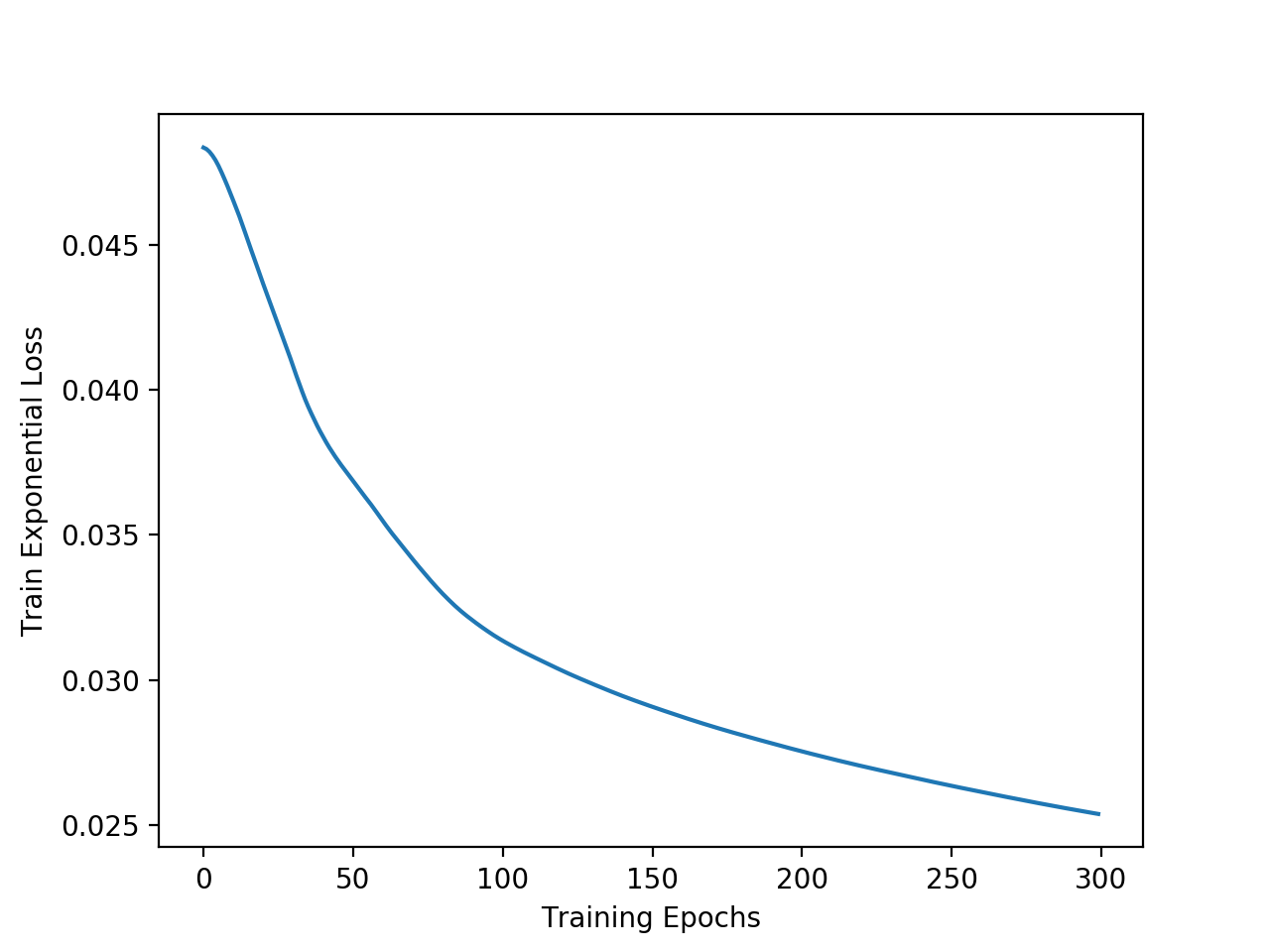}
\includegraphics[width=0.45\textwidth]{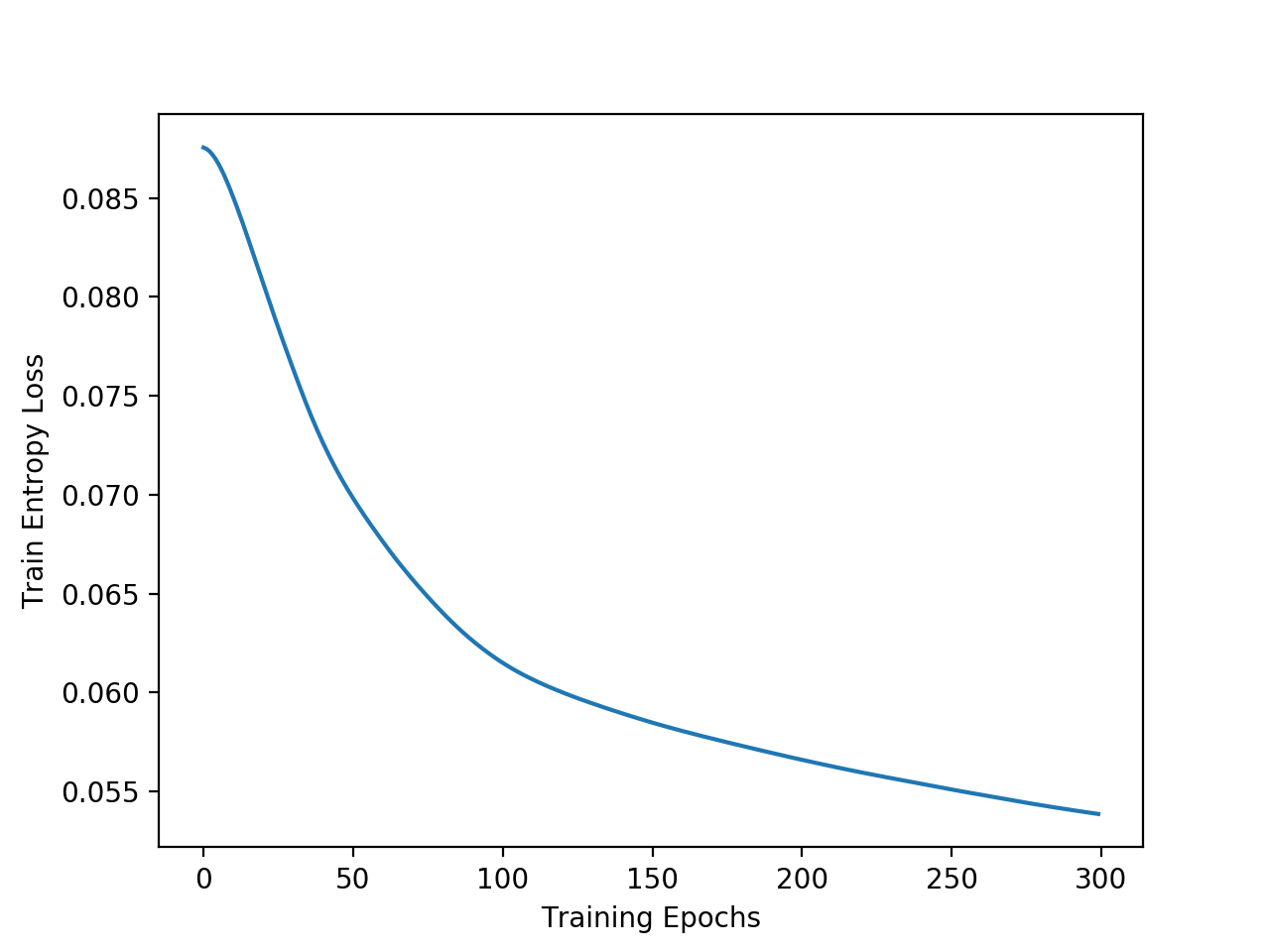}
\caption{\textbf{Left}: Training loss using $l(t)=\exp(-|t|)$; \textbf{Right}: Entropy loss when training using $l(t)=\exp(-|t|)$.}
\label{fig:ent_min}
\end{center}
\end{figure}

Self-training on target using $\ell(t)=\exp(-|t|)$ as an approximation for $\ell_{ent}(t)$ produces the same effect for binary MNIST (see figure \ref{fig:exp_loss}). We plot the training loss $l(t)=\exp(-|t|)$ and $l_{ent}(t)$ (Figure~\ref{fig:ent_min}) to show that they track each other really well.

\end{document}